\documentclass[11pt, a4paper]{article}
\usepackage{fontenc}
\usepackage[english]{babel}
\usepackage[utf8]{inputenc}
\usepackage{mathtools}
\usepackage{amsmath}
\usepackage{amsthm}
\usepackage{algorithm}
\usepackage{algorithmic}
\usepackage{amssymb}
\usepackage{stackengine}
\mathtoolsset{showonlyrefs}
\usepackage{tikz}
\usepackage{pgfplots}
\pgfplotsset{major grid style={dashed}}
\pgfplotsset{compat=1.18}
\usepackage[dvipsnames]{xcolor}
\usepackage{enumitem}
\usepackage
[
        a4paper,
        left=2.cm,
        right=2.cm,
        top=2.5cm,
        bottom=2.5cm,
]
{geometry}
\usepackage{comment}
\usepackage{mathrsfs} 
\usepackage{dsfont}
\usepackage{esint}
\usepackage[export]{adjustbox}
\usepackage{multirow}
\usepackage{mathtools}
\usepackage{enumitem}
\usepackage{scalerel}
\usepackage{color}
\usepackage{physics}
\definecolor{hanblue}{rgb}{0.27, 0.42, 0.81}
\definecolor{mordantred19}{rgb}{0.68, 0.05, 0.0}
\definecolor{red}{rgb}{0.68, 0.05, 0.0}
\definecolor{green}{rgb}{0.0, 0.5, 0.0}
\usepackage[colorlinks, citecolor=hanblue,linkcolor=green]{hyperref}

\DeclareMathOperator{\Dom}{\operatorname{Dom}}

\newcommand{\R}{\mathbb{R}}

\newcommand{\N}{\mathbb{N}}
\renewcommand{\S}{\mathbb{S}}

\newcommand{\E}{\mathcal{E}}

\renewcommand{\dd}{\, \mathrm{d}}

\DeclarePairedDelimiterX{\set}[1]{\lbrace}{\rbrace}{#1}

\newcommand\restr[2]{{
  \left.\kern-\nulldelimiterspace
  #1 
  \vphantom{\big|} 
  \right|_{#2}
  }}

\allowdisplaybreaks

\theoremstyle{plain}
\newtheorem{thm}{Theorem}
\numberwithin{thm}{section}

\newtheorem{lemma}[thm]{Lemma}
\newtheorem{prop}[thm]{Proposition}

\newtheorem{cor}[thm]{Corollary}

\theoremstyle{definition}
\newtheorem{defi}[thm]{Definition}

\newtheorem{remark}[thm]{Remark}
\theoremstyle{remark}

\newcommand{\CE}{\mathrm{CE}}
\renewcommand{\leq}{\leqslant}
\renewcommand{\geq}{\geqslant}
\renewcommand{\le}{\leqslant}
\renewcommand{\ge}{\geqslant}

\title{Multi-Headed Transformer Architectures as Time-dependent Wasserstein Gradient Flows}
\author{
Alex Massucco\textsuperscript{*1}\hspace{35pt}
Leonardo Del Grande\textsuperscript{*2}\hspace{35pt}
Marcello Carioni\textsuperscript{2}\\[15pt]
Christoph Brune\textsuperscript{2}\hspace{35pt}
Carola-Bibiane Schönlieb\textsuperscript{1}
}
\date{}

\begin{document}

\maketitle

\footnotetext[1]{Department of Applied Mathematics and Theoretical Physics, University of Cambridge, Cambridge, UK}
\footnotetext[2]{Department of Mathematics, University of Twente, Enschede, Netherlands}
\begingroup
\renewcommand{\thefootnote}{\fnsymbol{footnote}}
\footnotetext[1]{These authors contributed equally to this work.}
\endgroup

\begin{abstract}
\noindent
In recent years, transformer architectures have revolutionized the field of language processing, opening the door to previously unforeseen possibilities. However, from a theoretical point of view, the mathematical models proposed in the literature often lack direct contact with the actual architectures and depend on strong simplifying assumptions. In this paper, we reduce this gap by modelling the data flow in multi-headed transformer architectures as time-dependent gradient flows for a suitable interaction energy capturing the design of the attention mechanism. The explicit dependence on time allows us to consider different weights for each head and for each layer, without imposing constraints on the initialization method. Moreover, we prove that, under a suitable integrability assumption on the evolution of the weights, each element of the $\omega$-limit set of the gradient flows is a stationary point of the interaction energy at a limiting weight distribution. Finally, we analyse the stability of the gradient flows considering perturbations of both the initial data and the weights. Specifically, on the one hand, we study the robustness of the proposed models with respect to noisy inputs, establishing a continuous dependence of the gradient flows on the initial data and uniqueness of the flows. On the other hand, we prove the $\Gamma$-convergence of the perturbed interaction energy to the unperturbed one, leading to the convergence of the corresponding gradient flows. We complement these theoretical results with numerical experiments that confirm the predicted energy-dissipation identity and clarify the asymptotic behavior of the dynamics in both the autonomous-like (Ornstein--Uhlenbeck) and the genuinely non-autonomous (oscillating-weights) regimes.
\end{abstract}

\textbf{AMS subject classifications (MSC2020):} 68T07, 35Q93, 49J53, 46N10, 49Q20.

\section{Introduction}

Since their introduction \cite{vaswani2017attention}, transformers architectures have rapidly emerged as a leading paradigm in machine learning, achieving state-of-art performance in many data-science tasks. These include astonishing development of Large Language Models (such as GPT \cite{radford2019language}, Gemini \cite{team2023gemini}, Deepseek \cite{liu2024deepseek}), together with advances in machine translation \cite{wu2016google}, computer vision \cite{he2016deep, dosovitskiy2020image, radford2021learning}, speech and audio recognition \cite{baevski2020wav2vec, radford2023robust}, protein and molecular modeling \cite{jumper2021highly}, drug discovery \cite{stokes2020deep}, weather forecasting \cite{lam2023learning}, and applications in health care \cite{topol2019high} and finance \cite{gu2020empirical}.
A central feature that has distinguished transformers from the other architectures is the presence and the choice of the self-attention mechanism \cite{vaswani2017attention,kim2021lipschitz, sander2022sinkformers, shen2023study, hua2022transformer, ramapuram2024theory, castin2025unified}, which enables each input element to share and integrate information from distant parts of the data efficiently. 

The remarkable success of transformer architectures raises fundamental questions about the underlying structural properties that drive their effectiveness and the role of the self-attention mechanism. A rigorous mathematical analysis of their behavior and inductive biases can yield valuable insights into how architectural modifications influence performance in real-world applications.
Such an understanding may also clarify key features of transformers, including depth heterogeneity \cite{zhang2024transformers}, the emergence of token clustering \cite{geshkovski2024emergence}, stability properties \cite{zhang2019fixup}, and physical interpretability \cite{geneva2022transformers}.

Motivated by these questions, a growing body of work has recently advanced the mathematical theory of transformers by interpreting their behavior as gradient flows minimizing suitable interaction energies \cite{geshkovski2023mathematical, burger2025analysis, castin2025unified, sander2022sinkformers}. 
Inspired by the seminal analysis in \cite{geshkovski2023mathematical}, the inference process of a transformer can be viewed as a particle interaction system, in which each particle represents a token evolving in time according to a dynamics governed by the transformer parameter matrices.
This observation has laid the groundwork for a mean-field theory of transformers, in which the particle interaction system is reinterpreted as a flow of time-dependent probability measures, each describing the evolving distribution of tokens.
This mean-field viewpoint has provided a rigorous framework to study the clustering behavior of transformers \cite{chen2025quantitative, geshkovski2024emergence, burger2025analysis, karagodin2024clustering}. In particular, it has been shown that the output measure tends to concentrate around a small number of tokens, with this concentration bias depending on the structural properties of the weight matrices of the transformer.
Beyond clustering phenomena, this theoretical framework has also sparked growing interest in controllability and approximation properties of transformers \cite{geshkovski2024measure}. 
Moreover, it has been observed in \cite{geshkovski2023mathematical} and rigorously analyzed in \cite{burger2025analysis} that the mean-field transformer dynamics can be described by a Wasserstein gradient flow of a suitable interaction energy. This observation has made it possible to employ the powerful and well-established framework of Wasserstein gradient flows \cite{ambrosio2008gradient} to rigorously explain several empirically observed features of transformer dynamics, such as energy dissipation and clustering.

Despite resting on solid theoretical foundations, existing analysis rely on rather idealized assumptions regarding the structure of the attention matrix and the time-independence of the weight parameters. In particular, prior works have primarily focused on the single-head setting with weight matrices that remain fixed throughout the evolution.
Operating under these assumptions is mathematically convenient but misaligned with modern architectures. In practice, parameters vary across both heads and layers, so the effective attention interaction changes with depth. In a continuous-depth formulation, this leads naturally to a genuinely non-autonomous vector field. By contrast, the autonomous regime is precisely the one in which clustering and concentration results are usually derived. Moreover, recent works (see e.g. \cite{tong2025neural}) provide numerical evidence that learned time-dependent weight dynamics can invalidate the autonomous approximation and may prevent the onset of clustering. This suggests that clustering is not a universal feature of deep transformer dynamics and may even be undesirable in autoregressive next-token prediction settings \cite{tong2025neural}. In addition to that, the stability properties of transformers, as well as their interplay with layer normalization and multi-head aggregation, are more faithfully described within non-autonomous frameworks, where the weights are allowed to vary both across heads and over depth \cite{zhang2019fixup, kan2025stability, xiong2020layer}. 
Finally, time-dependent models are more suited to capture underlying physical structures, symmetries, energy landscapes and complex dynamics in transformer inference \cite{ramsauer2020hopfield, huo2025capturing}.  
These observations motivate the central goal of this paper: \emph{to bridge the gap between the autonomous mean-field/Wasserstein gradient flow theory and the non-autonomous, multi-head architectures used in practice}. Concretely, we generalize the mean-field transformer framework by incorporating multi-head attention in a way that includes single head (and empirical head) models as a special case, and allowing parameters to evolve with depth via a time-dependent parameter distribution. We then analyze the resulting model from the perspective of non-autonomous Wasserstein gradient flows, developing a framework that accounts for the explicit time dependence of the interaction energy and incorporates its time derivative into energy balance relations describing the dynamics.
This maintains the explanatory power of the Wasserstein gradient flow perspective, while aligning the model with the depth heterogeneity that is intrinsic to modern transformers. 

\subsection{Mathematical framework and main results}

In mathematical terms we
deviate from \cite{geshkovski2023mathematical} and \cite{burger2025analysis} by considering multi-head transformer dynamics with time-dependent weights that is described by the following dynamical system for the tokens $x_i$, here written in the case of the softmax self-attention mechanism:
\begin{align}
\dot{x}_i(t) = P_{x_i(t)}^{\perp}\left(\frac{1}{H} \sum_{h=1}^H \sum_{j=1}^n  \frac{e^{\langle Q(\theta_t^{(h)})x_i, K(\theta_t^{(h)})x_j\rangle}}{\sum_{p=1}^n e^{\langle Q(\theta_t^{(h)})x_i, K(\theta_t^{(h)})x_p\rangle}} V(\theta^{(h)})x_j \right),
\end{align}
where $H$ is the number of heads,  $Q(\theta_t^{(h)}),K(\theta_t^{(h)}),V(\theta_t^{(h)})$ are the time-dependent and head-dependent matrices of \emph{queries}, \emph{keys} and \emph{values} depending on the transformer parameter $\theta_t^{(h)}$ for each head $h$ and $P^\perp$ is the projection on the tangent space of the sphere $\mathbb{S}^{d-1}$ (encoding the layer normalization). 
We then consider the mean-field limit of the previous dynamical system both in the heads $\theta^{(h)}$ and in the tokens $x_i$. This is achieved by representing the tokens as a time-dependent probability measures $t \mapsto \mu_t$ and the transformer parameters $\theta^{(h)}$ as time-dependent probability measures $t \mapsto \Theta_t \in \mathcal{P}_2(\R^\alpha)$.
This recovers the multi-head formulation when $\Theta$ is an empirical probability measure and it reduces to the autonomous setting when $\Theta$ does not depend on time.
This leads to a mean-field transformer dynamics described by the following PDE (Transformer PDE):
\begin{align}\label{eq:pdet}
\partial_t{\mu_t} + \operatorname{div}_{\mathbb{S}^{d - 1}} \left(\mu_t \int_{\R^{\alpha}} \mathcal{V}[\mu_t](x;\theta) \dd \Theta_t(\theta)\right) = 0
\end{align}
where $\mathcal{V}[\mu_t]$ is defined as 
\begin{align}
\mathcal{V}[\mu_t](x;\theta) = P_x^{\perp}\left(\int_{{\mathbb{S}^{d-1}}} \frac{e^{\langle Q(\theta)x, K(\theta) y\rangle}}{\int_{\mathbb{S}^{d-1}} e^{\langle Q(\theta)x, K(\theta)y'\rangle} d\mu_t(y')}  V(\theta)y\dd \mu_t(y)\right).    
\end{align}
We choose to model $\Theta_t$ as the solution of a stochastic differential equation $\dd \tilde\theta_t = f(\tilde\theta_t, t) \dd t + g(\tilde\theta_t, t) \dd W_t$ for a suitable choice of $f$ and $g$. This allows to model random initialization strategies as well as noise injection throughout the layers.
 We refer to Section \ref{sec:intro_multi} and Section \ref{sec:intro-meanfield} for a precise mathematical description of our models, the assumptions we use and the generalization to other self-attention mechanisms.

With the goal of developing a Wasserstein gradient flow analysis of the previously described model, following and generalizing \cite{burger2025analysis}, we consider a weighted Wasserstein distance ${\mathrm{W}}_{m,2}(\cdot,\cdot;\Theta)$ where the non-linear mobility $m$ reflects the layer normalization. In the infinite-head regime, the corresponding kinetic energy is obtained by averaging, with respect to the time-dependent parameter distribution \(\Theta_t\), the contributions of weight-dependent velocity fields. A key point of our analysis is that this weighted distance is equivalent to the classical \(2\)-Wasserstein distance. This allows us to use the compactness, lower-semicontinuity, and stability tools of optimal transport while working in a geometry adapted to the normalized attention dynamics.

To understand time-dependent gradient flows, that is the ultimate goal of our paper, the first step in the analysis is the characterization of absolutely continuous curves with respect to \(W_{m,2}(\cdot,\cdot;\Theta)\). This is already one of the points where the present time-dependent, multi-head setting differs from the single-head autonomous framework. Indeed, in the multi-head regime, the velocity field depends not only on the token variable \(x\), but also on the parameter variable \(\theta\). However, the evolution of the token distribution only depends on the velocity obtained after averaging over the parameter distribution \(\Theta_t\).

The cost of producing a given averaged velocity is determined by the different mobilities associated with the parameters \(\theta\). More precisely, a given averaged velocity $\bar v$ on $\mathbb S^{d-1}$ admits infinitely many representations as a $\Theta_t$-average of $\theta$-dependent velocities; among these, the one minimizing the kinetic cost, weighted pointwise in $x$ by $m_{\mu_t}(x,\theta)$, is selected by a constrained quadratic minimization. Solving it explicitly via Lagrange multipliers (see Lemma~\ref{lem:metric_der}) yields the effective scalar mobility at $x$:
\[
b^{\Theta_t}_{\mu}(x)
:=
\left(
\int_{\mathbb R^\alpha}
\frac{1}{m_{\mu_t}(x,\theta)}
\,d\Theta_t(\theta)
\right)^{-1},
\]
that is the \(\Theta_t\)-weighted harmonic mean of the mobilities associated with the different heads. Its qualitative behavior reflects a typical feature of harmonic means: $b_{\mu_t}^{\Theta_t}(x)$ is small, and transport at $x$ correspondingly cheap, as soon as a non-negligible fraction of heads, weighted by $\Theta_t$, admits a low mobility $m_{\mu_t}(x,\theta)$ at $x$. In this way, $b_{\mu_t}^{\Theta_t}$ describes the local transport geometry seen by the token distribution after averaging over the head parameters. In particular, the dependence on the individual weights is not lost, but is encoded in a single scalar mobility factor that, at each point $x \in \mathbb{S}^{d-1}$, determines the energetic cost of moving mass; in the single-head case $\Theta_t = \delta_{\bar\theta}$, it collapses to $b_{\mu_t}^{\Theta_t}(x) = m_{\mu_t}(x,\bar\theta)$, recovering the framework of \cite{burger2025analysis}. This effective mobility is the key object that allows us to describe the multi-head weighted geometry as a Wasserstein-type geometry on the token space, while still retaining the influence of the time-dependent parameter distribution.

Building on such analysis and taking inspiration from \cite{ferreira2018gradient,rossi2008metric} we then define Wasserstein gradient flows of the time-dependent interaction energy 
\begin{align}\label{eq:intintro}
    \mathcal{E}(\mu,t) = \frac{1}{2} \int_{\R^\alpha \times \mathbb{S}^{d-1} \times \mathbb{S}^{d-1}} \mathscr{K}(x,y,\theta) \, \dd\mu(y)\dd\mu(x)\dd\Theta_t(\theta) 
\end{align}
through a minimizing movement approach  with the goal of relating such time-dependent gradient flow to solutions of the transformer PDE. In particular for a suitable choice of $\mathscr{K}$ this recovers \eqref{eq:pdet} in the case $Q^T K = V$. More precisely, we show that minimizing movements are solutions of the Transformer PDE and viceversa solutions of the Transformer PDE are curves of maximal slope for the (time-dependent) interaction energy, again defined taking into account that the interaction energy depends explicitly on time \cite{ferreira2018gradient, rossi2008metric} (see Definition \ref{def:curves_max_slope} for a precise definition). 
Such characterization is again crucially based on the characterization of the metric derivative obtained by considering the local transport geometry seen through the effective mobility $b^{\Theta_t}_{\mu}(x)$ and allows to derive a corresponding energy-dissipation identity (EDI) for solutions of the transformer PDE.   


We then investigate the asymptotic behavior of solutions to the transformer PDE, exploiting the fact that they are curves of maximal slope for the time-dependent energy $\mathcal{E}(\cdot,t)$. In contrast to \cite{burger2025analysis,geshkovski2023mathematical}, for a general time-dependent family $t \mapsto \Theta_t$ one cannot expect convergence to stationary states, especially when $\Theta_t$ exhibits oscillatory behavior. Nevertheless, if $\int_0^\infty |\Theta_t'| \, \mathrm{d}t < \infty$,
then part of the asymptotic behavior of the autonomous setting can still be recovered. More precisely, the $\omega$-limit set is non-empty, and each of its elements is a stationary point of a limit energy associated with a convergent subsequence $(\Theta_{t_n})_n$.

We also study the stability of the transformer dynamics with respect to perturbations both of the initial token distribution and of the time-dependent family of probability measures $\Theta_t$ representing the weights. In particular, we show that stability with respect to the initialization is governed by the quantity
$M_\theta(t) = \int_0^t \int_{\mathbb{R}^\alpha} \|D(\theta)\|^2 \, \mathrm{d}\Theta_\tau(\theta)\, \mathrm{d}\tau$.
This highlights a head- and layer-dependent normalization of the form
$
z \mapsto \frac{z}{t\sqrt{H}}$,
where $H$ denotes the number of heads, which keeps $M_\theta(t)$ independent of the network depth.

Finally, we complement the theoretical analysis with numerical experiments on $\mathbb S^2$ that confirm the predicted dissipation and stability behavior of the dynamics. Two qualitatively different regimes are considered: an Ornstein--Uhlenbeck (OU) evolution of the weights, which satisfies the integrability assumption needed for our long-time results, and a deterministic oscillating evolution, which does not. In the OU regime, the token cloud is observed to relax to a stationary configuration consistent with Theorem~\ref{thm:stationary_points}, and the strong upper gradient is shown to decay at the predicted rate $\mathcal{O}(H^{-1})$ in the number of heads. Instead, in the oscillating regime, neither relaxation nor decay of the gradient is observed, illustrating the sharpness of the integrability assumption. In both cases the discrete energy identity \eqref{eq:discrete_energy_balance}, which is a numerical analogue of \eqref{eq:var_ineq}, is verified along the simulated trajectories.

\subsection{Summary of main contributions and plan of the paper}

Our main contributions can be summarized as follows:
\begin{itemize}
\item \textbf{Mean field model for multi-head time-dependent transformers.}  We introduce a measure-valued parameterization $\Theta_t$ that recovers single-head and multi-head formulations as special cases, and we allow these weights to vary with time, matching the non-autonomous structure of practical transformers. In particular, starting from a Transformer ODE model we will recover its mean field formulation, called Transformer PDE. This is carried out in Section \ref{sec:intro_multi}.
  \item \textbf{Weighted Wasserstein geometries and existence of minimizing movements.}
    We introduce nonlinear anisotropic mobilities induced by the attention mechanism and, generalizing \cite{burger2025analysis}, we define the corresponding weighted Wasserstein distance ${\mathrm{W}}_{m,2}(\cdot,\cdot;\Theta)$ for time-dependent weights $\Theta_t$. We prove that ${\mathrm{W}}_{m,2}(\cdot,\cdot;\Theta)$ is equivalent to the classical \({\mathrm{W}}_2\), and  we construct Wasserstein gradient flows solutions for the time-dependent energy \( \mathcal{E}(\cdot,t) \). A central step here is the characterization of the metric derivative in terms of the effective mobility $b^{\Theta_t}_\mu(x)$ introduced above, namely the $\Theta_t$-weighted harmonic mean of $m_\mu(x,\cdot)$. This is carried out in Section \ref{sec:time_dep_GF} and provides the geometric and analytical foundations used throughout the paper.

  \item \textbf{A gradient flow characterization for the Transformer PDE.}
  We show that solutions to the constructed gradient flows solve the mean-field continuity equation \eqref{eq:pdet}. Conversely, we prove that weak solutions of the associated continuity equation are curves of maximal slope for \(\mathcal{E}(\cdot,t)\) with respect to $\mathrm{W}_{m,2}(\cdot,\cdot;\Theta)$. Furthermore, we study the long-time behavior of the system, showing that if the oscillations of $\Theta_t$ remain suitably controlled as $t \to +\infty$, then the $\omega$-limit set is non-empty and is contained in the union of the sets of stationary points associated with the cluster points of $\Theta_t$.
  This is carried out in Section \ref{sec:grad_flow_PDE}. 

  \item \textbf{Robustness and stability for transformer inference dynamics.}
  We establish a Grönwall-type estimate that controls the distance between trajectories corresponding to different initial data, justifying layer-dependent normalizations, and to perturbations \(f^\varepsilon,g^\varepsilon\) of the weight-initialization functions \(f,g\). This is carried out in Section \ref{sec:stability}.

  \item \textbf{Numerical validation.} We derive a discrete analogue \eqref{eq:discrete_energy_balance} of the energy-dissipation identity \eqref{eq:var_ineq} for a projected Euler / Euler--Maruyama scheme of the coupled token--weight system, and use it to numerically corroborate the predicted long-time behavior on $\mathbb{S}^2$. Monte Carlo simulations in the Ornstein--Uhlenbeck regime quantitatively confirm the $\mathcal{O}(H^{-1})$ decay of the strong upper gradient in the number of heads $H$ and the relaxation of tokens to a stationary configuration predicted by Theorem~\ref{thm:stationary_points}; the oscillating-weights regime, in which the underlying integrability assumption fails, exhibits no such decay and no persistent clustering, providing numerical evidence for the necessity of the assumption. This is carried out in Section~\ref{sec:numerics_long_time}.
\end{itemize}

\section{From Transformer ODE to Transformer PDE} \label{sec:intro_multi}
\subsection{Multi-head transformers as time-dependent particle interacting systems}
Let us fix $n, d \in \mathbb{N}$, $n, d \geq 1$, and consider a set of tokens (or particles) $(x_i)_{i=1}^n \subseteq \mathbb{R}^d$. Moreover, for every $h = 1, \dots, H$ with $H \geq 1$, we set $\theta^{(h)} \in \R^N$, $N \in \mathbb{N}$, and $\tilde\theta = (\theta^{(1)}, \dots, \theta^{(H)}) \in \mathbb{R}^\alpha$, $\alpha := HN$.
Then, we define the \emph{generalized multi-headed attention vector field} $\mathcal{A}: \mathbb{R}^d \to \mathbb{R}^d$ as 
\begin{equation}\label{eq:attention_mechanism}
\mathcal{A}(x_i; \tilde\theta) := \frac{1}{H}\sum_{h = 1}^H\sum_{j = 1}^n A_{ij}(\theta^{(h)}) B_{ij}(\theta^{(h)}),
\end{equation}
where, for every $h \in [1, H]$, $B_{ij}(\theta^{(h)}) = B(x_i, x_j, \theta^{(h)}) \in \mathbb{R}^d$ selects the directions along which the attention head operates and $A_{ij}(\theta^{(h)}) = A(x_i, x_j, \theta^{(h)}) \in \R$ is the \textit{attention matrix} given by
\begin{equation}\label{eq:self_att}
A_{ij}(\theta^{(h)}) := \frac{\sigma(x_i, x_j, \theta^{(h)})}{\mathcal{N}_{i}(\theta^{(h)})},
\end{equation}
with $\sigma: \mathbb{R}^d \times \mathbb{R}^d \times \mathbb{R}^N \to \mathbb{R}$ and $\mathcal{N}_i(\theta^{(h)}) = \mathcal{N}(x_i, \theta^{(h)}) \in \mathbb{R}$ is a normalization factor. 
A common example of the above mechanism is the \textit{softmax self-attention}, in which, for every $\theta$
\begin{align}
\mathcal{N}_{i}(\theta) := \sum_{p=1}^{n} \sigma(x_i, x_p, \theta), \qquad \sigma(x_i, x_j; \theta) = e^{\langle Q(\theta) x_i, K(\theta) x_j\rangle}, \qquad B_{ij}(\theta) = V(\theta) x_j,
\end{align}
where $\langle\cdot,\cdot\rangle$ denotes the standard Euclidean scalar product, and $Q(\theta)$, $K(\theta) \in \mathbb{R}^{k \times d}$, $V(\theta) \in \mathbb{R}^{d \times d}$ are respectively the \textit{queries}, the \textit{keys} and the \textit{values} weights matrices. Moreover, to simplify the notation, it is often defined $D(\theta) := Q^\top(\theta) K(\theta) \in \R^{d\times d}$, so that \eqref{eq:attention_mechanism} reads as
\begin{equation}\label{eq:softmax_self_attention}
\mathcal{A}_{\mathrm{softmax}}(x_i; \tilde\theta) := \frac{1}{H}\sum_{h = 1}^H\sum_{j = 1}^n \frac{e^{\langle x_i, D(\theta^{(h)}) x_j\rangle}}{\sum_{p = 1}^n e^{\langle x_i, D(\theta^{(h)}) x_p\rangle}} V(\theta^{(h)})x_j.
\end{equation}
Heuristically, this mechanism creates a linear combination of the tokens weighted by the components in the rows of the matrix $A$, which we can interpret as the probability vectors encoding the (nonlinear) weighted coupling of a token with every other particle in the considered set. For a more comprehensive analysis of several other (less frequent) attention mechanisms, we refer, for instance, to \cite{castin2025unified}.

Employing \eqref{eq:attention_mechanism} and following \cite{geshkovski2023mathematical}, we can define a first simple example of a multi-headed transformer model as
\[
x_i^\ell = x_i^{\ell - 1} + hP_{x_i^{\ell -1}}^{\perp}\mathcal{A}(x^{\ell -1}_i; \tilde\theta), \qquad i \in [1, n], \, \ell \in [1, L], \, h > 0,
\]
where $P_x^{\perp}y := y - \langle y,x\rangle x$ is the orthogonal projector on the tangent space of the sphere $\mathbb{S}^{d - 1}$. The choice to include this projector (that is also used in \cite{burger2025analysis,geshkovski2023mathematical}) is justified by the strong empirical performance of architectures that normalize hidden states at each layer \cite{touvron2023llama}. In particular, the recently adopted normalization RMSNorm \cite{zhang2019root} scales each token to unit norm, effectively projecting representations onto $\mathbb{S}^{d-1}$ after each layer.

The above update formula can be interpreted (see for instance \cite{geshkovski2023mathematical}) as an $h$-step of an explicit Euler scheme, discretizing the following ODE
\begin{equation}
\left\{\begin{array}{l}
\dot{x}_i(t)=P_{x_i(t)}^{\perp}\mathcal{A}(x_i(t); \tilde\theta), \qquad i \in [1, n],\\
x_i(0)=x_i^0,
\end{array}\right.
\end{equation}
where the time $t \in [0, T]$, $T > 0$, embeds the continuous layer evolution. We will refer to this as the \textit{Transformer ODE}. 

The model presented is actually quite different from the architectures used in practice. In order to reduce this gap, we will allow the weight vector $\tilde \theta$ to depend on the time. We consider the following multi-headed fully time-dependent transformer model
\begin{equation}\label{eq:full_transformer_ODE}
\left\{\begin{array}{l}
\dot{x}_i(t) = P_{x_i(t)}^{\perp}\left(\frac{1}{H} \sum_{h=1}^H \sum_{j=1}^n A_{ij}(\theta^{(h)}_t) B_{ij}(\theta^{(h)}_t)\right), \qquad i \in [1, n]\\
\dd \tilde\theta_t = f(\tilde\theta_t, t) \dd t + g(\tilde\theta_t, t) \dd W_t,\\ 
x_i(0)=x_i^0, \, \tilde\theta_0=\theta^0,
\end{array}\right.
\end{equation}

where $W_t$ is a $\xi$-dimensional Brownian motion and $f: \R^\alpha \times\mathbb R \to \R^\alpha$, $g: \R^\alpha \times \mathbb R \to  \mathbb{R}^{\alpha \times \xi}$ are given by
\begin{align}
f(\tilde\theta_t, t)=\Big(f_1(\theta^{(1)}_t,t),\, \dots,\,f_H(\theta^{(H)}_t,t)\Big), \qquad g(\tilde\theta_t, t) = \Big{(}g_1(\theta_t^{(1)}, t), \,\dots, \, g_H(\theta_t^{(H)}, t) \Big{)},
\end{align}
satisfying that $f\in L^1_{\mathrm{loc}}([0,+\infty);\, L_{\mathrm{loc}}^\infty(\mathbb{R}^\alpha;\, \mathbb{R}^\alpha))$, $g \in L^2_{\mathrm{loc}}([0,+\infty);\, L_{\mathrm{loc}}^\infty(\mathbb{R}^\alpha;\, \mathbb{R}^{\alpha \times \xi}))$ and also that\\
$f_h \in L^1_{\mathrm{loc}}([0,+\infty);\, L^\infty_{\mathrm{loc}}(\mathbb{R}^N_\theta;\, \mathbb{R}^N))$, $g_h \in L^2_{\mathrm{loc}}([0,+\infty);\, L_{\mathrm{loc}}^\infty(\mathbb{R}^N;\, \mathbb{R}^{N \times \xi}))$ for every $h \in [1, H]$.

We emphasize that we have modelled the weights evolution as the solution of the SDE $\dd \tilde\theta_t = f(\tilde\theta_t, t) \dd t + g(\tilde\theta_t, t) \dd W_t$.
This can be mathematically modelled as an Itô process and includes a wide class of initialization techniques involving random sampling.
Finally, we underline that all the arguments we will present in the following can be generalized to any Lévy process with some additional technicalities involving its generator.

\subsection{A mean-field formulation for tokens and heads}\label{sec:intro-meanfield}

Now, we will give an equivalent mean-field interpretation of the previous transformer model. Let us consider the empirical measures $\mu_t \in \mathcal{P}_2(\mathbb{S}^{d - 1})$ and $\Theta_t \in \mathcal{P}_2(\R^\alpha)$ given by
\begin{equation*}
    \mu_t := \frac{1}{n} \sum_{i=1}^n \delta_{x_i(t)}\quad \text{and} \quad \Theta_t := \frac{1}{H} \sum_{h = 1}^H \delta_{\theta_t^{(h)}},
\end{equation*}
where $\delta_{x_i(t)}$ is a Dirac delta concentrated at the token $x_i(t)$. Then, recalling \eqref{eq:self_att}, and noticing that $P^\perp_{x}$ is linear in $y$, the first equation of \eqref{eq:full_transformer_ODE} can be rewritten as
\begin{equation}\label{eq:almost_pde_full_ode}
\dot{x}_i(t) = \int_{\R^{\alpha}}\mathcal{V}[\mu_t](x_i(t); \theta) \dd \Theta_t(\theta),
\end{equation}
where, for every $\mu \in \mathcal{P}_2(\mathbb{S}^{d - 1})$, we denote by $\mathcal{V}[\mu]: \mathbb{S}^{d-1} \times \R^\alpha \rightarrow \mathrm{T}\mathbb{S}^{d-1}$ the velocity field given by
\begin{equation}\label{eq:vect_field}
\mathcal{V}[\mu](x; \theta) := P_x^{\perp}\left(\int_{{\mathbb{S}^{d-1}}} A(x, y, \theta) B(x, y; \theta) \dd \mu(y)\right).
\end{equation}
Following \cite{burger2025analysis, geshkovski2023mathematical}, for every function $\varphi \in C_c^\infty(\mathbb{S}^{d - 1} \times (0,+\infty))$, by the chain rule, it holds
\begin{equation}\label{eq:chain_rule}
\frac{\dd}{\dd t} \int_{\mathbb{S}^{d-1}} \varphi(x, t) \dd \mu_t(x) = \int_{\mathbb{S}^{d-1}} \partial_t \varphi(x, t) + \Big{\langle}\nabla_{\mathbb{S}^{d-1}} \varphi(x, t), \dot{x}(t)\Big{\rangle} \dd \mu_t(x),
\end{equation}
where $\nabla_{\mathbb{S}^{d-1}}$ is the Riemannian gradient on the submanifold $\mathbb{S}^{d-1}$ induced by the metric associated with the standard inner product on $\mathbb{R}^d$, given by 
\[
\nabla_{\mathbb{S}^{d - 1}}\varphi(x) := P_x^\perp(\nabla \varphi(x)),
\]
with $\nabla$ indicating the usual gradient in $\mathbb{R}^d$. Then, by \eqref{eq:chain_rule} and \eqref{eq:almost_pde_full_ode}, we get the following weak continuity equation 
\begin{equation}\label{eq:cont_eq_weak}
\int_0^\infty \int_{\mathbb{S}^{d-1}} \partial_t \varphi(x, t) + \left\langle \nabla_{\mathbb{S}^{d-1}} \varphi(x, t), \int_{\R^\alpha} \mathcal{V}\left[\mu_t\right](x, \theta) \dd \Theta_t(\theta)\right\rangle \dd \mu_t(x) \dd t = 0, \, \forall \varphi \in C^\infty_c(\mathbb{S}^{d - 1} \times  (0,+\infty)).
\end{equation}
Finally, considering the second equation in~\eqref{eq:full_transformer_ODE}, since the SDE coefficients $f$ and $g$ depend only on $\tilde\theta_t$ and $t$, applying Itô's formula yields the following weak formulation
\begin{equation}\label{eq:ito_weak}
\begin{aligned}
\int_{0}^{\infty} \int_{\R^\alpha} \partial_t \psi(\theta, t) & + \Big{\langle} \nabla_\theta \psi(\theta, t), f(\theta, t)\Big{\rangle}  -  \nabla_\theta^2\psi(\theta, t): G(\theta, t) \dd \Theta_t(\theta) \dd t = 0, \, \forall \psi \in C^\infty_c(\R^\alpha \times (0,+\infty)),
\end{aligned}
\end{equation}
where $:$ indicates the Frobenius inner product, $G(\theta, t) = \frac{1}{2} g(\theta, t) g(\theta, t)^\top \in \mathbb{R}^{\alpha \times \alpha}$. 
Therefore, for some initial distributions $\mu^0 \in \mathcal{P}_2(\mathbb{S}^{d - 1})$ and $\Theta^0 \in \mathcal{P}_2(\R^\alpha)$, we can rewrite system \eqref{eq:full_transformer_ODE} as (a weak formulation of)
\begin{align}
\left\{\begin{array}{l}
\displaystyle{\partial_t{\mu_t} + \operatorname{div}_{\mathbb{S}^{d - 1}} \left(\mu_t \int_{\R^\alpha} \mathcal{V}[\mu_t](x;\theta) \dd \Theta_t(\theta)\right) = 0,}\\
\displaystyle{\partial_t{\Theta_t}(\theta) + \operatorname{div}_\theta \Big{(}\Theta_t(\theta) f(\theta, t)\Big{)} - \operatorname{div}_\theta\Big{(}\operatorname{div}_\theta(\Theta_t(\theta) G(\theta, t))\Big{)} = 0,}\\
\displaystyle{\mu_0= \mu^0, \, \Theta_0 = \Theta^0.}
\end{array}\right.
\label{eq:continuity_eq_full}
\end{align}
We will refer to the system \eqref{eq:continuity_eq_full} as the \emph{Transformer PDE}.
We also emphasize that the solution of the equation governing the flow of $\Theta_t$ in \eqref{eq:continuity_eq_full} can be obtained \textit{a priori} since $f$ and $g$ are given.

\section{Wasserstein Time-dependent Minimizing Movements}\label{sec:time_dep_GF}
In this section, we employ a minimizing movement construction to derive Wasserstein (weighted) time-dependent gradient flows for the interaction energy defined as follows:
\begin{equation}\label{eq:time_dep_interaction_energy}
\mathcal{E}(\mu, t) = \frac{1}{2} \iint_{\R^\alpha \times \mathbb{S}^{d-1}} \underbrace{\int_{\mathbb{S}^{d-1}} \mathscr{K}(x, y, \theta) \dd \mu(y)}_{:=\mathscr{K}[\mu](x, \theta)} \dd \mu(x) \dd \Theta_t(\theta),
\end{equation}
where $\Theta_t$ is a solution of the second equation in \eqref{eq:continuity_eq_full}.
We make the following assumptions on the kernel $\mathscr{K} : \mathbb{S}^{d-1} \times \mathbb{S}^{d-1} \times \R^\alpha \rightarrow \R$:
\begin{enumerate}[label=(\textbf{K\arabic*})]
    \item\label{hyp:ker4} $\theta \mapsto \mathscr{K}(x, y, \theta)$ is supported in a compact set contained in $ \mathbb{R}^\alpha$ for every $x,y \in \mathbb{S}^{d-1}$. 
    \item\label{hyp:ker1} $\mathscr{K}(\cdot,\cdot, \theta) \in C^{1, 1}(\mathbb{S}^{d-1} \times \mathbb{S}^{d-1})$ for all $\theta \in \mathbb{R}^\alpha$, $\mathscr{K}(x,y, \cdot) \in C^2(\R^\alpha)$ for all $x, y \in \mathbb{S}^{d-1}$.
    \item\label{hyp:ker3} $\mathscr{K}$ is symmetric in the first two entries, that is $\mathscr{K}(x, y, \cdot) = \mathscr{K}(y, x, \cdot)$ for every $x, y \in \mathbb{S}^{d-1}$.
\end{enumerate}
We emphasize that the explicit dependency on time in the definition of $\mathcal{E}$ is expressed through the distribution $\Theta_t$ provided by the second equation in the Transformer PDE \eqref{eq:continuity_eq_full}.

Following \cite{geshkovski2023mathematical, burger2025analysis} the Wasserstein gradient flow will be defined with respect to a weighted Wasserstein distance defined using non-linear mobilities needed to deal with the normalized attention. For this reason, in this section, we first define the weighted metric setting specifying the assumptions required for the non-linear mobility. Then, since $\mathcal{E}$ is a time-dependent functional, we adapt some definitions from \cite{ambrosio2008gradient} to work in this scenario, following closely the theory developed in \cite{ferreira2018gradient, rossi2008metric}. Finally, we will prove the existence of minimizing movements for $\mathcal{E}$, providing the first building block for the next discussion.

\subsection{Time-dependent Gradient Flows}
Let $(X, d)$ be a complete metric space and let 
\[
E:X\times[0, T]\to(-\infty,+\infty], \qquad T > 0,
\]
be such that for every $x\in X$, the map $t\mapsto E(x,t)$ is locally absolutely continuous. 
\begin{defi}[Metric derivative]\label{def:metric_derivative}
For any absolutely continuous curve $v:[0,T]\rightarrow X$, $T > 0$, we define the \emph{metric derivative} of $v$ for a.e. $t\in [0,T]$ as 
\begin{equation}\label{eq:metric_derivative}
\left|\dot{v}\right|(t)=\lim _{h \rightarrow 0} \frac{d(v(t+h), v(t))}{|h|}\,.
 \end{equation}
\end{defi}
\noindent
It is exactly the minimal function $g \in$ $L^1([0,T])$ satisfying the inequality 
\begin{equation}\label{eq:metric_slope_int}
d(v(s), v(r)) \leq \int_s^r g(t)\, d t \quad \forall\,0 \leq s \leq r \leq T\,.
\end{equation}

\begin{defi}[Strong upper gradient]\label{def:Strong_upper_gradient}
A family of functions $\mathcal{G}_t: X \rightarrow [0,+\infty]$ with $t\in[0,T]$ is a \textit{strong upper gradient} for $E\colon X\times[0,T]\to(-\infty,+\infty]$ if, for every absolutely continuous curve $v:[0,T]\rightarrow X$, the map $t\mapsto\mathcal{G}_t(v(t))$ is Borel on $[0,T]$ and it holds
\begin{equation}\label{def:strong_upper_grad}
\left| E(v(r),r) - E(v(s),s) - \int_s^r \partial_t E(v(t),t)\dd t \right| \leq \int_s^r \mathcal{G}_t(v(t))\left|\dot{v}\right|(t)\dd t, \qquad \forall\, 0\leq s\leq r\leq T.
\end{equation}
\end{defi}

\begin{defi}[Time-dependent curves of maximal slope]\label{def:curves_max_slope}
A locally absolutely continuous $u\colon[0,T]\to X$ is a \emph{curve of maximal slope} with respect to a family of functions $\mathcal{G}_t\colon X\rightarrow[0,+\infty]$, $t\in[0,T]$, if $t\mapsto\mathcal{G}_t(u(t))$ is Borel on $[0,T]$, $t\mapsto E(u(t),t)$ is absolutely continuous and, for all $0\leq s\leq r\leq T$, the following energy--dissipation inequality holds
\begin{equation}\label{eq:EDI_td}
E\!\left(u(r), r\right) 
+\frac12\int_s^r |\dot{u}|^2(t)\dd t
+\frac12\int_s^r \mathcal{G}_t^2\!\left(u(t)\right) \dd t
\;\leq\; 
E\!\left(u(s),s\right)+\int_s^r \partial_tE\!\left(u(t), t\right)\dd t.
\end{equation}
When equality holds in \eqref{eq:EDI_td} for every $0\leq s\leq r\leq T$ we say that $u$ satisfies the energy--dissipation identity (EDI).
\end{defi}

\subsection{Weighted Wasserstein distance}\label{sec:w_wass_dist}
Let us consider a Riemannian manifold $\mathcal{M} \subseteq \mathbb{R}^n$, $n \in \mathbb{N}$ with tangent bundle $\mathrm{T}\mathcal{M}$, and let $T > 0$. Then, for every couple of measures $\mu_0$, $\mu_T \in \mathcal{P}_2(\mathcal{M})$, we can define the classical Wasserstein distance through its dynamical (or Benamou-Brenier \cite{benamou2000computational}) formulation as follows:
\begin{align}\label{eq:wass}
{\mathrm{W}}_{2}^2(\mu_0, \mu_T) := \inf\left\{\int_0^T\int_{\mathcal{M}}  \|v_t(x)\|^2  \dd \mu_t(x) \dd t: (\mu_t, v_t) \in \mathrm{CE}((0, T); \mu_0 \rightarrow \mu_T)\right\},
\end{align}
where 
\[
\mathrm{CE}((0, T); \nu \rightarrow \eta) := \{(\mu_t, v_t) \in \mathrm{CE}(0, T): \mu_0 = \nu, \mu_T = \eta\}
\]
and $\mathrm{CE}(0, T)$ is the set of distributional solutions of the continuity equation on the manifold $\mathcal{M}$:
\begin{equation}\label{eq:CE}
\begin{aligned} \mathrm{CE}(0, T) :=\biggl\{(\mu_t, v_t) \in \mathrm{AC}([0, T];\mathcal{P}_2(\mathcal{M})) \times \Gamma([0,T] \times \mathcal{M};\mathrm{T}\mathcal{M})
: & \,\, \dot{\mu}_t = \mathrm{div}_{\mathcal{M}}(\mu_t v_t),\\
& \,\, \int_0^T \int_{\mathcal{M}} \|v_t\| \dd \mu_t(x) \dd t < \infty\biggr\}.\end{aligned}
\end{equation}
Here, $\Gamma([0,T] \times \mathcal{M},\mathrm{T}\mathcal{M})$ is the set of measurable vector fields $v : [0,T] \times \mathcal{M} \rightarrow \mathrm{T}\mathcal{M}$ such that $v_t(x) \in \mathrm{T}_x\mathcal{M}$ for every $x$. Moreover, the  continuity equation $\dot{\mu}_t = \mathrm{div}_{\mathcal{M}}(\mu_t v_t)$ is intended distributionally, that is 
\begin{equation}\label{eq:riem_ce}
\int_0^T\!\!\int_{\mathcal{M}}\partial_t\phi(x,t)
+ \langle \nabla_{\mathcal{M}}\phi(x,t),\, v_t(x)\rangle\,\mathrm d\mu_t(x)\,\dd t
=0
\end{equation}
for all test functions $\phi\in C_c^1\big((0,T)\times \mathcal{M}\big)$.

${\mathrm{W}}_2$ is an actual distance on $\mathcal{P}_2(\mathcal{M})$, so we will call a Wasserstein space the couple $(\mathcal{P}_2(\mathcal{M}), {\mathrm{W}}_2)$. We briefly recall in the following lemma the fundamental properties of this space (see for instance \cite{villani2008optimal, Santambrogio} for a more detailed analysis).
\begin{lemma}\label{lem:W2-prelim}
Let $(\mathcal{M}, d)$ be a Riemannian manifold. Then, the following holds:
\begin{enumerate}
  \item[\textnormal{(i)}] \emph{Lower boundedness and metricity.} 
  for every fixed $\nu\in\mathcal P_2(\mathcal{M})$ the map $\mu\mapsto {\mathrm{W}}_2^2(\mu,\nu)$ is finite and bounded from below on $\mathcal P_2(\mathcal{M})$.

  \item[\textnormal{(ii)}] \emph{Narrow lower semicontinuity.} Suppose $\mu_n,\nu_n\in\mathcal P_2(\mathcal{M})$ satisfy
  \[
  \mu_n \rightharpoonup \mu,\qquad \nu_n \rightharpoonup \nu \quad \text{narrowly in }\mathcal P(\mathcal{M}),
  \]
  and
  \[
  \sup_{n}\int_\mathcal{M} d(x,x_0)^2\,\mathrm d\mu_n(x) < \infty,\qquad 
  \sup_{n}\int_\mathcal{M} d(y,x_0)^2\,\mathrm d\nu_n(y) < \infty
  \]
  for some $x_0 \in \mathcal{M}$. Then, $\mu, \nu \in \mathcal{P}_2(\mathcal{M})$ and
  \[
  {\mathrm{W}}_2^2(\mu,\nu) \leq \liminf_{n\to\infty} {\mathrm{W}}_2^2(\mu_n,\nu_n).
  \]
  In particular, for a fixed $\nu\in\mathcal P_2(\mathcal{M})$, if $\mu_n\rightharpoonup\mu$ narrowly and $\sup_n\int d(x,x_0)^2\,\mathrm d\mu_n<\infty$, then $\mu\in\mathcal P_2(\mathcal{M})$ and $\mu\mapsto {\mathrm{W}}_2^2(\mu,\nu)$ is narrowly lower semicontinuous at $\mu$.

  \item[\textnormal{(iii)}] \emph{Narrow compactness of ${\mathrm{W}}_2$-bounded sets.} Let $A\subset\mathcal P_2(\mathcal{M})$ be ${\mathrm{W}}_2$-bounded, i.e., there exists $\bar \mu\in\mathcal P_2(\mathcal{M})$ and $r<\infty$ such that ${\mathrm{W}}_2(\mu,\bar \mu)\leq r$ for every $\mu\in A$. Then:
  \begin{enumerate}
    \item $A$ has uniformly bounded moment:
    \[
    \sup_{\mu\in A}\int_\mathcal{M} d(x,x_0)^2\,\mathrm d\mu(x) < \infty \quad \text{for any }x_0\in \mathcal{M};
    \]
    \item $A$ is tight, hence relatively compact for the narrow topology by Prokhorov’s theorem. In particular, every sequence in $A$ admits a narrowly convergent subsequence with limit in $\mathcal P_2(\mathcal{M})$.
  \end{enumerate}
\end{enumerate}
\end{lemma}

As highlighted in \cite{geshkovski2023mathematical}, because of the presence of the normalization factor $\mathcal{N}$, it is not possible to describe \eqref{eq:full_transformer_ODE} as the gradient flow of a functional in a classical Wasserstein space. In order to overcome this issue, following \cite{burger2025analysis}, we will embed $\mathcal{N}$ directly in the underlying metric considering a \emph{non-linear anisotropic mobility} $m: \mathcal{M} \times \mathbb{R}^\alpha \to \mathbb{R}^+$ as 
\begin{equation}\label{eq:m}
m_\mu(x, \theta) = \int_{\mathcal{M}} |\mathfrak{M}(x, y, \theta)| \dd \mu_t(y),
\end{equation}
where $\mathfrak{M}: \mathcal{M} \times \mathcal{M} \times \R^\alpha \to \mathbb{R}$. 
In the following, we will assume $\mathfrak{M}$ in \eqref{eq:m} satisfying \ref{hyp:ker1} and the following assumption:
\begin{enumerate}[label=(\textbf{M})]
    \item\label{hyp:ker2} there exists $c, \bar c>0$ such that $c \leqslant \mathfrak{M}(x, y, \theta) \leqslant \bar c$ for every $x,y \in \mathcal{M}$, $\theta\in \R^\alpha$.
\end{enumerate}

Let us now call $\pi_1: (p, q) \mapsto p$, $\pi_2: (p, q) \mapsto q$ the projections on the first and second component, respectively, and $\iota: p \mapsto (p, \pmb{0})$ the inclusion map such that $\pi_1 \circ \iota = \mathrm{Id}$. We are now ready to define a Wasserstein distance weighted by the mobility $m$ defined in \eqref{eq:m}.

\begin{defi}
For any $t \mapsto \Xi_t \in AC^2\!\big((0,T);\mathcal{P}_2(\R^\alpha)\big)$, we define the \emph{weighted Wasserstein distance} $\mathrm{W}_{m,2}(\cdot,\cdot\,;\Xi)$ between $\mu_0, \mu_T \in \mathcal{P}_2(\mathcal{M})$ as
\begin{align}\label{eq:weigh_wass}
{\mathrm{W}}_{m, 2}^2(\mu_0, \mu_T;\Xi) := \inf\Big\{\int_0^T\iint_{\mathcal{M} \times \R^\alpha} m_\mu(x, \theta) \|\pi_1(v_t(x, \theta))\|^2  & \dd \Xi_t(\theta) \dd \mu_t(x) \dd t: \\
& (\mu_t, v_t) \in \mathrm{CE}_{\theta}((0, T); \mu_0 \rightarrow \mu_T)\Big\},
\end{align}
where
\[
\mathrm{CE}_\theta((0, T); \nu \rightarrow \eta) := \{(\mu_t, v_t) \in \mathrm{CE}_\theta(0, T): \mu_0 = \nu, \mu_T = \eta\}
\]
and
\begin{equation}\label{eq:CE_theta}
\begin{aligned}
\mathrm{CE}_\theta(0,T):=\biggl\{
(\mu_t,v_t)\in&\;\mathrm{AC}(0,T;\mathcal{P}_2(\mathcal{M}))
\times\Gamma([0,T]\times\mathcal{M}\times\R^\alpha,
\mathrm{T}(\mathcal{M}\times \R^\alpha)):\\
&\dot{\mu}_t(x)
=\mathrm{div}_{\mathcal{M}}^x\biggl(
\mu_t(x)\int_{\R^\alpha}
\pi_1(v_t(x,\theta))\,d\Xi_t(\theta)\biggr),\\
&\int_0^T\!\int_{\mathcal{M}}
\biggl\|\int_{\R^\alpha}
\pi_1(v_t(x,\theta))\,\dd\Xi_t(\theta)
\biggr\|\,\dd\mu_t(x)\,\dd t<\infty
\biggr\}.
\end{aligned}
\end{equation}
\end{defi}
Note that we consider vector fields $v_t \in \Gamma([0,T]\times\mathcal{M} \times \R^\alpha, \mathrm{T}(\mathcal{M} \times \R^\alpha))$ that map between the product manifold $\mathcal{M}\times \R^\alpha$ and its tangent bundle. This will be useful to modify the metric depending on the weights. Due to the action of the projection $\pi_1$, the continuity equation above can be interpreted distributionally in $\mathcal{M}$ as in \eqref{eq:riem_ce}. The semicolon notation emphasizes that $\Xi$
enters as a \emph{parameter}: the distance
$\mathrm{W}_{m,2}(\cdot,\cdot\,;\Xi)$ is a metric on
$\mathcal{P}_2(\mathcal{M})$ for each fixed choice of $\Xi$. We note that, when
$\Xi=\delta_{\bar\theta}$ is a Dirac mass, the definition reduces to
the classical weighted Wasserstein distance of
\cite{dolbeault2009new,burger2025analysis} with mobility
$m_\mu(x,\bar\theta)$.

A wider analysis of weighted Wasserstein distances can be found in \cite{dolbeault2009new}, where such objects were first studied for $\mu_0,\mu_T \in \mathcal{P}_2(\mathbb{R}^n)$.

\begin{remark}
We first observe that for any couple $(\mu, v) \in \mathrm{CE}(0, T)$, the vector field defined as $\tilde v_t(x,\theta) = \iota(v_t(x))$ for every  $\theta$ is such that $(\mu, \tilde v) \in \mathrm{CE}_\theta(0, T)$. In particular, this implies that $\mathrm{CE}_\theta(0, T) \ne \emptyset$.
Moreover, we emphasize that the choice $\mathfrak{M} \equiv 1$ trivially satisfies \ref{hyp:ker1}, \ref{hyp:ker2} and, in this case, if we restrict the minimization problem to vector fields of the form $\tilde v_t(x,\theta) = \iota(v_t(x))$
we immediately recover the classical Wasserstein distance $\mathrm{W}_2$. 
\end{remark}

Thanks to \ref{hyp:ker1} and \ref{hyp:ker2}, we can adapt the arguments presented in \cite{burger2025analysis} and conclude the existence of a couple $(\mu, v) \in \mathrm{CE}_\theta(0,T)$ for which the infimum in \eqref{eq:weigh_wass} is actually attained. Moreover, it can also be proven that $\mathrm{W}_{m,2}(\cdot,\cdot\,;\Xi)$ is an actual distance on $\mathcal{P}_2(\mathcal{M})$.

To conclude, we remark that the introduction of the multi-headed
architecture does not allow for a trivial generalization of the
argument presented in \cite{burger2025analysis} to prove the
equivalence of $\mathrm{W}_{m,2}(\cdot,\cdot\,;\Xi)$ and
$\mathrm{W}_2$ and the consequent completeness of
$(\mathcal{P}_2(\mathcal{M}),\mathrm{W}_{m,2}(\cdot,\cdot\,;\Xi))$.
We then explicitly prove these properties in the following proposition.

\begin{prop}\label{prop:equiv_of_norms}
Let $\mathrm{W}_2$ and $\mathrm{W}_{m,2}(\cdot,\cdot\,;\Xi)$ be
as in \eqref{eq:wass} and \eqref{eq:weigh_wass}, respectively.
Then, there exist two positive constants $C_1$, $C_2$ such that
\begin{equation}\label{eq:equiv_wass_wei_wass}
C_1\,\mathrm{W}_2(\mu,\nu)
\leq\mathrm{W}_{m,2}(\mu,\nu;\Xi)
\leq C_2\,\mathrm{W}_2(\mu,\nu),
\qquad\forall\mu,\nu\in\mathcal{P}_2(\mathcal{M}).
\end{equation}
Moreover,
$(\mathcal{P}_2(\mathcal{M}),
\mathrm{W}_{m,2}(\cdot,\cdot\,;\Xi))$
is complete.
\end{prop}
\begin{proof}
    By assumption \ref{hyp:ker2} there exist
constants $0<c\leq\bar{c}<\infty$ such that
\begin{equation}\label{eq:m_bounds}
c\leq m_\mu(x,\theta)\leq\bar{c}
\qquad\forall\,\mu\in\mathcal{P}_2(\mathcal{M}),\;
(x,\theta)\in\mathcal{M}\times\R^\alpha.
\end{equation}

\medskip
\noindent\textbf{Step 1: lower bound.}
Let
$(\mu_t,v_t)\in\mathrm{CE}_\theta((0,T);\mu_0\to\mu_T)$.
Define the averaged velocity
\[
\bar{v}_t(x)
:=\int_{\R^\alpha}\pi_1(v_t(x,\theta))\,\dd\Xi_t(\theta).
\]
By the definition of $\mathrm{CE}_\theta(0,T)$, this is well-defined
for a.e.\ $(x,t)$ and satisfies
$\int_0^T\int_{\mathcal{M}}\|\bar{v}_t(x)\|\,\dd\mu_t(x)\dd t
<\infty$.
Since the continuity equation in $\mathrm{CE}_\theta(0,T)$ reads
$\dot{\mu}_t=\mathrm{div}_{\mathcal{M}}^x(\mu_t\,\bar{v}_t)$,
we have
$(\mu_t,\bar{v}_t)\in\mathrm{CE}((0,T);\mu_0\to\mu_T)$.
By the lower bound $m_\mu\geq c$ and Jensen's inequality
(since $\Xi_t$ is a probability measure and $\|\cdot\|^2$ is
convex), we get:
\[
\begin{aligned}
\iint_{\mathcal{M}\times\R^\alpha}
m_{\mu_t}\|\pi_1(v_t)\|^2\,d\Xi_t\,d\mu_t
&\geq c\int_{\mathcal{M}}
\int_{\R^\alpha}\|\pi_1(v_t(x,\theta))\|^2
\,d\Xi_t(\theta)\,\dd\mu_t(x)\\
&\geq c\int_{\mathcal{M}}
\bigg\|\int_{\R^\alpha}
\pi_1(v_t(x,\theta))\,\dd\Xi_t(\theta)
\bigg\|^2 \dd\mu_t(x)
=c\int_{\mathcal{M}}
\|\bar{v}_t(x)\|^2\,\dd\mu_t(x).
\end{aligned}
\]
Integrating over $[0,T]$, we obtain
\[
\int_0^T\!\iint_{\mathcal{M}\times\R^\alpha}
m_{\mu_t}\|\pi_1(v_t)\|^2\,\dd\Xi_t\,\dd\mu_t\, \dd t
\geq c\int_0^T\!\int_{\mathcal{M}}
\|\bar{v}_t\|^2\,\dd\mu_t\,\dd t
\geq c\,\mathrm{W}_2^2(\mu_0,\mu_T),
\]
where the last inequality holds because
$\mathrm{W}_2^2$ is the infimum over all pairs in
$\mathrm{CE}$ and $(\mu_t,\bar{v}_t)$ is one such pair.
Since this holds for every
$(\mu_t,v_t)\in\mathrm{CE}_\theta((0,T);\mu_0\to\mu_T)$,
taking the infimum yields
$\mathrm{W}_{m,2}^2(\mu_0,\mu_T;\Xi)
\geq c\,\mathrm{W}_2^2(\mu_0,\mu_T)$.

\medskip
\noindent\textbf{Step 2: upper bound.}
Let $(\mu_t,w_t)\in\mathrm{CE}((0,T);\mu_0\to\mu_T)$ be a pair
such that
\[
\int_0^T\!\int_{\mathcal{M}}\|w_t(x)\|^2\,\dd\mu_t(x)\,\dd t<\infty.
\]
Define $v_t(x,\theta):=\iota(w_t(x))$ for every $\theta$, so
that $\pi_1(v_t(x,\theta))=w_t(x)$ is independent of $\theta$.
Since $\Xi_t$ is a probability measure, we have
\[
\int_{\R^\alpha}\pi_1(v_t(x,\theta))\,\dd\Xi_t(\theta)
=w_t(x)\int_{\R^\alpha}\dd\Xi_t(\theta)=w_t(x),
\]
so the continuity equation
$\dot{\mu}_t=\mathrm{div}_{\mathcal{M}}^x
(\mu_t\int_{\R^{\alpha}}\pi_1(v_t)\,d\Xi_t)$ reduces to
$\dot{\mu}_t=\mathrm{div}_{\mathcal{M}}^x(\mu_t\,w_t)$, which
holds by assumption. The integrability condition in
$\mathrm{CE}_\theta$ is similarly inherited from $\mathrm{CE}$. Therefore,
$(\mu_t,v_t)\in\mathrm{CE}_\theta((0,T);\mu_0\to\mu_T)$.
For this $\theta$-constant test field, the upper bound
$m_\mu\leq\bar{c}$ gives
\begin{align}
\iint_{\mathcal{M}\times\R^\alpha}
m_{\mu_t}(x,\theta)\,\|\pi_1(v_t(x,\theta))\|^2
\,\dd\Xi_t(\theta)\,\dd\mu_t(x)
&=\int_{\mathcal{M}}\|w_t(x)\|^2
\int_{\R^{\alpha}}m_{\mu_t}(x,\theta)
\,\dd\Xi_t(\theta)\,\dd\mu_t(x)\\
&\leq\bar{c}\int_{\mathcal{M}}\|w_t(x)\|^2\,\dd\mu_t(x).
\end{align}
Integrating over $[0,T]$ and taking the infimum over
$(\mu_t,w_t)\in\mathrm{CE}((0,T);\mu_0\to\mu_T)$ yields
\[
\mathrm{W}_{m,2}^2(\mu_0,\mu_T;\Xi)
\leq\bar{c}\,\mathrm{W}_2^2(\mu_0,\mu_T).
\]

Combining Steps~1 and~2,
\[
c\,\mathrm{W}_2^2(\mu_0,\mu_T)
\leq\mathrm{W}_{m,2}^2(\mu_0,\mu_T;\Xi)
\leq\bar{c}\,\mathrm{W}_2^2(\mu_0,\mu_T),
\]
which gives \eqref{eq:equiv_wass_wei_wass} with
$C_1:=\sqrt{c}$ and $C_2:=\sqrt{\bar{c}}$. The completeness of
$(\mathcal{P}_2(\mathcal{M}),
\mathrm{W}_{m,2}(\cdot,\cdot\,;\Xi))$ follows
immediately from the completeness of
$(\mathcal{P}_2(\mathcal{M}),\mathrm{W}_2)$ and the equivalence
\eqref{eq:equiv_wass_wei_wass}.

\end{proof}


The following lemma extends \cite[Lemma~2.4]{burger2025analysis} to the
time-dependent multi-head setting. Its key idea is the
identification of the local length structure induced by
$\mathrm W_{m,2}(\cdot,\cdot;\Xi)$ with a weighted Wasserstein structure
on $\S^{d-1}$ in the sense of \cite{dolbeault2009new}, whose effective
mobility is the harmonic mean of $m_\mu(x,\cdot)$ against $\Xi_t$. The
same harmonic mean rescaling will reappear in
Lemma~\ref{lem:chain_rule}, where it arises from the same 
minimization principle (see Remark~\ref{rmk:harmonic_mean_metric} below).

\begin{lemma}[Metric derivative]\label{lem:metric_der}
Let $(\Xi_t)_{t\in(0,T)}\in AC^2((0,T);\mathcal P_2(\R^\alpha))$, and
assume that $m_\mu$ satisfies~\ref{hyp:ker2}. For
$\mu\in\mathcal P_2(\S^{d-1})$, define the \emph{effective mobility}
\begin{equation}\label{eq:effective_mobility}
    b_\mu^{\Xi_t}(x)
    :=
\biggl(\int_{\R^\alpha}\frac{1}{m_\mu(x,\theta)}\,\dd\Xi_t(\theta)\biggr)^{-1},
    \qquad x\in\S^{d-1}.
\end{equation}
Let $(\mu_t)_{t\in(0,T)}\subset\mathcal P_2(\S^{d-1})$ be absolutely
continuous with respect to $\mathrm W_{m,2}(\cdot,\cdot;\Xi)$.
Then there exists a Borel velocity
field $
v\in\Gamma([0,T]\times\S^{d-1}\times\R^\alpha,
\mathrm T(\S^{d-1}\times\R^\alpha))$
such that $(\mu,v)\in\CE_\theta(0,T)$ and, for a.e.\ $t\in(0,T)$,
\begin{equation}\label{eq:metric_derivative_charac_correct}
    |\dot\mu_t|^2
    =
    \int_{\S^{d-1}}
    b_{\mu_t}^{\Xi_t}(x)
    \biggl\|
    \int_{\R^\alpha}
    \pi_1(v_t(x,\theta))
    \,\dd\Xi_t(\theta)
    \biggr\|^2
    \,\dd\mu_t(x).
\end{equation}
Moreover, $v_t$ can be chosen in the horizontal form
\begin{equation}\label{eq:horizontal_form}
    \pi_1(v_t(x,\theta))
    =
    \frac{b_{\mu_t}^{\Xi_t}(x)\,\bar v_t(x)}
    {m_{\mu_t}(x,\theta)},
    \qquad
    \pi_2(v_t(x,\theta))=0,
\end{equation}
for some Borel vector field $\bar v_t(x)\in T_x\S^{d-1}$.

Conversely, if $(\mu,v)\in\CE_\theta(0,T)$ satisfies
\begin{equation}\label{eq:integrability_assumption}
    \int_0^T
    \biggl(
    \iint_{\S^{d-1}\times\R^\alpha}
    m_{\mu_t}(x,\theta)
    \|\pi_1(v_t(x,\theta))\|^2
    \,\dd\Xi_t(\theta)\,\dd\mu_t(x)
    \biggr)^{1/2}
    \dd t
    <+\infty,
\end{equation}
then $t\mapsto\mu_t$ is absolutely continuous with respect to
$\mathrm W_{m,2}(\cdot,\cdot;\Xi)$ 
and for a.e.\ $t\in(0,T)$,
\begin{equation}\label{eq:metric_derivative_upper_correct}
    |\dot\mu_t|^2
    \leq
    \int_{\S^{d-1}}
    \!\!b_{\mu_t}^{\Xi_t}(x)
    \biggl\|
    \int_{\R^\alpha}
    \pi_1(v_t(x,\theta))
    \dd\Xi_t(\theta)
    \biggr\|^2
    \!\!\dd\mu_t(x)
    \leq
    \iint_{\S^{d-1}\times\R^\alpha}
    m_{\mu_t}(x,\theta)
    \|\pi_1(v_t(x,\theta))\|^2
\!\dd\Xi_t(\theta)\!\dd\mu_t(x).
\end{equation}
\end{lemma}
\begin{proof}
    By Proposition~\ref{prop:equiv_of_norms} and \eqref{eq:m_bounds},
\begin{equation}\label{eq:b_bounds}
    c\leq b_\mu^{\Xi_t}(x)\leq\overline c
\end{equation}
uniformly in $\mu$, $x$, and $t$. 
Now, fix $t\in(0,T)$, $x\in\S^{d-1}$, $\mu\in\mathcal P_2(\S^{d-1})$, and
$\bar w\in T_x\S^{d-1}$. Consider the constrained quadratic
minimization
\begin{equation}\label{eq:fiberwise_min}
    \min\biggl\{
    \int_{\R^\alpha}
    m_\mu(x,\theta)
    \|U(\theta)\|^2
    \,\dd\Xi_t(\theta)
    :
    U\in L^2(\Xi_t;T_x\S^{d-1}),
    \quad
    \int_{\R^\alpha} U(\theta)\,\dd\Xi_t(\theta)=\bar w
    \biggr\}.
\end{equation}
Note that
$x$ is frozen, and the only variable is
$U(\theta)$. Since $m_\mu(x,\cdot)\geq c>0$, the cost is strictly convex
in $U$, and first-order optimality characterizes the unique minimizer. In particular, we can write the associated Lagrangian 
\[
    \mathcal L(U,\lambda)
    =
    \int_{\R^\alpha}
    m_\mu(x,\theta)
    \|U(\theta)\|^2
    \,\dd\Xi_t(\theta)
    -
    2\biggl\langle
    \lambda,
    \int_{\R^\alpha}
    U(\theta)\,\dd\Xi_t(\theta)-\bar w
    \biggr\rangle .
\]
Then, first-order optimality with respect to $U$ gives
$ m_\mu(x,\theta)V_x^{\bar w}(\theta)=\lambda$,  $\Xi_t\text{-a.e. in }\theta$, for some $\lambda\in T_x\S^{d-1}$. Substituting into the constraint, we get
\[
    \bar w
    =
    \int_{\R^\alpha}
    \frac{\lambda}{m_\mu(x,\theta)}
    \,\dd\Xi_t(\theta)
    =
    \frac{\lambda}{b_\mu^{\Xi_t}(x)},
    \qquad\text{hence}\qquad
    \lambda=b_\mu^{\Xi_t}(x)\bar w.
\]
The unique minimizer of~\eqref{eq:fiberwise_min} is therefore
\begin{equation}\label{eq:optimal_u_star}
    V_x^{\bar w}(\theta)
    =
    \frac{b_\mu^{\Xi_t}(x)\bar w}
    {m_\mu(x,\theta)},
\end{equation}
with optimal value
\[
    \int_{\R^\alpha}
    m_\mu(x,\theta)
    \|V_x^{\bar w}(\theta)\|^2
    \,\dd\Xi_t(\theta)
    =
    b_\mu^{\Xi_t}(x)\|\bar w\|^2.
\]
For any $(\mu,v)\in\CE_\theta(0,T)$, define
\[
    \bar v_t(x)
    :=
    \int_{\R^\alpha}
    \pi_1(v_t(x,\theta))
    \,\dd\Xi_t(\theta).
\]
Applying 
\eqref{eq:fiberwise_min} with
$U(\theta)=\pi_1(v_t(x,\theta))$ and $\bar w=\bar v_t(x)$, and then
integrating against $\mu_t(x)$, for a.e.\ $t\in(0,T)$, it holds
\begin{equation}\label{eq:fiberwise_minimization_metric_derivative}
    \iint_{\S^{d-1}\times\R^\alpha}
    m_{\mu_t}(x,\theta)
    \|\pi_1(v_t(x,\theta))\|^2
\,\dd\Xi_t(\theta)\,\dd\mu_t(x)
    \geq
    \int_{\S^{d-1}}
    b_{\mu_t}^{\Xi_t}(x)
    \|\bar v_t(x)\|^2
    \,\dd\mu_t(x).
\end{equation}
Equality holds if and only if
\[
\pi_1(v_t(x,\theta))=V_x^{\bar v_t(x)}(\theta)
    =
    \frac{b_{\mu_t}^{\Xi_t}(x)\bar v_t(x)}
    {m_{\mu_t}(x,\theta)}
\]
for $\mu_t\otimes\Xi_t$-a.e.\ $(x,\theta)$, and is realized by the
horizontal lift constructed below. 

Since the continuity equation in $\CE_\theta$ involves $v_t$ only
through $\bar v_t$, for every $(\mu, v) \in \CE_\theta(0, T)$, the pair $(\mu, \bar{v})$ satisfies the standard continuity equation on $\S^{d-1}$, that is, $(\mu, \bar{v}) \in \CE(0, T)$. Vice versa, given
$(\mu,\bar v)\in\CE(0,T)$, define $v_t^*(x,\theta)$ via the optimal
profile~\eqref{eq:optimal_u_star} with $\bar w=\bar v_t(x)$, by
\begin{equation}\label{eq:lift_construction}
    \pi_1(v_t^*(x,\theta))
    :=V_x^{\bar v_t(x)}(\theta)=
    \frac{b_{\mu_t}^{\Xi_t}(x)\bar v_t(x)}
    {m_{\mu_t}(x,\theta)},
    \qquad
    \pi_2(v_t^*(x,\theta)):=0.
\end{equation}
This immediately implies $ \int_{\R^\alpha}
    \pi_1(v_t^*(x,\theta))
    \,\dd\Xi_t(\theta)
    =
    \bar v_t(x),$ hence $(\mu,v^*)\in\CE_\theta(0,T)$. In particular, $v_t^*(x,\theta)$ realizes the equality in \eqref{eq:fiberwise_minimization_metric_derivative}. Combining the lower bound \eqref{eq:fiberwise_minimization_metric_derivative}, which holds for every $(\mu,v)\in\CE_\theta(0,T)$, with the equality realized by $v^*$ of any $(\mu,\bar v)\in\CE(0,T)$, the infimum of the action on $\CE_\theta((0,T);\mu_0\to\mu_T)$ equals the infimum of the reduced action on $\CE((0,T);\mu_0\to\mu_T)$.  By performing the same argument as before on intervals $[s,r]$ with $0 \leq s < r \leq T$, by the definition of $\mathrm W^2_{m,2}$ in~\eqref{eq:weigh_wass} applied on $[s,r]$, it holds that
\begin{equation}\label{eq:reduced_metric_derivative_action}
\mathrm W^2_{m,2}(\mu_s,\mu_r;\Xi)
\;=\;\inf\biggl\{\int_s^r\!\!\iint_{\S^{d-1}\times\R^\alpha}\!\! m_{\mu_t}\|\pi_1(v_t)\|^2\dd\Xi_t\dd\mu_t\dd t\biggr\}
\;=\;\inf\biggl\{\int_s^r\!\!\int_{\S^{d-1}}\!\!b_{\mu_t}^{\Xi_t}\|\bar v_t\|^2\dd\mu_t\dd t\biggr\},
\end{equation}
where the first infimum is over $(\mu,v)\in\CE_\theta((s,r);\mu_s\to\mu_r)$ and the second over $(\mu,\bar v)\in\CE((s,r);\mu_s\to\mu_r)$.

Note that by~\ref{hyp:ker1} together with~\eqref{eq:m_bounds}, the map $\theta\mapsto 1/m_\mu(x,\theta)$ is $L$-Lipschitz on $\R^\alpha$ uniformly in $(\mu,x)$, with $L:=\|\nabla_\theta\mathfrak M\|_\infty/c^2$. Let $\bar t, \tau \in (0,T)$. From the definition~\eqref{eq:effective_mobility}, we have
\begin{equation}\label{eq:b_diff}
b_\mu^{\Xi_\tau}(x)-b_\mu^{\Xi_{\bar t}}(x)
\;=\;b_\mu^{\Xi_\tau}(x)\,b_\mu^{\Xi_{\bar t}}(x)\int_{\R^\alpha}\frac{\dd(\Xi_{\bar t}-\Xi_\tau)(\theta)}{m_\mu(x,\theta)},
\end{equation}
so that, by Kantorovich--Rubinstein duality and $b^{\Xi_\tau}_\mu b^{\Xi_{\bar t}}_\mu\le\overline c^{\,2}$,
\begin{equation}\label{eq:b_lip}
|b_\mu^{\Xi_\tau}(x)-b_\mu^{\Xi_{\bar t}}(x)|\;\le\;\overline c^{\,2}L\,\mathrm W_2(\Xi_\tau,\Xi_{\bar t}),\qquad\forall\,\mu,x,\tau,\bar t.
\end{equation}
Now, let $h>0$ and $(\mu^h,\bar v^h)\in\CE((\bar t,\bar t+h);\mu_{\bar t}\to\mu_{\bar t+h})$ be optimal for $\mathrm W^2_{m,2}(\mu_{\bar t},\mu_{\bar t+h};\Xi)$ in~\eqref{eq:reduced_metric_derivative_action}. Adding and subtracting $b_{\mu^h_t}^{\Xi_{\bar t}}(x)$ inside the integral, we get
\begin{align}
\mathrm W^2_{m,2}(\mu_{\bar t},\mu_{\bar t+h};\Xi)
&=\int_{\bar t}^{\bar t+h}\!\!\int_{\S^{d-1}} b_{\mu^h_t}^{\Xi_t}\|\bar v^h_t\|^2\dd\mu^h_t\dd t\notag\\
&=\underbrace{\int_{\bar t}^{\bar t+h}\!\!\int_{\S^{d-1}} b_{\mu^h_t}^{\Xi_{\bar t}}\|\bar v^h_t\|^2\dd\mu^h_t\dd t}_{:= I_1}
+\underbrace{\int_{\bar t}^{\bar t+h}\!\!\int_{\S^{d-1}}\bigl(b_{\mu^h_t}^{\Xi_t}-b_{\mu^h_t}^{\Xi_{\bar t}}\bigr)\|\bar v^h_t\|^2\dd\mu^h_t\dd t}_{:=I_2}.
\label{eq:add_subtract}
\end{align}
Note that since we consider the optimal $(\mu^h,\bar v^h)$, by~\eqref{eq:b_lip} and the bound $b^{\Xi_{\bar t}}\ge c$, we have
\begin{align}
|I_2|
&\le\overline c^{\,2}L\sup_{\bar t\le t\le \bar t+h}\!\!\mathrm W_2(\Xi_t,\Xi_{\bar t})\int_{\bar t}^{\bar t+h}\!\!\int_{\S^{d-1}}\|\bar v^h_t\|^2\dd\mu^h_t\dd t\notag\\
&\le\frac{\overline c^{\,2}L}{c}\sup_{\bar t\le t\le \bar t+h}\!\!\mathrm W_2(\Xi_t,\Xi_{\bar t})\;I_1
\;:=\;\omega(\bar t,h)\,I_1,
\label{eq:cross_bound}
\end{align}
where $\omega(\bar t,h):=(\overline c^{\,2}L/c)\sup_{|\tau-\bar t|\le h}\mathrm W_2(\Xi_\tau,\Xi_{\bar t})\xrightarrow[h\to 0^+]{}0$ by absolute continuity of $\Xi$. $I_1$ is not directly identifiable as a metric derivative since $(\mu^h,\bar v^h)$ is optimal for the time-dependent problem, not for the autonomous frozen one. To fix this, denote by $\mathrm W_{m,2}^{\,\bar t}(\cdot,\cdot;\Xi)$ the right-hand side of~\eqref{eq:reduced_metric_derivative_action} with $\Xi_{\bar t}$ in place of $\Xi_t$. Since, the mobility 
$\mu\mapsto b^{\Xi_{\bar t}}_\mu\in[c,\overline c]$ is autonomous, it holds~\cite[Lemma~2.4]{burger2025analysis} which provides a Borel field $\bar v^{\bar t}$ with $(\mu,\bar v^{\bar t})\in\CE(0,T)$ and
\begin{equation}\label{eq:frozen_metric_derivative_identity}
|\dot\mu_t|_{\bar t}^{\,2}=\int_{\S^{d-1}}b_{\mu_t}^{\Xi_{\bar t}}\|\bar v_t^{\bar t}\|^2\dd\mu_t\quad\text{a.e.\ }t,
\end{equation}
where $|\dot\mu_t|_{\bar t}$ is the metric derivative in $t$ of $\mathrm W_{m,2}^{\,\bar t}(\cdot,\cdot;\Xi)$. 
Since $(\mu^h,\bar v^h)$ is admissible for realizing the infimum in $\mathrm W_{m,2}^{\,\bar t}$, we have $I_1\ge\mathrm W_{m,2}^{2,\bar t}(\mu_{\bar t},\mu_{\bar t+h};\Xi)$.  Combined with~\eqref{eq:add_subtract} and~\eqref{eq:cross_bound}, we get
\begin{equation}\label{eq:lower_squeeze}
\mathrm W_{m,2}^2(\mu_{\bar t},\mu_{\bar t+h};\Xi)\;=\;I_1+I_2\;\ge\;(1-\omega(\bar t,h))\,I_1\;\ge\;(1-\omega(\bar t,h))\,\mathrm W_{m,2}^{2,\bar t}(\mu_{\bar t},\mu_{\bar t+h};\Xi).
\end{equation}
For the upper bound, take instead the optimal pair $(\tilde\mu,\tilde v)$ of $\mathrm W_{m,2}^{2,\bar t}(\mu_{\bar t},\mu_{\bar t+h};\Xi)$, and run the same argument~\eqref{eq:add_subtract}--\eqref{eq:cross_bound} with the roles of $\Xi_t,\Xi_{\bar t}$ exchanged:
\begin{equation}\label{eq:upper_squeeze}
\mathrm W_{m,2}^{2,\bar t}(\mu_{\bar t},\mu_{\bar t+h};\Xi)\;\ge\;(1-\omega(\bar t,h))\int_{\bar t}^{\bar t+h}\!\!\int_{\S^{d-1}}\!\!b_{\tilde\mu_t}^{\Xi_t}\|\tilde v_t\|^2\dd\tilde\mu_t\dd t\;\ge\;(1-\omega(\bar t,h))\,\mathrm W_{m,2}^2(\mu_{\bar t},\mu_{\bar t+h};\Xi),
\end{equation}
where the last inequality holds by the admissibility of $(\tilde\mu,\tilde v)$ for $\mathrm W_{m,2}^2$. Hence, together ~\eqref{eq:lower_squeeze} and~\eqref{eq:upper_squeeze} give 
\begin{equation}\label{eq:two_sided_squeeze}
(1-\omega(\bar t,h))\,\mathrm W_{m,2}^{2,\bar t}(\mu_{\bar t},\mu_{\bar t+h};\Xi)\;\le\;\mathrm W_{m,2}^2(\mu_{\bar t},\mu_{\bar t+h};\Xi)\;\le\;\frac{\mathrm W_{m,2}^{2,\bar t}(\mu_{\bar t},\mu_{\bar t+h};\Xi)}{1-\omega(\bar t,h)}.
\end{equation}
Dividing~\eqref{eq:two_sided_squeeze} by $h^2$ and letting $h\to 0^+$,  by~\eqref{eq:frozen_metric_derivative_identity}, we obtain
\begin{equation}\label{eq:metric_derivative_freezing}
|\dot\mu_{\bar t}|^2\;=\;|\dot\mu_{\bar t}|_{\bar t}^{\,2}\;=\;\int_{\S^{d-1}}b_{\mu_{\bar t}}^{\Xi_{\bar t}}\|\bar v_{\bar t}^{\bar t}\|^2\dd\mu_{\bar t}\qquad\text{a.e.\ }\bar t\in(0,T).
\end{equation}
The previous identity \eqref{eq:metric_derivative_freezing} gives the pointwise integrand value at the diagonal $t=\bar t$ via a $\bar t$-indexed family $\{\bar v^{\bar t}\}_{\bar t\in(0,T)}$ of distinct Borel fields; instead we require one Borel field $\bar v$ on $(0,T)\times\S^{d-1}$ with $(\mu,\bar v)\in\CE(0,T)$ realising $|\dot\mu_t|^2=\int b^{\Xi_t}_{\mu_t}\|\bar v_t\|^2\dd\mu_t$ a.e.\ $t$, for which neither the joint Borel measurability of the diagonal selection $(t,x)\mapsto\bar v^t_t(x)$ nor its CE are automatic. We obtain such a $\bar v$ by a piecewise frozen approximation, the time-dependent analogue of the argument used in~\cite[Appendix~A.3]{burger2025analysis} to prove their Lemma~2.4 in the autonomous case. In particular, for $N\in\mathbb N$, set the step size $\tau:=T2^{-N}$, take the partition $\{t^N_k:=k\tau\}_{k=0}^{2^N}$ of $[0,T]$, and define
$
\bar v^N_t:=\bar v^{t^N_k}_t$ for $t\in[t^N_{k-1},t^N_k),$ Borel by construction. Gluing the CE solutions via~\cite[Proposition~A.1]{burger2025analysis} gives $(\mu,\bar v^N)\in\CE(0,T)$. Replacing $b^{\Xi_t}_{\mu_t}$ by $b^{\Xi_{t^N_k}}_{\mu_t}$ inside each piece and using~\eqref{eq:frozen_metric_derivative_identity}, gives
\[
\int_{t^N_{k-1}}^{t^N_k}\!\!\int_{\S^{d-1}}b^{\Xi_{t^N_k}}_{\mu_t}\|\bar v^{t^N_k}_t\|^2\dd\mu_t\dd t\;=\;\int_{t^N_{k-1}}^{t^N_k}|\dot\mu_t|^2_{t^N_k}\dd t.
\]
Since $|b_{\mu_t}^{\Xi_t}(x)-b_{\mu_t}^{\Xi_{t_k^N}}(x)|$ is controlled by~\eqref{eq:b_lip} and it holds $\int\|\bar v^{t^N_k}_t\|^2\dd\mu_t\le|\dot\mu_t|^2_{t^N_k}/c$, we obtain:
\[
\biggl|\int_0^T\!\!\int_{\S^{d-1}}b^{\Xi_t}_{\mu_t}\|\bar v^N_t\|^2\dd\mu_t\dd t\;-\;\int_0^T|\dot\mu_t|^2_{t^N_k(t)}\dd t\biggr|\;\le\;\frac{\overline c^{\,2}L}{c}\sup_{|t-t^N_k(t)|\le 2^{-N}T}\!\!\!\mathrm W_2(\Xi_t,\Xi_{t^N_k(t)})\int_0^T|\dot\mu_t|^2_{t^N_k(t)}\dd t,
\]
where $t^N_k(t)$ denotes the right endpoint of the piece containing $t$; the supremum vanishes as $N\to\infty$ since $\Xi\in AC^2$. As $\bar t\mapsto|\dot\mu_t|^2_{\bar t}$ is continuous at $\bar t=t$  for a.e.\ $t$ by~\eqref{eq:b_lip}, dominated convergence and~\eqref{eq:metric_derivative_freezing} give
\[
\int_0^T|\dot\mu_t|^2_{t^N_k(t)}\dd t\;\xrightarrow[N\to\infty]{}\;\int_0^T|\dot\mu_t|^2_{t}\dd t\;=\;\int_0^T|\dot\mu_t|^2\dd t,
\]
hence the action of $\bar v^N$ converges to $\int_0^T|\dot\mu_t|^2\dd t$. The compactness~\cite[Lemma~A.2]{burger2025analysis} extracts a weak $L^2$ limit $\bar v$ on $(0,T)\times\S^{d-1}$ with $(\mu,\bar v)\in\CE(0,T)$. Then,  lower semicontinuity of $\int_0^T \int_{\mathbb{S}^{d-1}} b_{\mu_t}^{\Xi_t}\left\|\bar{v}_t\right\|^2 \dd\mu_t \dd t$ together with $|\dot\mu_t|^2\le\int b^{\Xi_t}_{\mu_t}\|\bar v_t\|^2\dd\mu_t$ a.e. ($\bar v$ is an admissible competitor for $\mathrm{W}_{m, 2}^2\left(\mu_t, \mu_{t+h} ; \Xi\right)$), yields
\begin{equation}\label{eq:reduced_metric_derivative_identity}
|\dot\mu_t|^2\;=\;\int_{\S^{d-1}}b^{\Xi_t}_{\mu_t}\|\bar v_t\|^2\dd\mu_t\qquad\text{a.e.\ }t\in(0,T).
\end{equation}
 In particular,  the diagonal of~\eqref{eq:metric_derivative_freezing} identifies $\bar v_t=\bar v^t_t$ in $L^2$ for a.e.\ $t$. Finally, the lift~\eqref{eq:lift_construction} of $\bar v_t$ produces a Borel field $v_t$ with $(\mu,v)\in\CE_\theta(0,T)$ and $\int_{\R^\alpha}\pi_1(v_t)\,\dd\Xi_t = \bar v_t$, and~\eqref{eq:metric_derivative_charac_correct} follows from~\eqref{eq:reduced_metric_derivative_identity}; the  horizontal form~\eqref{eq:horizontal_form} is exactly~\eqref{eq:lift_construction}.

For the converse, given $(\mu,v)\in\CE_\theta(0,T)$ satisfying~\eqref{eq:integrability_assumption}, set $\bar v_t:=\int\pi_1(v_t)\dd\Xi_t$, so $(\mu,\bar v)\in\CE(0,T)$ and, by~\eqref{eq:fiberwise_minimization_metric_derivative}, it holds
\[\int_{\S^{d-1}}
    b_{\mu_t}^{\Xi_t}(x)
    \|\bar v_t(x)\|^2
    \,\dd\mu_t(x)<+\infty.
    \]
  The converse direction of~\cite[Lemma~2.4]{burger2025analysis} applied to $\mathrm W_{m,2}^{\,\bar t}$, combined with~\eqref{eq:two_sided_squeeze}, gives absolute continuity of $\mu$ in $\mathrm W_{m,2}(\cdot,\cdot;\Xi)$ and
  \[
    |\dot\mu_t|^2
    \leq
    \int_{\S^{d-1}}
    b_{\mu_t}^{\Xi_t}(x)
    \|\bar v_t(x)\|^2
    \,\dd\mu_t(x)
    \qquad\text{for a.e.\ }t\in(0,T).
\]
This is the first inequality in~\eqref{eq:metric_derivative_upper_correct}, while the second is~\eqref{eq:fiberwise_minimization_metric_derivative}.
\end{proof}

\begin{remark}\label{rmk:harmonic_mean_metric}
The effective mobility $b_\mu^{\Xi_t}(x)$ defined in
\eqref{eq:effective_mobility} is the $\Xi_t$-weighted harmonic mean of
$m_\mu(x,\cdot)$. It arises directly from the fiberwise minimization
\eqref{eq:fiberwise_min} via the Lagrange-multiplier computation leading
to \eqref{eq:optimal_u_star}. The horizontal form
\eqref{eq:horizontal_form} is precisely the optimal profile
$V_x^{\bar v_t(x)}$ realizing this minimum, where
horizontality refers to $v_t$ having no component along the parameter
direction $\R^\alpha$.

The same harmonic mean rescaling, resulting from a structurally identical
minimization, will reappear in Lemma~\ref{lem:chain_rule} when
we descend the gradient $\nabla^{\mathfrak L^m}f$ on the product
manifold to an effective gradient $\nabla^m f$ on $\S^{d-1}$. In the
single-head case $\Xi_t=\delta_{\bar\theta}$,
$b_\mu^{\Xi_t}(x)=m_\mu(x,\bar\theta)$, and the reduced formulation
recovers the weighted Wasserstein structure of
\cite{burger2025analysis}.
\end{remark}

\subsection{Existence of Minimizing Movements}
We are now ready to define minimizing movements for the time-dependent interaction energy defined in \eqref{eq:time_dep_interaction_energy}. Note that throughout the rest of the paper, we specialize the weighted
Wasserstein framework of Section~\ref{sec:w_wass_dist} to the
case $\Xi_t=\Theta_t$, where $\Theta_t$ is the distribution of
head parameters provided by the second equation in the
Transformer PDE \eqref{eq:continuity_eq_full}. For brevity, we will omit the dependence on $\Theta$ from the notation whenever the choice $\Xi=\Theta$ is clear from the context, writing $\mathrm{W}_{m, 2}$ in place of $\mathrm{W}_{m, 2}(\cdot, \cdot ; \Theta)$ and  $b_\mu(x)$ in place of $b_\mu^{\Theta}(x)$.

Thanks to the equivalence result provided by Proposition~\ref{prop:equiv_of_norms}, we can first consider our problem in $(\mathcal{P}_2(\S^{d-1}),\mathrm{W}_2)$ and then obtain similar results for $(\mathcal{P}_2(\S^{d-1}),\mathrm{W}_{m,2})$.

\begin{lemma}\label{lem:existence_of_minimizer}
Fix $t > 0$ and let \(\nu\in\mathcal{P}_2(\mathbb{S}^{d-1})\). Then, for every \(|\tau|>0\), there exists \(u\in\mathcal{P}_2(\mathbb{S}^{d-1})\) minimizing
\begin{equation}\label{eq:Euler_min_prob}
\frac{{\mathrm{W}}_2^2(u, \nu)}{2|\tau|} + \mathcal{E}(u, t).
\end{equation}
\end{lemma}

\begin{proof}
Let us first observe that $\mathcal{E}$ is bounded from below. Moreover, it is also continuous so that, in particular, it is lower semicontinuous with respect to the narrow convergence. Finally, recalling the properties of the Wasserstein distance stated in Lemma \ref{lem:W2-prelim}, the thesis follows by the direct method of calculus of variations.
\end{proof}

Fix $T>0$ and a partition $\tau=\{0=t_0<t_1<\cdots<t_N=T\}$, and set the mesh size
$
|\tau| := \max_{1\leq k\leq N} (t_k-t_{k-1}).
$
By Lemma \ref{lem:existence_of_minimizer}, for any $u_0 \in \mathcal{P}_2(\mathbb{S}^{d-1})$ we set $u_\tau^0:=u_0$ and construct $(u_\tau^k)_{k=1}^N$ recursively by minimizing~\eqref{eq:Euler_min_prob} with $\nu=u_\tau^{k-1}$ and $t=t_k$. Then, the piecewise constant interpolation $u_\tau: [0, T] \to \mathcal{P}_2(\mathbb{S}^{d - 1})$ is well defined by
\begin{equation}\label{eq:piece_wise_linear}
u_\tau(t) = \left\{
\begin{aligned}
& u_0, \qquad  t = 0,\\
& u_\tau^k, \qquad t \in (t_{k - 1}, t_k], \quad k=1,\ldots,N.
\end{aligned}
\right.
\end{equation}

As for standard minimizing movements, we can prove that for $|\tau| \rightarrow 0$ the constructed sequence converges (up to subsequence) to an absolutely continuous limit curve.

\begin{thm}\label{thm:existence_of_GMMs}
Let $T>0$ and let $(\tau_n)_{n\in\N}$ be a family of partitions of $[0,T]$ with $|\tau_n|\to 0$ as $n\to\infty$. Then, there exists a (non-relabeled) subsequence $(u_{\tau_n})_{n\in\N}$ and a limit curve $u\in AC^2([0,T];\mathcal{P}_2(\S^{d-1}))$ such that
\begin{equation}
u_{\tau_n}(t)\rightharpoonup u(t)
\end{equation}
narrowly in $\mathcal{P}_2(\S^{d-1})$ for every $t\in[0,T]$.
\end{thm}

\begin{proof}
By minimality of $u_\tau^k$ with respect to \eqref{eq:Euler_min_prob}, it follows that
\[
\frac{{\mathrm{W}}_2^2(u_\tau^k, u_\tau^{k - 1})}{2|\tau|} + \mathcal{E}(u_\tau^k, t_k) \leq \mathcal{E}(u_\tau^{k - 1}, t_k).
\]
Then, summing over $k=1, \ldots, N$, gives
\begin{equation}\label{eq:upper_bound_Euler}
\mathcal{E}(u_\tau^N, T) + \sum_{k = 1}^{N} \frac{{\mathrm{W}}_2^2(u_\tau^k, u_\tau^{k - 1})}{2 |\tau|} \leq \mathcal{E}(u_0, 0) + \sum_{k = 0}^{N - 1} \left[\mathcal{E}(u_\tau^k, t_{k + 1}) - \mathcal{E}(u_\tau^k, t_k)\right].
\end{equation}

Using the fundamental theorem of calculus on each $\left[t_k, t_{k+1}\right]$, the definition of $\mathcal{E}$ in \eqref{eq:time_dep_interaction_energy}, and Fubini, we get
\begin{align}
\sum_{k=0}^{N-1}\bigl[\mathcal{E}(u_\tau^k,t_{k+1})-\mathcal{E}(u_\tau^k,t_k)\bigr]
&= \sum_{k=0}^{N-1}\int_{t_k}^{t_{k+1}}\partial_t\mathcal{E}(u_\tau^k,t)\,\dd t \\
&= \frac{1}{2}\int_0^T\iint_{\S^{d-1}\times\S^{d-1}}\frac{\dd}{\dd t}\int_{\R^\alpha}\mathscr{K}(x,y,\theta)\,\dd\Theta_t(\theta)\,\dd u_\tau(t)(x)\,\dd u_\tau(t)(y)\,\dd t.
\end{align}

Therefore, by the triangle inequality, the weak formulation of the $\Theta_t$-dynamics \eqref{eq:continuity_eq_full}, and \ref{hyp:ker4}, we have
\begin{equation}
\begin{aligned}
\left|\sum_{k = 0}^{N - 1} \mathcal{E}(u_\tau^k, t_{k + 1}) - \mathcal{E}(u_\tau^k, t_k)\right| & \leq \frac{1}{2} \int_{0}^{T} \left|\iint_{\mathbb{S}^{d-1} \times \mathbb{S}^{d-1}} \frac{\dd}{\dd t}\left(\int_{\R^{\alpha}} \mathscr{K}(x, y, \theta) \dd \Theta_{t}(\theta)\right) \dd u_\tau(x) \dd u_\tau(y)\right| \dd t \\
& \leq \frac{1}{2} \int_{0}^{T} \iiint_{\mathbb{S}^{d-1} \times \mathbb{S}^{d-1} \times \R^{\alpha}} |\nabla_\theta \mathscr{K}(x, y, \theta)|\, |f(\theta, t)| \, +\\
& \qquad\qquad\qquad\,\,\, + \,\frac{1}{2} |\nabla_\theta^2 \mathscr{K}(x, y, \theta)|\,|g(\theta, t)|^2 \dd \Theta_{t}(\theta) \dd u_\tau(x) \dd u_\tau(y) \dd t\\
& \leq C \Big{(}\|f\|_{L^1([0, T];\, L^\infty(\R^{\alpha}))} \, + \,\|g\|_{L^2([0, T];\, L^\infty(\R^{\alpha}))}^2\Big{)}=: C_{f g},
\end{aligned}
\end{equation}
where $C$ depends only on $\sup _{(x, y, \theta)}\left|\nabla_\theta \mathscr{K}\right|$ and $\sup _{(x, y, \theta)}\left|\nabla_\theta^2 \mathscr{K}\right|$, which  are finite by assumptions \ref{hyp:ker4} and \ref{hyp:ker1}.

Setting then $N = \max\{n \in \mathbb{N}: n|\tau| \leq T\}$, the previous inequality implies
\begin{equation}\label{eq:est_der_u}
\frac{1}{2} \int_{0}^{T} \zeta_\tau^2(t) \, \dd t \leq \mathcal{E}(u_0, 0) - \mathcal{E}(u_\tau(T), T) + C_{fg},
\end{equation}
where
\begin{equation}\label{eq:zeta}
\zeta_\tau(t):=\frac{{\mathrm{W}}_2\left(u_\tau(t), u_\tau(t-|\tau|)\right)}{|\tau|},
\end{equation}
 with the convention $u_\tau(t-|\tau|)=u_0$ when $t-|\tau| \leq 0$.  Since $\mathbb{S}^{d-1}$ is compact and $\mathcal{E}(\cdot, T)$ is bounded from below, we have 
 \[
\frac{1}{2} \int_{0}^{T} \zeta_\tau^2(t) \, \dd t \leq \bar C,
\]
 that is $\left(\zeta_\tau\right)$ is uniformly bounded in $L^2(0, T)$. Hence, up to a (non-relabeled) subsequence,
\begin{align}
\zeta_\tau \rightharpoonup \zeta \quad \text { weakly in } L^2(0, T)
\end{align}
for some $\zeta \in L^2(0, T)$ with $\frac{1}{2} \int_0^T \zeta^2 \leq \bar C$. Because $\mathbb{S}^{d-1}$ is compact, $\mathcal{P}_2(\mathbb{S}^{d-1})$ is narrowly compact. In particular,
\begin{equation}
\mathcal{U}_T := \{u_\tau^k: \text{$|\tau| > 0$, $k \in \mathbb{N}$  s.t.  $ k|\tau| \leq T$}\}
\end{equation}
is narrowly relatively compact. Moreover, for any $0 \leq s \leq r \leq T$,
\[
\begin{aligned}
\limsup_{|\tau| \rightarrow 0} {\mathrm{W}}_2(u_\tau(r), u_\tau(s)) & \leq  \limsup_{|\tau| \rightarrow 0} \sum_{j = \left\lfloor \frac{s}{|\tau|}
 \right\rfloor}^{\left\lfloor \frac{r}{|\tau|} \right\rfloor} \frac{{\mathrm{W}}_2(u_\tau^j, u_\tau^{j - 1})}{|\tau|} |\tau| \leq \limsup_{|\tau| \rightarrow 0} \int_{\left\lfloor \frac{s}{|\tau|} \right\rfloor |\tau|}^{\left\lfloor \frac{r}{|\tau|} \right\rfloor |\tau|} \zeta_\tau(t) \dd t\\
& \leq \int_{s}^{r} \zeta(t) \dd t \leq  \left(\int_{0}^{T} \zeta^2(t) \dd t\right)^\frac{1}{2} \, |r - s|^\frac{1}{2} \leq \sqrt{2 \bar C} \, |r - s|^\frac{1}{2}.
\end{aligned}
\]
 Hence, we can apply the refined version of Ascoli--Arzelà Theorem \cite[Prop. 3.3.1]{ambrosio2008gradient}, obtaining the desired convergence.
\end{proof}

Similarly to the construction made in
the classical Wasserstein setting, we now define time-dependent minimizing
movements with respect to
$\mathrm{W}_{m,2}$ by
\begin{equation}\label{eq:min_mov_weighted}
u_\tau^0:=u_0,\qquad
u_\tau^k\in\operatorname*{argmin}_{u\in
\mathcal{P}_2(\mathbb{S}^{d-1})}
\left\{
\mathcal{E}(u,t_k)
+\frac{1}{2\tau}\,
\mathrm{W}_{m,2}^2(u,u_\tau^{k-1})
\right\},
\quad k\geq 1.
\end{equation}
We denote by
$u_\tau:[0,T]\to\mathcal{P}_2(\mathbb{S}^{d-1})$ the
piecewise constant interpolation of the discrete sequence
$(u_\tau^k)_k$. Relying on the equivalence of the Wasserstein distances in Proposition \ref{prop:equiv_of_norms}, the existence of minimizing movements follows similarly to
Theorem~\ref{thm:existence_of_GMMs}. In particular, there exists a subsequence $(u_{\tau_n})_{n\in\N}$ and a limit curve $u\in AC^2([0,T];(\mathcal{P}_2(\S^{d-1}),\mathrm{W}_{m,2}))$ such that
\begin{equation}
u_{\tau_n}(t) \rightharpoonup u(t)\qquad\text{narrowly in }\mathcal{P}_2(\S^{d-1})\quad\text{for every }t\in[0,T].
\end{equation}

\section{Gradient flow characterization of Transformers PDE}\label{sec:grad_flow_PDE}

In this section, we provide a full characterization of the flows satisfying \eqref{eq:full_transformer_ODE}. We begin by establishing an interchanging flow estimate, which allows us to prove that the gradient flows of $\mathcal{E}$ are tangent to $\nabla_{\mathbb{S}^{d - 1}}^x\mathscr{K}(x, y, \theta)$ and satisfy a suitable continuity equation. We then demonstrate that specific choices of $\mathscr{K}$ recover the input-output flow of a transformer architecture with unnormalized attention $\sigma(x, y, \theta) B(x, y,\theta)$ (Corollary \ref{cor:our_model}). Furthermore, leveraging the weighted Wasserstein defined in Section \ref{sec:w_wass_dist}, we derive \eqref{eq:continuity_eq_full}. Conversely, we establish that solutions to \eqref{eq:continuity_eq_full} satisfy the energy-dissipation identity \eqref{eq:EDI_td}, thereby completing the characterization. Finally, we analyze the long-time behavior of these gradient flows and validate the theory with numerical experiments.

\subsection{Interchanging Flows}
Let us consider a complete metric space $(X, d)$ and introduce the so-called \textit{interchanging flows estimate} \cite{MMCS}. Heuristically, this method is based on the idea that the dissipation of a functional along the gradient flow of another functional equals the dissipation of the second functional along the gradient flow of the first one.

We start by recalling the following definitions and properties. We denote $\operatorname{Dom}(\mathscr{M}):=\{\mu \in X: \mathscr{M}(\mu)<\infty\}$.
\begin{defi}\cite[Eq.~3.15]{MMCS}\label{def:beta_flow}
Let $\mathscr{M}: X \rightarrow (-\infty,+\infty]$ be a proper and lower semicontinuous functional and fix $\beta \in \mathbb{R}$. We say that $\mathscr{M}$ \emph{generates a $\beta$-flow} if there exists a continuous semigroup $S_s^{\mathscr{M}}: \operatorname{Dom}(\mathscr{M}) \rightarrow \operatorname{Dom}(\mathscr{M})$, $s\geq 0$, such that, for every $\mu, \nu \in \operatorname{Dom}(\mathscr{M})$ and every $s\geq 0$, the following evolution variational inequality (EVI) holds:
\begin{equation}\label{eq:flow_ineq}
\underset{h \searrow 0}{\limsup} \frac{d^2\left(S_{s+h}^{\mathscr{M}}(\mu), \nu\right)-d^2\left(S_s^{\mathscr{M}}(\mu), \nu\right)}{2 h}+\frac{\beta}{2} d^2\left(S_s^{\mathscr{M}}(\mu), \nu\right) \leq \mathscr{M}(\nu)-\mathscr{M}\left(S_s^{\mathscr{M}}(\mu)\right).
\end{equation}
\end{defi}
\begin{remark}
    We refer to $S_h^\mathscr{M}$ as the semigroup associated to the flow generated by $\mathscr{M}$, which satisfies 
\begin{align}
S_{s+h}^{\mathscr{M}}=S_h^{\mathscr{M}} \circ S_s^{\mathscr{M}}, \quad \lim _{h \searrow 0} S_h^{\mathscr{M}}(\mu)=\mu \quad \forall \mu \in \operatorname{Dom}(\mathscr{M}), s\geq 0.
\end{align}
    We note that the case $\beta>0$ corresponds to contractive flows with rate $\beta$, while $\beta=0$ is the standard EVI formulation for gradient flows of 0-convex functionals.
\end{remark}
\begin{defi}\label{def:dissipation_def}
Let $\mathscr M:X\to(-\infty,+\infty]$ generate a $\beta$–flow with semigroup $(S_h^{\mathscr M})_{h\geq 0}$ as in Definition \ref{def:beta_flow}. Let $\mathscr N:X\times[0,\infty)\to(-\infty,+\infty]$ be proper and l.s.c. in the first variable, and assume
$\operatorname{Dom}(\mathscr N(\cdot,t))\subset \operatorname{Dom}(\mathscr M)$ for every $t\geq 0$. 
For a fixed $t\geq 0$, the \emph{dissipation of $\mathscr N$} along the $\beta$–flow at $\mu$ is
\begin{equation}\label{eq:dissipation}
\mathfrak D_{\mathscr M}\mathscr N(\mu,t)
:= \limsup_{h\searrow 0}\frac{\mathscr N(\mu,t)-\mathscr N\!\big(S_h^{\mathscr M}(\mu),\,t\big)}{h},
\end{equation}
for every $\mu\in\Dom(\mathscr N(\cdot,t))$.
\end{defi}

We can then prove the following result, which is simply the adapted version of \cite[Theorem 3.2]{MMCS} in our time-dependent case for $(X,d)$ being $(\mathcal{P}_2(\mathbb{S}^{d-1}),\mathrm{W}_2)$. The proof follows exactly as the proof of \cite[Equation 3.21]{MMCS}.
\begin{lemma}\label{lem:flow_interchange_estimate}
Let $\mathcal{E}$ be as in  \eqref{eq:time_dep_interaction_energy}. Fix $u_0 \in \mathcal{P}_2(\mathbb{S}^{d-1})$ and let $\left\{u_\tau^k\right\}_{k \in \mathbb{N}}$ be the minimizing movement sequence generated by recursively minimizing \eqref{eq:Euler_min_prob} with $\nu=u_\tau^{k-1}$. For each $t \in[0, T]$, let $\mathscr{M}_t: X \rightarrow(-\infty,+\infty]$ generate a $\beta$-flow with semigroup $(S_h^{\mathscr{M}_t})_{h \geq 0}$. Then, for every $k \geq 1$ and $t \in[0, T]$, we have
\begin{equation}\label{eq:interch_flow_est_E}
\mathfrak D_{\mathscr M_t}\mathcal E(u_\tau^k,t)+ \beta \frac{\mathrm{W}_2^2(u_\tau^k, u_\tau^{k - 1})}{2 |\tau|} \leq \frac{\mathscr M_t(u_\tau^{k-1})-\mathscr M_t(u_\tau^k)}{|\tau|}.
\end{equation}
\end{lemma}
\begin{proof}
Fix $k$ and $t$. By the minimality of the sequence $\{u_\tau^k\}_{k \in \mathbb{N}}$, for every $h> 0$ it holds
\begin{equation}
\mathcal{E}(u_\tau^k, t) - \mathcal{E}(S_h^{\mathscr{M}_{t}}(u_\tau^k), t) \leq \frac{\mathrm{W}_2^2(S_h^{\mathscr{M}_{t}}(u_\tau^k), u_\tau^{k - 1}) - \mathrm{W}_2^2(u_\tau^k, u_\tau^{k - 1})}{2|\tau|}.
\end{equation}
Dividing by $h$, adding $\frac{\beta}{2|\tau|} d^2\left(u_\tau^k, u_\tau^{k-1}\right)$ to both sides and taking the $\limsup _{h \searrow 0}$ we get
\begin{align}
\mathfrak D_{\mathscr M_t}\mathcal E(u_\tau^k,t) + \beta \frac{\mathrm{W}_2^2(u_\tau^k, u_\tau^{k - 1})}{2 |\tau|} \leq \frac{1}{|\tau|}\limsup_{h\searrow 0}  \Bigg{(}\frac{\mathrm{W}_2^2(S_h^{\mathscr{M}_{t}}(u_\tau^k), u_\tau^{k - 1}) - \mathrm{W}_2^2(u_\tau^k, u_\tau^{k - 1})}{2h}+ \frac{\beta}{2} \mathrm{W}_2^2(u_\tau^k, u_\tau^{k - 1})\Bigg{)}.
\end{align}
By the EVI \eqref{eq:flow_ineq} for $\mathscr M_t$ with $\mu=u_\tau^k$ and $\nu=u_\tau^{k-1}$, we also have
\begin{equation}
    \limsup _{h \searrow 0}\left[\frac{\mathrm{W}_2^2\left(S_h^{\mathscr{M}_t}(u_\tau^k), u_\tau^{k-1}\right)-\mathrm{W}_2^2(u_\tau^k, u_\tau^{k-1})}{2 h}+\frac{\beta}{2} \mathrm{W}_2^2(u_\tau^k, u_\tau^{k-1})\right] \leq \mathscr{M}_t(u_\tau^{k-1})-\mathscr{M}_t(u_\tau^k)
\end{equation}
implying the thesis.
\end{proof}

\subsection{Gradient flows as solutions of the continuity equation}
We aim now to exploit the general theory presented in the previous paragraph in order to characterize the gradient flows obtained above as weak solutions of a suitable continuity equation. To this end we introduce a class of linear functionals along which we will estimate the dissipation of $\mathcal E$.
\begin{prop}[{\cite[Theorems 11.2.1, 11.2.3]{ambrosio2008gradient}}]\label{prop:general_-beta_flows}
Let $(\mathcal M,\mathfrak g)$ be a smooth Riemannian manifold and let $\nabla_{\mathcal M}$ denote the Riemannian gradient on $\mathcal M$ induced by $\mathfrak g$. Fix $t\in(0,T)$ and let $\varphi_t\in C_c^\infty(\mathcal M)$. Assume that
$
\beta \;\geq\; \big\|\nabla^2_{\mathcal M}\varphi_t\big\|_{\infty},
$
where $\nabla^2_{\mathcal M}\varphi_t$ is the Riemannian Hessian and the norm is the operator norm on tangent spaces. 
Then, for any $t \in (0, T)$, the functional $\mathscr{V}_t: \mathcal{P}_2(\mathcal{M}) \to \R$ defined by 
\begin{equation}\label{eq:V}
\mathscr{V}_t(\nu) := \int_\mathcal{M} \varphi_t(x)\,d \nu(x)
\end{equation}
generates a $(-\beta)$-flow in $\mathcal{P}_2(\mathcal{M})$  and, for every $\mu \in \mathcal{P}_2(\mathcal{M})$, the associated semigroup is given by $S_s^{\mathscr{V}_t}(\mu) = (F_s)_\# \mu$, $s \geq 0$, where $F_s:\mathcal M\to\mathcal M$ solves
\begin{equation}\label{eq:sist_F}
\left\{
\begin{aligned}
& \partial_s F_s(x) = - \nabla_\mathcal{M} \varphi_t(F_s(x)),\\
& F_0(x) = x.
\end{aligned}
\right.
\end{equation}
\end{prop}
Relying on the previous proposition and on the minimizing–movement construction from the previous section, we prove that the continuity equation \eqref{eq:cont_eq_grad_form} admits at least one weak solution for arbitrary initial data. In particular, we first establish the result for the classical Wasserstein distance in the following proposition and then extend it in Theorem \ref{thm:existence_of_sol_cont_eq_mobility} to the weighted case with a similar argument.
\begin{prop}\label{prop:existence_for_time_dep_cont_eq}
Let $\mathcal{E}$ be as in \eqref{eq:time_dep_interaction_energy}. Then, for every initial datum $u_0 \in \mathcal{P}_2(\mathbb{S}^{d - 1})$, there exists at least one weak solution $u \in AC^2([0, T]; (\mathcal{P}_2(\mathbb{S}^{d - 1}), {\mathrm{W}}_2))$ to 
\begin{equation}\label{eq:cont_eq_grad_form}
\left\{
\begin{aligned}
& \partial_t u_t(x) = \operatorname{div}_{\mathbb{S}^{d - 1}}^x \left(u_t(x) \int_{\R^{\alpha}} \mathcal{D}\mathscr{K}[u](x, \theta) \dd \Theta_t(\theta)\right),\\
& \lim_{t \searrow 0} u_t(x) = u_0(x),
\end{aligned}
\right.
\end{equation}
where $\mathcal{D}\mathscr{K}[u](x, \theta) := \int_{\mathbb{S}^{d - 1}} \nabla_{\mathbb{S}^{d - 1}}^x \mathscr{K}(x, y, \theta) \dd u_t(y)$.
\end{prop}
\begin{proof}
Given a test function $\varphi \in C_c^\infty((0,T) \times \mathbb{S}^{d - 1})$, for every $t \in (0, T)$, let $\mathscr{V}_t(\mu):=\int_{\mathbb{S}^{d-1}}\varphi_t(x)\,\mathrm{d}\mu(x)$ be as in Proposition \ref{prop:general_-beta_flows} and $\mu \in \mathcal{P}_2(\mathbb{S}^{d - 1})$. Exploiting the definition of push-forward, \ref{hyp:ker4}, \ref{hyp:ker3}, and dominated convergence, for every $s \geq 0$, we have
\begin{equation}\label{eq:derivative_energy}
\begin{aligned}
\frac{\mathcal{E}(S_{s + h}^{\mathscr{V}_{t}}(\mu), t) - \mathcal{E}(S_s^{\mathscr{V}_{t}}(\mu), t)}{h} & = \frac{1}{2} \iiint_{\R^{\alpha} \times \mathbb{S}^{d - 1} \times \mathbb{S}^{d - 1}} \Big{(} \mathscr{K}(F_{s + h}(x), F_{s + h}(y), \theta) \\
& \qquad \quad \qquad \qquad \qquad - \mathscr{K}(F_s(x), F_s(y), \theta)\Big{)} \dd \mu(x) \dd \mu(y) \dd \Theta_{t}(\theta)\\
& \underset{h \to 0}{\to} \iint_{\R^{\alpha} \times\mathbb{S}^{d - 1}} \Big{\langle}\mathcal{D}\mathscr{K}[\mu](F_s(x),\theta), \partial_s F_s(x)\Big{\rangle} \dd \mu(x) \dd \Theta_{t}(\theta)\\
& = - \iint_{\mathbb{S}^{d - 1} \times \R^{\alpha}} \Big{\langle}\mathcal{D} \mathscr{K}[\mu](F_s(x), \theta), \nabla_{\mathbb{S}^{d - 1}} \varphi_{t}(F_s(x))\Big{\rangle} \dd \Theta_{t}(\theta) \dd\mu(x).
\end{aligned}
\end{equation}

In order to complete this proof, we aim to show that the limit function $u$ obtained in Theorem \ref{thm:existence_of_GMMs} is also a weak solution of \eqref{eq:cont_eq_grad_form}. To this end, let us consider $u_\tau(t)$ as in \eqref{eq:piece_wise_linear} and $t \in (t_{k - 1}, t_k]$. Then, recalling Definition \ref{def:dissipation_def} and setting $s = 0$ in \eqref{eq:derivative_energy}, it holds
\[
\begin{aligned}
\mathfrak{D}_{\mathscr{V}_t} \mathcal{E}(u_\tau(t), t) & = \limsup_{h \searrow 0} \frac{\mathcal{E}(u_\tau^k, t) - \mathcal{E}(S_h^{\mathscr{V}_{t}} (u_\tau^k), t)}{h} = \iint_{\mathbb{S}^{d - 1} \times \R^{\alpha}} \Big{\langle}\mathcal{D} \mathscr{K}[u_{\tau}^k] , \nabla_{\mathbb{S}^{d - 1}} \varphi_{t}\Big{\rangle} \dd \Theta_{t} \dd u_{\tau}^k.
\end{aligned}
\]
Applying the flow–interchange estimate (Lemma \ref{lem:flow_interchange_estimate}) to the pair
\(\big(\mathcal{E}(\cdot,t),\mathscr{V}_t\big)\) and using \eqref{eq:derivative_energy} at $s=0$, we obtain,
for every $t\in (t_{k-1},t_k]$,
\[
-\frac{\beta}{2} {\mathrm{W}}_2^2(u_\tau^k, u_\tau^{k - 1}) + |\tau|\iint_{\mathbb{S}^{d - 1} \times \R^{\alpha}} \Big{\langle}\mathcal{D} \mathscr{K}[u_{\tau}^k], \nabla_{\mathbb{S}^{d - 1}} \varphi_{t}\Big{\rangle} \dd \Theta_{t} \dd u_{\tau}^k  \leq \mathscr{V}_{t}(u_\tau^{k - 1}) - \mathscr{V}_{t}(u_\tau^k).
\]

Repeating the same passages for $-\mathscr{V}_t$, we get
\[
-\frac{\beta}{2} {\mathrm{W}}_2^2(u_\tau^k, u_\tau^{k - 1}) - |\tau|\iint_{\mathbb{S}^{d - 1} \times \R^{\alpha}} \Big{\langle}\mathcal{D} \mathscr{K}[u_{\tau}^k], \nabla_{\mathbb{S}^{d - 1}} \varphi_{t}\Big{\rangle} \dd \Theta_{t} \dd u_{\tau}^k \leq \mathscr{V}_{t}(u_\tau^k) - \mathscr{V}_{t}(u_\tau^{k - 1}),
\]
so that, the following estimate holds
\[
- \frac{\beta}{2} \frac{{\mathrm{W}}_2^2(u_\tau^k, u_\tau^{k - 1})}{|\tau|} \leq \frac{1}{|\tau|} \int_{\mathbb{S}^{d - 1}} \varphi \dd (u_\tau^k - u_\tau^{k - 1}) + \iint_{\mathbb{S}^{d - 1} \times \R^{\alpha}} \Big{\langle}\mathcal{D} \mathscr{K}[u_{\tau}^k], \nabla_{\mathbb{S}^{d - 1}} \varphi_{t}\Big{\rangle} \dd \Theta_{t} \dd u_{\tau}^k \leq \frac{\beta}{2} \frac{{\mathrm{W}}_2^2(u_\tau^k, u_\tau^{k - 1})}{|\tau|}.
\]

Then, integrating in time and by a change of variables, since $\varphi\in C_c^\infty((0,T)\times\mathbb S^{d-1})$, we obtain that 
\begin{align*}
\left| \int_{0}^{T} \int_{\mathbb{S}^{d - 1}} \left(\frac{\varphi_t - \varphi_{t + |\tau|}}{|\tau|} + \int_{\R^{\alpha}} \Big{\langle}\mathcal{D} \mathscr{K}[u_{\tau}(t)], \nabla_{\mathbb{S}^{d - 1}} \varphi_{t}\Big{\rangle} \dd \Theta_{t} \right) \dd u_{\tau}(t) \dd t \right| \leq \frac{\beta |\tau|}{2} \int_{0}^{T} \zeta_\tau^2 \dd t,
\end{align*}
where $\zeta_\tau(t):= \mathrm{W}_2\!\big(u_\tau(t),u_\tau(t-|\tau|)\big)/|\tau|$ as in \eqref{eq:zeta}. 
Finally, letting $|\tau| \to 0$, by Theorem \ref{thm:existence_of_GMMs} we get that along a suitably chosen subsequence it holds 
\[
\begin{aligned}
& \int_{\mathbb{S}^{d - 1}} \frac{\varphi_t - \varphi_{t + |\tau|}}{|\tau|} \dd u_\tau(t) \to - \int_{\mathbb{S}^{d - 1}} \partial_t \varphi \dd u(t), \qquad \frac{\beta |\tau|}{2} \int_{0}^{T} \zeta_\tau^2 \dd t \to 0,
\end{aligned}
\]
and
\[
\iint_{\mathbb{S}^{d - 1} \times \R^{\alpha}} \Big{\langle}\mathcal{D}\mathscr{K}[u_{\tau}], \nabla_{\mathbb{S}^{d - 1}} \varphi_{t}\Big{\rangle} \dd \Theta_{t} \dd u_{\tau}(t) \to \iint_{\mathbb{S}^{d - 1} \times \R^{\alpha}}  \Big{\langle}\mathcal{D}\mathscr{K}[u], \nabla_{\mathbb{S}^{d - 1}} \varphi_t\Big{\rangle}\dd \Theta_t \dd u(t),
\]
for almost every $t$, which imply the thesis.
\end{proof}


\begin{thm}\label{thm:existence_of_sol_cont_eq_mobility}
Let  $m_\mu$ be as in \eqref{eq:m}. Then, for every initial datum $u_0 \in \mathcal{P}_2(\mathbb{S}^{d - 1})$, any limit curve $u$ obtained from the
minimizing movement scheme \eqref{eq:min_mov_weighted} is a weak solution of
\begin{equation}\label{eq:cont_eq_with_mobility}
\left\{
\begin{aligned}
& \partial_t u_t(x) = \operatorname{div}_x \left(u_t(x) \int_{\R^\alpha} \frac{\mathcal{D}\mathscr{K}[u_t](x, \theta)}{m_{u_t}(x, \theta)} \dd \Theta_t(\theta) \right),\\
& \lim_{t \searrow 0} u_t(x) = u_0(x).
\end{aligned}
\right.
\end{equation}
\end{thm}

\begin{proof}
By Proposition~\ref{prop:equiv_of_norms},
$\mathrm{W}_{m,2}$ is equivalent to $\mathrm{W}_2$.
Hence, up to different constants, all the estimates from
Lemma~\ref{lem:existence_of_minimizer} and
Theorem~\ref{thm:existence_of_GMMs} hold on
$(\mathcal{P}_2(\mathbb{S}^{d-1}),\mathrm{W}_{m,2})$. In particular, there exists a limit curve
$u\in AC^2([0,T];
(\mathcal{P}_2(\mathbb{S}^{d-1}),\mathrm{W}_{m,2}))$
such that $u_\tau(t)\rightharpoonup u(t)$ narrowly for
every $t\in[0,T]$. We prove that $u$ satisfies the
distributional formulation of
\eqref{eq:cont_eq_with_mobility}.

We observe that the weighted distance $\mathrm{W}_{m,2}$
can be obtained by considering
$\mathbb{S}^{d-1}\times\mathbb{R}^\alpha$ endowed with the
degenerate metric
$\mathfrak{L}^m
:=\mathfrak{g}^m\circ(\iota \circ \pi_1, \iota \circ\pi_1)$ as follows.

First, we introduce the non-degenerate Riemannian metric
$\mathfrak{g}^m$ on $\mathbb{S}^{d-1}\times\mathbb{R}^\alpha$ induced by
the inner product
\begin{equation}\label{eq:w_inner_prod}
\mathfrak{g}^m(v(x,\theta),w(x,\theta))
=
 m(x,\theta)\,
\langle v(x,\theta),w(x,\theta)\rangle,
\qquad
\forall\,v,w\in
\mathrm{T}_{(x,\theta)}(\mathbb{S}^{d-1}\times\mathbb{R}^\alpha),
\end{equation}
where $\langle\cdot,\cdot\rangle$ denotes the standard product
Riemannian metric on
$\mathbb{S}^{d-1}\times\mathbb{R}^\alpha$.
Denoting by $\nabla^m$ the Riemannian gradient associated with
$\mathfrak{g}^m$, by definition, for every
$\phi\in C^1(\mathbb{S}^{d-1}\times\mathbb{R}^\alpha)$ and every
$w\in\mathrm{T}_{(x,\theta)}(\mathbb{S}^{d-1}\times\mathbb{R}^\alpha)$,
it holds
\begin{equation}\label{eq:rimeann_grad}
\mathfrak{g}^m(\nabla^m\phi(x,\theta),w(x,\theta))
=
\langle
\nabla_{\mathbb{S}^{d-1}\times\mathbb{R}^\alpha}\phi(x,\theta),
w(x,\theta)
\rangle.
\end{equation}
Thus, by \eqref{eq:w_inner_prod},
\begin{equation}\label{eq:w_grad}
\nabla^m\phi(x,\theta)
=
\frac{1}{m(x,\theta)}\,
\nabla_{\mathbb{S}^{d-1}\times\mathbb{R}^\alpha}\phi(x,\theta).
\end{equation}

We now pass to the degenerate metric
$\mathfrak{L}^m$,
which maintains only the $\mathbb{S}^{d-1}$-component:
\begin{equation}\label{eq:Lm_def}
\mathfrak{L}^m_{(x,\theta)}(v,w)
=
m(x,\theta)\,
\langle \pi_1(v),\pi_1(w)\rangle,
\qquad
\forall\,v,w\in
\mathrm{T}_{(x,\theta)}(\mathbb{S}^{d-1}\times\mathbb{R}^\alpha).
\end{equation}
We define $\nabla^{\mathfrak{L}^m}\phi(x,\theta)$ as the unique
vector in $\mathrm{T}_{(x,\theta)}
(\mathbb{S}^{d-1}\times\mathbb{R}^\alpha)$ satisfying
\begin{equation}\label{eq:F2_horizontal}
\pi_2\bigl(\nabla^{\mathfrak{L}^m}\phi(x,\theta)\bigr)=0
\end{equation}
and
\[
\mathfrak{L}^m\bigl(\nabla^{\mathfrak{L}^m}\phi(x,\theta),w\bigr)
=
\bigl\langle
\nabla_{\mathbb{S}^{d-1}\times\mathbb{R}^\alpha}\phi(x,\theta),
w
\bigr\rangle
\]
for every
$w\in\mathrm{T}_{(x,\theta)}(\mathbb{S}^{d-1}\times\mathbb{R}^\alpha)$
such that $\pi_2(w)=0$.  The condition~\eqref{eq:F2_horizontal} is the same horizontality condition appearing in Lemma~\ref{lem:metric_der}, that is $\nabla^{\mathfrak{L}^m}\phi$ has no component along the parameter direction $\R^\alpha$.
We claim that
\begin{equation}\label{eq:w_grad_degenerate}
\nabla^{\mathfrak{L}^m}\phi(x,\theta)
=
(\iota\circ\pi_1)(\nabla^m\phi(x,\theta))
=
\iota\biggl(
\frac{\nabla_{\mathbb{S}^{d-1}}\phi(x,\theta)}
{m(x,\theta)}
\biggr).
\end{equation}
Indeed, let $w$ be such that $\pi_2(w)=0$. By
\eqref{eq:Lm_def} and \eqref{eq:w_grad}, we have
\begin{align*}
\mathfrak{L}^m\bigl((\iota\circ\pi_1)(\nabla^m\phi),w\bigr)
&=
m(x,\theta)\,
\bigl\langle
\pi_1(\nabla^m\phi),\pi_1(w)
\bigr\rangle \\
&=
\bigl\langle
\nabla_{\mathbb{S}^{d-1}}\phi,\pi_1(w)
\bigr\rangle \\
&=
\bigl\langle
\nabla_{\mathbb{S}^{d-1}\times\mathbb{R}^\alpha}\phi,w
\bigr\rangle,
\end{align*}
where in the last step we used $\pi_2(w)=0$.
Moreover,
\[
\pi_2\bigl((\iota\circ\pi_1)(\nabla^m\phi)\bigr)=0.
\]
Therefore $(\iota\circ\pi_1)(\nabla^m\phi)$ satisfies the defining
properties of $\nabla^{\mathfrak{L}^m}\phi$, proving
\eqref{eq:w_grad_degenerate}.
Now, define
$\mathscr{V}_t:
\mathcal{P}_2(\mathbb{S}^{d-1}\times\mathbb{R}^\alpha)\to\mathbb{R}$
by
\[
\mathscr{V}_t(\mu\times\Xi)
:=
\int_{\mathbb{S}^{d-1}\times\mathbb{R}^\alpha}
\phi_t(x,\theta)\,d(\mu\times\Xi)(x,\theta).
\]
By Proposition~\ref{prop:general_-beta_flows},
$\mathscr{V}_t$ generates a flow on
$(\mathbb{S}^{d-1}\times\mathbb{R}^\alpha,\mathfrak{L}^m)$ whose
associated semigroup is given by
\[
\mathscr{S}_s^{\mathscr{V}_t}(\mu\times\Xi)
:=
(\mathscr{F}_s)_\#(\mu\times\Xi),
\qquad
\forall\,s\geq0,
\]
where
$\mathscr{F}_s:\mathbb{S}^{d-1}\times\mathbb{R}^\alpha\to
\mathbb{S}^{d-1}\times\mathbb{R}^\alpha$
solves
\begin{equation}\label{eq:sist_F_w}
\left\{
\begin{aligned}
&\partial_s\mathscr{F}_s(x,\theta)
=
-\nabla^{\mathfrak{L}^m}\phi_t(\mathscr{F}_s(x,\theta))
=
-\iota\biggl(
\frac{
\nabla_{\mathbb{S}^{d-1}}\varphi_t
(\mathscr{F}_s^{(1)}(x,\theta))
}{
 m_\mu(\mathscr{F}_s^{(1)}(x,\theta),\theta)
}
\biggr),\\
&\mathscr{F}_0(x,\theta)=(x,\theta).
\end{aligned}
\right.
\end{equation}
Letting
$\mathscr{F}_s^{(1)}:=\pi_1\circ\mathscr{F}_s$
and
$\mathscr{F}_s^{(2)}:=\pi_2\circ\mathscr{F}_s$,
\eqref{eq:F2_horizontal} implies that the $\theta$-component is
constant along the flow, hence
\begin{equation}\label{eq:F2_is_Id}
\mathscr{F}_s^{(2)}=\mathrm{Id}, \qquad \forall\,s\geq0.
\end{equation}

Let us consider now the time-independent interaction energy $\mathscr{E}$ of the form
\begin{equation}\label{eq:augmented_interaction_energy}
\mathscr{E}(\mu \times \Xi) = \iint_{(\mathbb{S}^{d - 1} \times \R^\alpha) \times (\mathbb{S}^{d - 1} \times \R^\alpha)} \mathfrak{K}(x, y, \theta, \theta') \dd (\mu(x) \times \Xi(\theta)) \dd (\mu(y) \times \Xi(\theta')).
\end{equation}
Moreover, for every $\theta \in \R^\alpha$, let $\ell_\theta$ be the reproducing kernel of $H^\alpha(\R^\alpha; \Xi)$. Then, calling $\langle \cdot, \cdot \rangle_{H^\alpha}$ the standard dot product in $H^\alpha(\R^\alpha; \Xi)$ and choosing $\mathfrak{K}(x, y, \theta, \theta')$ satisfying
\[
\int_{\R^\alpha} \mathfrak{K}(x, y, \theta, \theta') \dd \Xi(\theta'):= \left\langle \mathscr{K}(x, y, \frac{\theta + \theta'}{2}), \ell_\theta(\theta')\right\rangle_{H^\alpha} = \mathscr{K}(x, y, \theta), \qquad \forall (x,y,\theta) \in \S^{d-1} \times \S^{d-1} \times \R^\alpha,
\]
we have that $\mathscr{E}$ coincides with the time-dependent interaction energy $\mathcal{E}$ introduced in \eqref{eq:time_dep_interaction_energy}. In addition, for every fixed $t \in [0, T]$, let us consider $\varphi_t \in C_c^\infty(\mathbb{S}^{d - 1})$ and set $\phi_t(x, \theta) = \varphi_t(\pi_1(x, \theta)) = \varphi_t(x)$ and $\Xi = \Theta_t$.

Then, arguing as in \eqref{eq:derivative_energy} and recalling \eqref{eq:F2_is_Id}, we obtain
\begin{align}\label{eq:derivativeenergy_with_m}
&\frac{\mathscr{E}(\mathscr{S}_{s + h}^{\mathscr{V}_{t}}(\mu\times \Theta)) - \mathscr{E}(\mathscr{S}_s^{\mathscr{V}_{t}}(\mu\times \Theta))}{h} = \frac{1}{2}\iiint_{\mathbb{S}^{d - 1} \times \mathbb{S}^{d - 1} \times \R^\alpha} \left\langle\mathscr{K}\left(\mathscr{F}_{s + h}^{(1)}(x, \theta), \mathscr{F}_{s + h}^{(1)}(y, \theta'), \frac{\theta + \theta'}{2}\right) \,\right. + \\
&\qquad\qquad\qquad\qquad\qquad\qquad \left.- \, \mathscr{K}\left(\mathscr{F}_s^{(1)}(x, \theta), \mathscr{F}_s^{(1)}(y, \theta'), \frac{\theta + \theta'}{2}\right), \ell_\theta(\theta')\right\rangle_{H^\alpha} \dd \mu(x) \dd \mu(y) \dd\Theta_t(\theta)\\
& \quad\qquad\qquad\qquad \underset{h \to 0}{\to} \iint_{\mathbb{S}^{d - 1} \times \R^{\alpha}} \Big\langle \iota \Big{(}\mathcal{D}\mathscr{K}[\mu](\mathscr{F}_s^{(1)}(x, \theta), \theta)\Big{)}, \partial_s \mathscr{F}_s(x, \theta)\Big\rangle_{d + \alpha} \dd \Theta_{t}(\theta) \dd \mu(x)\\
& \quad\qquad\qquad\qquad\,\, = - \iint_{\mathbb{S}^{d - 1} \times \R^{\alpha}} \left\langle \iota \Big{(}\mathcal{D} \mathscr{K}[\mu](\mathscr{F}_s^{(1)}(x, \theta), \theta)\Big{)}, \nabla_{\mathbb{S}^{d - 1} \times \R^{\alpha}}^{\mathfrak{L}^m} \varphi_{t}(\mathscr{F}_s^{(1)}(x, \theta))\right\rangle_{d + \alpha} \dd \Theta_{t}(\theta) \dd \mu(x)\\
& \quad\qquad\qquad\qquad\,\, = - \iint_{\mathbb{S}^{d - 1} \times \R^{\alpha}} \left\langle \frac{\mathcal{D} \mathscr{K}[\mu](\mathscr{F}_s^{(1)}(x, \theta), \theta)}{m_\mu(\mathscr{F}_s^{(1)}(x, \theta), \theta)}, \nabla_{\mathbb{S}^{d - 1}} \varphi_{t}(\mathscr{F}_s^{(1)}(x, \theta))\right\rangle \dd \Theta_{t}(\theta) \dd \mu(x),
\end{align}
where $\langle \cdot, \cdot \rangle_{d + \alpha}$ is the standard dot product in $\R^{d + \alpha}$ and in the last equality we exploited \eqref{eq:w_grad_degenerate}.
Thus, since
$\nabla_{\mathbb{S}^{d-1}}\varphi_t$ depends only on $x$, we can exchange the order of integration and
write the dissipation of $\mathcal{E}$ along $\mathscr{V}_t$ as
\begin{equation}\label{eq:dissipation_on_curves}
\mathfrak{D}_{\mathscr{V}_t}\mathcal{E}(\mu,t)
:=
-\frac{\dd}{\dd s}\bigg|_{s=0}
\mathscr{E}(\mathscr{S}_s^{\mathscr{V}_t}
(\mu\times\Theta_t))
=
\int_{\mathbb{S}^{d-1}}
\biggl\langle
\mathcal{V}[\mu,\Theta_t](x),\,
\nabla_{\mathbb{S}^{d-1}}\varphi_t(x)
\biggr\rangle \dd\mu(x),
\end{equation}
where
\begin{equation}\label{eq:def_V}
\mathcal{V}[\mu,\Theta_t](x)
:=
 \int_{\mathbb{R}^\alpha}
\frac{\mathcal{D}\mathscr{K}[\mu](x,\theta)}
{m_\mu(x,\theta)}\,d\Theta_t(\theta).
\end{equation}
Finally, applying the same flow–interchange argument as in Proposition \ref{prop:existence_for_time_dep_cont_eq}, we obtain the desired result.
\end{proof}

To conclude, from the above result, we can recover the transformer model considered in \eqref{eq:full_transformer_ODE} as follows.
\begin{cor}\label{cor:our_model}
Let us consider the same setting as in Theorem \ref{thm:existence_of_sol_cont_eq_mobility}. Then, any gradient flow $u$ of $\mathcal{E}$ is a solution of \eqref{eq:continuity_eq_full} with a general velocity field (attention mechanism) given by
\begin{equation}\label{eq:general_attention}
\mathcal{V}[\mu](x, \theta) = \int_{{\mathbb{S}^{d-1}}}\frac{\nabla_{\mathbb{S}^{d - 1}}^x \mathscr{K}(x, y, \theta)}{m_u(x, \theta)} \dd u_t(y).
\end{equation}
Moreover, choosing
\begin{equation}\label{eq:choice_of_K}
\mathscr{K}(x, y, \theta) = \mathit{\Sigma}(\mathfrak{B}(x, y, \Theta)),
\end{equation}
with $\mathfrak{B}(x, y, \Theta): \mathbb{S}^{d - 1} \times \mathbb{S}^{d - 1} \times \R^{\alpha} \to \mathbb{R}$ of class \(C^1\) in \(x\) and \(\Sigma\in C^1(\mathbb{R})\), allows us to recover \eqref{eq:vect_field} with 
\begin{equation}\label{eq:A_B_recovered}
A(x, y, \theta) = \frac{\mathit{\Sigma}'(\mathfrak{B}(x, y, \theta))}{m_u(x, \theta)}, \qquad B(x, y, \theta) =  \nabla_{\mathbb{S}^{d - 1}}^x \mathfrak{B}(x, y, \theta),
\end{equation}
where $\mathit{\Sigma}'$ denotes the derivative of $\mathit{\Sigma}$.
\end{cor}
\begin{remark}\label{rmk:softmax_self_attention}
It is worth noting that choosing
\begin{equation}\label{eq:recover_of_softmax_self_attention}
\mathit{\Sigma}(s) = e^s, \qquad \mathfrak{B}(x, y, \theta) = \langle x, D(\theta)y\rangle, \qquad m_u(x, \theta) = \int_{\mathbb{S}^{d - 1}} e^{\langle x, D(\theta)y\rangle} \dd u_t(y),
\end{equation}
where $D(\theta) \in \mathbb{R}^{d \times d}$, thanks to Corollary \ref{cor:our_model}, we recover
the continuous version of the softmax self-attention \eqref{eq:softmax_self_attention} with $V(\theta) = D(\theta)$ (in accordance with
the time-independent analysis carried out in \cite{burger2025analysis}).
\end{remark}

\subsection{Solutions to the continuity equation as curves of maximal slopes}
In the previous section we proved that the generalized minimizing movements obtained minimizing $\mathcal{E}$ are solutions to \eqref{eq:cont_eq_with_mobility}. We now establish the converse, showing that the solutions to \eqref{eq:cont_eq_with_mobility} are curves of maximal slope for $\mathcal{E}$.

Theorem~\ref{thm:existence_of_sol_cont_eq_mobility} produced a gradient
$\nabla^{\mathfrak{L}^m}\phi$ on the product manifold
$\S^{d-1}\times\R^\alpha$, which is a vector field on the full product space and
depends pointwise on both $x$ and $\theta$. To analyse the dynamics on
$\S^{d-1}$ alone, however, we need a notion of gradient that has been averaged
out in the $\theta$-direction against $\Theta_t$. This is the content of
Lemma~\ref{lem:chain_rule} below: we lift the pointwise metric $\mathfrak{L}^m$
to an $L^2$-inner product on $\theta$-parametrized tangent vectors at fixed $x$,
and then descend to $T_x\S^{d-1}$ via Riesz duality.

\begin{lemma}[Degenerate gradient on $\S^{d-1}$]\label{lem:chain_rule}
Let $\mathfrak{L}^m$ be the degenerate metric on $\S^{d-1}\times\R^\alpha$
defined in~\eqref{eq:Lm_def}, let $\Xi\in\mathcal{P}_2(\R^\alpha)$, and let
$f\colon\S^{d-1}\times\R^\alpha\to\R$ be differentiable. Then, $\mathfrak{L}^m$
together with $\Xi$ induces a vector field
$\nabla^m f\colon\S^{d-1}\to\mathrm{T}\S^{d-1}$ given by
\begin{equation}\label{eq:duality_product_induced}
(\nabla^m f)(x)
:=  b^{\Xi}(x)
   \int_{\R^\alpha}\frac{\nabla_x f(x,\theta)}{m(x,\theta)}\dd\Xi(\theta),
\end{equation}
where $b^{\Xi}(x):= \left(\int_{\R^\alpha}\frac{1}{m(x,\theta)}\dd\Xi(\theta)\right)^{-1}$ is the nonlinear effective mobility in~\eqref{eq:effective_mobility} when $m = m_\mu$.
\end{lemma}
\begin{proof}Consider the weighted Bochner space
\begin{equation}\label{eq:Hx_def}
\mathbf{H}_x := L^2\big(\R^\alpha,\Xi;\,\mathrm{T}_x\S^{d-1}\big),
\qquad
\langle U, W\rangle_{\mathbf{H}_x}
:= \int_{\R^\alpha} m(x,\theta)\,\langle U(\theta), W(\theta)\rangle \dd\Xi(\theta),
\end{equation}
whose elements are square-integrable maps $U\colon\R^\alpha\to\mathrm{T}_x\S^{d-1}$;
we refer to \cite[Ch.~II]{diestel1977vector} for an introduction to Bochner spaces.  By assumption~\ref{hyp:ker2}, the weight $m(x, \theta)$ is bounded above and away from zero, hence $\langle\cdot, \cdot\rangle_{\mathbf{H}_x}$ is an inner product and $\mathbf{H}_x$ a Hilbert space. The inner product on $\mathbf{H}_x$  is the $\Xi$-average of the
pointwise metric $\mathfrak{L}^m$ from Theorem~\ref{thm:existence_of_sol_cont_eq_mobility}
fiberwise in $x$: for $U,W\in\mathbf{H}_x$, the inclusion
$\iota\colon\mathrm{T}_x\S^{d-1}\hookrightarrow\mathrm{T}_{(x,\theta)}(\S^{d-1}\times\R^\alpha)$,
$\iota(v) = (v,0)$, satisfies $\pi_1\circ\iota=\mathrm{Id}$, so the
definition~\eqref{eq:Lm_def} of $\mathfrak{L}^m$ yields
\begin{equation}\label{eq:bochner_via_Lm}
\mathfrak{L}^m_{(x,\theta)}\bigl(\iota(U(\theta)),\iota(W(\theta))\bigr)
= m(x,\theta)\,\langle U(\theta),W(\theta)\rangle,
\end{equation}
hence
\begin{equation}\label{eq:bochner_is_integrated_Lm}
\langle U,W\rangle_{\mathbf{H}_x}
= \int_{\R^\alpha}
\mathfrak{L}^m_{(x,\theta)}\bigl(\iota(U(\theta)),\iota(W(\theta))\bigr)\dd\Xi(\theta).
\end{equation}
Now, consider the bounded linear map
\[
\mathfrak{P}\colon \mathbf{H}_x \to \mathrm{T}_x\S^{d-1},
\qquad
\mathfrak{P}(U) := \int_{\R^\alpha} U(\theta) \dd\Xi(\theta),
\]
which assigns to each $\theta$-parametrized vector field its $\Xi$-average.
$\mathfrak{P}$ is surjective, so the quotient
$\mathbf{H}_x/\ker\mathfrak{P}$ is isomorphic to $\mathrm{T}_x\S^{d-1}$, with the
quotient norm realized by orthogonal projection onto $(\ker\mathfrak{P})^\perp$.
Concretely, this gives the induced norm on $\mathrm{T}_x\S^{d-1}$
\begin{equation}\label{eq:induced_norm}
\|v_x\|_m^2 = \inf_{\substack{U\in\mathbf{H}_x \\ \mathfrak{P}(U) = v_x}} \|U\|_{\mathbf{H}_x}^2,
\qquad \forall v_x \in \mathrm{T}_x\S^{d-1}.
\end{equation} 
The minimization problem~\eqref{eq:induced_norm} can be solved explicitly. By
the equality case of Cauchy--Schwarz (or, equivalently, by Lagrange multipliers),
the unique minimizer is
\begin{equation}\label{eq:optimal_V}
V_x^{v}(\theta) = b^{\Xi}(x)\frac{v_x}{m(x, \theta)},
\end{equation}
and substituting back gives
\begin{equation}\label{eq:induced_norm_explicit}
\|v_x\|_m^2 = b^{\Xi}(x)\|v_x\|^2.
\end{equation}
We use the gradient $\nabla^{\mathfrak{L}^m}f$ from
Theorem~\ref{thm:existence_of_sol_cont_eq_mobility}, recalling that the
computation leading to~\eqref{eq:w_grad_degenerate} applies verbatim to any
differentiable $f\colon\S^{d-1}\times\R^\alpha\to\R$ and gives
\begin{equation}\label{eq:pi1_grad_Lm}
\pi_1\bigl(\nabla^{\mathfrak{L}^m} f\bigr)(x,\theta)
= \frac{\nabla_{\S^{d-1}} f(x,\theta)}{m(x,\theta)}.
\end{equation}
The map $v_x \mapsto\left\langle\pi_1\left(\nabla^{\mathfrak{L}^m} f\right), V_x^v\right\rangle_{\mathbf{H}_x}$ is linear on $\mathrm{T}_x \S^{d-1}$, since $v_x \mapsto V_x^v$ is linear by~\eqref{eq:optimal_V}. Therefore, by the
Riesz representation theorem there exists a unique
$(\nabla^m f)(x)\in \mathrm{T}_x\S^{d-1}$ such that \begin{equation}\label{eq:riesz_definition}
\big\langle (\nabla^m f)(x),\, v_x\big\rangle
= \big\langle \pi_1(\nabla^{\mathfrak{L}^m}f),\, V_x^v\big\rangle_{\mathbf{H}_x}=\int_{\R^\alpha} m(x,\theta)\,\big\langle \pi_1(\nabla^{\mathfrak{L}^m}f)(x,\theta),\, V_x^{v}(\theta)\big\rangle \dd\Xi(\theta)
\quad\forall\,v_x\in\mathrm{T}_x\S^{d-1}. 
\end{equation}
Plugging~\eqref{eq:pi1_grad_Lm} and~\eqref{eq:optimal_V} into the right-hand
side and simplifying, we obtain
\begin{align*}
\big\langle (\nabla^m f)(x), v_x\big\rangle
&= \int_{\R^\alpha} m(x,\theta)\,
   \Big\langle\frac{\nabla_x f(x,\theta)}{m(x,\theta)},\,
   b^{\Xi}(x) \frac{v_x}{m(x,\theta)}\Big\rangle\dd\Xi(\theta)\\
&= b^{\Xi}(x)
\Big\langle\int_{\R^\alpha}\frac{\nabla_x f(x,\theta)}{m(x,\theta)}\dd\Xi(\theta),\;v_x\Big\rangle,
\end{align*}
and~\eqref{eq:duality_product_induced} follows.
\end{proof}
\begin{remark}
The construction of Lemma~\ref{lem:chain_rule} also induces, fiberwise in $x$, an inner product on $\mathrm{T}_x\S^{d-1}$,
\begin{equation}\label{eq:induced_inner_product}
\langle v_x,w_x\rangle_{x,m}
:= b^{\Xi}(x)\langle v_x,w_x\rangle,
\qquad v_x,w_x\in\mathrm{T}_x\S^{d-1},
\end{equation}
namely the Euclidean inner product rescaled by the harmonic mean of $m(x,\cdot)$ with respect to $\Xi$. It naturally appears here because we are minimizing a quadratic cost subject to a linear averaging constraint, and it agrees with the harmonic mean rescaling of Lemma~\ref{lem:metric_der}; in the single-head case $\Xi = \delta_{\bar\theta}$, $b^\Xi(x) = m(x,\bar\theta)$ and the construction recovers the metric structure of \cite{burger2025analysis}.
\end{remark}
Throughout the rest of the paper, $\nabla^m f$ denotes the descended gradient on $\S^{d-1}$ defined in~\eqref{eq:duality_product_induced}; this should not be confused with the Riemannian gradient on $\S^{d-1}\times\R^\alpha$ from~\eqref{eq:w_grad}.

\begin{lemma}\label{lem:strong_upper_gradient}
Let us consider $T > 0$ and $\mathcal{E}$ as in \eqref{eq:time_dep_interaction_energy}. Then, $\mathscr{G}_t: \mathcal{P}_2(\S^{d - 1}) \to \R$ given by
\begin{equation}\label{eq:upper_grad}
\mu \mapsto \mathscr{G}_t(\mu) := \left(\int_{\mathbb{S}^{d - 1}}  b_\mu(x)\left\|\int_{\R^\alpha}\frac{\mathcal{D}\mathscr{K}[\mu](x,\theta)}{m_{\mu}(x,\theta)} \dd \Theta_t (\theta)\right\|^2 \dd \mu(x)\right)^\frac{1}{2}, \qquad \forall t \in [0, T]
\end{equation}
is a strong upper gradient for $\mathcal{E}$.
\end{lemma}
\begin{proof}
    Let $(\mu_\tau)_{\tau\in[0,T]}\in \mathrm{AC}([0,T];\mathcal{P}_2(\S^{d-1}))$. By Lemma~\ref{lem:metric_der} with $\Xi_\tau\equiv\Theta_\tau$, there exists a Borel field $(v_\tau)_{\tau\in(0,T)}$ such that $(\mu,v)\in\CE_\theta(0,T)$ and, for a.e.\ $\tau\in(0,T)$,
\begin{equation}\label{eq:metric_derivative_strong_upper_gradient}
|\dot{\mu}|(\tau)=\left(\int_{\S^{d-1}} b_{\mu_{\tau}}(x)\left\|\int_{\R^\alpha}\pi_1(v_\tau(x,\theta))\dd\Theta_\tau(\theta)\right\|^2\dd\mu_\tau(x)\right)^{\frac12}.
\end{equation}
Set $\bar{v}_\tau(x):=\int_{\R^\alpha}\pi_1(v_\tau(x,\theta))\dd\Theta_\tau(\theta)$, so $(\mu_\tau,\bar v_\tau)\in\CE(0,T)$ on $\S^{d-1}$.
For $0\leq s\leq r\leq T$, decompose the increment of $\mathcal{E}$ along the diagonal $(\mu_\tau,\tau)$,
\begin{equation}\label{eq:sug_chain_decomp}
\mathcal{E}(\mu_r,r)-\mathcal{E}(\mu_s,s)=\int_s^r\partial_t\mathcal{E}(\mu_\tau,\tau)\dd\tau+\int_s^r\frac{\dd}{\dd\tau}\mathcal{E}(\mu_\tau,t)\bigg|_{t=\tau}\dd\tau.
\end{equation}
By Lemma~\ref{lem:chain_rule}, applied with $\Xi=\Theta_\tau$ and $f=\mathscr{K}[\mu_\tau]$, the descended gradient $\nabla^m\mathscr{K}[\mu_\tau]$ satisfies
\begin{equation}\label{eq:sug_diag}
\frac{\dd}{\dd\tau}\mathcal{E}(\mu_\tau,t)\bigg|_{t=\tau}=\int_{\S^{d-1}}\langle(\nabla^m\mathscr{K}[\mu_\tau])(x),\bar{v}_\tau(x)\rangle\dd\mu_\tau(x).
\end{equation}
If we rewrite the integrand of \eqref{eq:sug_diag} as
\[
\left\langle b_{\mu_{\tau}}^{-\frac{1}{2}}(x)(\nabla^m\mathscr{K}[\mu_\tau])(x),\;b_{\mu_{\tau}}^{\frac{1}{2}}(x)\bar{v}_\tau(x)\right\rangle,
\]
then Cauchy--Schwarz pointwise in $x$, followed by Cauchy--Schwarz in $L^2(\mu_\tau)$, gives
\begin{align*}
\left|\frac{\dd}{\dd\tau}\mathcal{E}(\mu_\tau,t)\bigg|_{t=\tau}\right|
&\leq\left(\int_{\S^{d-1}} b_{\mu_{\tau}}^{-1}(x)\|(\nabla^m\mathscr{K}[\mu_\tau])(x)\|^2\dd\mu_\tau(x)\right)^{\frac12}\\
&\qquad\cdot\left(\int_{\S^{d-1}} b_{\mu_{\tau}}(x)\|\bar{v}_\tau(x)\|^2\dd\mu_\tau(x)\right)^{\frac12}\\
&=\mathscr{G}_\tau(\mu_\tau)\,|\dot{\mu}|(\tau),
\end{align*}
where the last equality uses, on one hand, the explicit formula \eqref{eq:duality_product_induced} of Lemma~\ref{lem:chain_rule} for $\nabla^m\mathscr{K}[\mu_\tau]$, which gives
\[
 b_{\mu_{\tau}}^{-1}(x)\|(\nabla^m\mathscr{K}[\mu_\tau])(x)\|^2= b_{\mu_{\tau}}(x)\left\|\int_{\R^\alpha}\frac{\mathcal{D}\mathscr{K}[\mu_\tau]}{m_{\mu_\tau}}\dd\Theta_\tau\right\|^2,\]
together with \eqref{eq:upper_grad}, and on the other hand \eqref{eq:metric_derivative_strong_upper_gradient}. Substituting into \eqref{eq:sug_chain_decomp}, we obtain
\[
\left|\mathcal{E}(\mu_r,r)-\mathcal{E}(\mu_s,s)-\int_s^r\partial_t\mathcal{E}(\mu_\tau,\tau)\dd\tau\right|\leq\int_s^r\mathscr{G}_\tau(\mu_\tau)\,|\dot{\mu}|(\tau)\dd\tau,\qquad\forall\,0\leq s\leq r\leq T.
\]
By Definition~\ref{def:Strong_upper_gradient}, we conclude.
\end{proof}

Finally, we are ready to prove the following characterization result.

\begin{thm}\label{thm:GMMs_are_curve_of_max_slope}
Let us consider $\mu\in \mathrm{AC}(0,T;\mathcal{P}_2(\S^{d-1}))$ and $v_t\in\mathrm{T}(\S^{d-1}\times\R^{\alpha})$ such that $(\mu_t,v_t)\in\CE_\theta(0,T)$. Then, it holds
\begin{equation}\label{eq:var_ineq}
\begin{aligned}
\mathcal{E}(\mu_T,T)&-\mathcal{E}(\mu_0,0)-\int_0^T\partial_t\mathcal{E}(\mu_t,t)\,\dd t\\
&\quad+\frac{1}{2}\int_0^T\int_{\S^{d-1}}b_{\mu_t}(x)\left\|\int_{\R^\alpha}\pi_1(v_t(x,\theta))\,\dd\Theta_t(\theta)\right\|^2\dd\mu_t(x)\,\dd t\\
&\quad+\frac{1}{2}\int_0^T\int_{\S^{d-1}}b_{\mu_t}(x)\left\|\int_{\R^\alpha}\frac{\mathcal{D}\mathscr{K}[\mu_t](x,\theta)}{m_{\mu_t}(x,\theta)}\,\dd\Theta_t(\theta)\right\|^2\dd\mu_t(x)\,\dd t\;\geq\;0.
\end{aligned}
\end{equation}
Moreover, equality holds if and only if $\mu$ solves~\eqref{eq:cont_eq_with_mobility}. Equivalently, $\mu$ is a curve of maximal slope for $\mathcal{E}$ if and only if it solves~\eqref{eq:cont_eq_with_mobility}.
\end{thm}

\begin{proof}
    Set $\bar v_t(x):=\int_{\R^\alpha}\pi_1(v_t(x,\theta))\,\dd\Theta_t(\theta)$, so $(\mu_t,\bar v_t)\in\CE(0,T)$ on $\S^{d-1}$. Decomposing the increment of $\mathcal{E}$ along the diagonal $(\mu_t,t)$ as in~\eqref{eq:sug_chain_decomp} and applying Lemma~\ref{lem:chain_rule} with $\Xi=\Theta_t$ and $f=\mathscr{K}[\mu_t]$ to the second term, we obtain
\begin{equation}\label{eq:total_diff}
\mathcal{E}(\mu_T,T)-\mathcal{E}(\mu_0,0)
=\int_0^T\partial_t\mathcal{E}(\mu_t,t)\,\dd t
+\int_0^T\int_{\S^{d-1}}\bigl\langle(\nabla^m\mathscr{K}[\mu_t])(x),\bar v_t(x)\bigr\rangle\dd\mu_t(x)\,\dd t.
\end{equation} 
By Cauchy--Schwarz applied with the splitting $\langle b_{\mu_t}^{-1/2}\nabla^m\mathscr{K}[\mu_t],\,b_{\mu_t}^{1/2}\bar v_t\rangle$, followed by Young's inequality,
\begin{equation}\label{eq:holder_young}
\begin{aligned}
\int_{\S^{d-1}}\!\!\bigl\langle\nabla^m\mathscr{K}[\mu_t],\bar v_t\bigr\rangle\dd\mu_t
&\geq -\Bigl(\!\int_{\S^{d-1}}\!\! b_{\mu_t}^{-1}\|\nabla^m\mathscr{K}[\mu_t]\|^2\dd\mu_t\Bigr)^{1/2}\Bigl(\!\int_{\S^{d-1}}\!\!b_{\mu_t}\|\bar v_t\|^2\dd\mu_t\Bigr)^{1/2}\\
&\geq -\frac{1}{2}\int_{\S^{d-1}}\!\! b_{\mu_t}^{-1}\|\nabla^m\mathscr{K}[\mu_t]\|^2\dd\mu_t-\frac{1}{2}\int_{\S^{d-1}}\!\! b_{\mu_t}\|\bar v_t\|^2\dd\mu_t.
\end{aligned}
\end{equation}
By the explicit formula~\eqref{eq:duality_product_induced} of Lemma~\ref{lem:chain_rule},
\[
b_{\mu_t}^{-1}(x)\|(\nabla^m\mathscr{K}[\mu_t])(x)\|^2
=b_{\mu_t}(x)\left\|\int_{\R^\alpha}\frac{\mathcal{D}\mathscr{K}[\mu_t](x,\theta)}{m_{\mu_t}(x,\theta)}\,\dd\Theta_t(\theta)\right\|^2.
\]
Substituting into~\eqref{eq:holder_young} and then into~\eqref{eq:total_diff}, and rearranging, gives~\eqref{eq:var_ineq}. Equality in~\eqref{eq:holder_young} holds if and only if $\bar v_t=-b_{\mu_t}^{-1}\nabla^m\mathscr{K}[\mu_t]$ for a.e.\ $t$. By the explicit formula~\eqref{eq:duality_product_induced} of Lemma~\ref{lem:chain_rule}, this is equivalent to $\bar v_t=-\int_{\R^\alpha}\mathcal D\mathscr K[\mu_t]/m_{\mu_t}\,\dd\Theta_t$, that is, to $(\mu_t)$ solving~\eqref{eq:cont_eq_with_mobility}. The final equivalence with the curve of maximal slope property follows from Definition~\ref{def:curves_max_slope}, Lemma~\ref{lem:metric_der}, Lemma~\ref{lem:strong_upper_gradient}, and Theorem~\ref{thm:existence_of_sol_cont_eq_mobility}.
\end{proof}

\subsection{Long-time behavior}\label{sec:long_time_behavior}
We aim to study the limiting behavior of the gradient flows constructed in this section as $t$ diverges.

Let us start by recalling the following definitions.
\begin{defi}[\cite{KLS}, Def. 2.32]\label{def:omega_limit}
Let $\mathcal{M}$ be any compact Riemannian manifold. Then, for any $\nu \in \mathrm{AC}([0, \infty); \mathcal{P}_2(\mathcal{M}))$, we define its $\omega$-limit $\omega_\nu \subseteq \mathcal{P}_2(\mathcal{M})$ as
\begin{equation}\label{eq:omega_limit}
\omega_\nu := \{\mu \in \mathcal{P}_2(\mathcal{M}): \text{$\exists t_n \nearrow \infty$ such that $\lim_{n \to \infty} \nu_{t_n} = \mu$ in $\mathcal{P}_2(\mathcal{M})$}\}.
\end{equation}
\end{defi}

Since the strong upper gradient $\mathscr{G}_t$ defined in Lemma~\ref{lem:strong_upper_gradient} depends on time through the distribution $\Theta_t$, the notion of stationarity for $\omega$-limit points is tied to the asymptotic behavior of $\Theta_t$. To make this dependence explicit, we introduce the following.

\begin{defi}[$\bar\Theta$-stationary point]\label{def:stationary_point_fixed}
Let $\mathcal{E}:\mathcal{P}_2(\mathcal{M})\times[0,\infty)\to\R$ and let $\mathscr{G}_t:\mathcal{P}_2(\mathcal{M})\to[0,\infty]$ be a strong upper gradient for $\mathcal{E}(\cdot,t)$ for each $t \geqslant 0$. We say that $\mu \in \mathcal{P}_2(\mathcal{M})$ is a \emph{$\bar\Theta$-stationary point} of $\mathcal{E}$, or equivalently that $\mu \in \mathbb{E}_\mathscr{G}^{\bar\Theta}$, if there exist $\nu \in \mathrm{AC}([0,\infty);\mathcal{P}_2(\mathcal{M}))$ and a subsequence $t_n \nearrow \infty$ such that
\begin{equation}\label{eq:stationary_vanishing_slope}
\nu_{t_n} \rightharpoonup \mu \quad \text{in } \mathcal{P}_2(\mathcal{M}),
\qquad
\Theta_{t_n} \rightharpoonup \bar\Theta \quad \text{in } \mathcal{P}_2(\R^\alpha),
\qquad \text{and} \qquad
\mathscr{G}_{t_n}(\nu_{t_n}) \to 0.
\end{equation}
\end{defi}

We also denote by
\[
\Omega(\Theta) := \bigl\{\bar\Theta \in \mathcal{P}_2(\R^\alpha) : \exists\, s_n \nearrow \infty,\; \Theta_{s_n} \rightharpoonup \bar\Theta\bigr\}
\]
the set of narrow cluster points of the weight trajectory $(\Theta_t)_{t \geqslant 0}$. By \ref{hyp:ker4}, $\Theta_t$ is supported in a common compact set for all $t$, so $\Omega(\Theta) \neq \emptyset$ by Prokhorov's theorem.

We have then all the elements to prove the following result.

\begin{thm}\label{thm:stationary_points}
Let $\mathcal{E}$ be as in~\eqref{eq:time_dep_interaction_energy}, and let $u \in AC^2([0,\infty);\mathcal{P}_2(\S^d))$ be a solution of~\eqref{eq:cont_eq_with_mobility}. Assume moreover that
\begin{equation}\label{eq:integrable_time_derivative_energy}
\sup_{t \geqslant 0}\int_0^t \partial_\tau\mathcal{E}(u_\tau,\tau)\dd\tau < +\infty.
\end{equation}
Then, $\omega_u \neq \emptyset$ and for every $\varphi \in \omega_u$ there exists $\bar\Theta \in \Omega(\Theta)$ such that $\varphi \in \mathbb{E}_\mathscr{G}^{\bar\Theta}$. In particular, it holds
\begin{equation}\label{eq:inclusion}
\omega_u \;\subseteq\; \bigcup_{\bar\Theta \in \Omega(\Theta)} \mathbb{E}_\mathscr{G}^{\bar\Theta}.
\end{equation}
\end{thm}
\begin{proof}
   Since $\mathbb{S}^{d-1}$ is compact, $\mathcal{P}_2(\mathbb{S}^{d-1})$ is sequentially compact for the narrow topology by Prokhorov's Theorem (see e.g.\ \cite[Box~1.4]{Santambrogio}), so $\omega_u \neq \emptyset$.
   
 Since $u$ solves~\eqref{eq:cont_eq_with_mobility}, equality holds in~\eqref{eq:var_ineq} of Theorem~\ref{thm:GMMs_are_curve_of_max_slope}, which yields the EDI
\begin{equation}\label{eq:EDI_stationary_new}
\mathcal{E}(u_r,r)+\frac{1}{2}\int_s^r|\dot u|^2(t)\,\dd t+\frac{1}{2}\int_s^r\mathscr{G}_t^2(u_t)\,\dd t = \mathcal{E}(u_s,s)+\int_s^r\partial_t\mathcal{E}(u_t,t)\,\dd t,\qquad\forall\,0\leq s\leq r<\infty.
\end{equation}
Defining
\[
F(t) := \mathcal{E}(u_t,t)-\int_0^t\partial_\tau\mathcal{E}(u_\tau,\tau)\,\dd\tau,
\]
\eqref{eq:EDI_stationary_new} rewrites as
\begin{equation}\label{eq:F_monotone_new}
F(s)-F(r) = \frac{1}{2}\int_s^r|\dot u|^2(t)\,\dd t+\frac{1}{2}\int_s^r\mathscr{G}_t^2(u_t)\,\dd t\geq 0,\qquad\forall\,0\leq s\leq r<\infty,
\end{equation}
so $F$ is non-increasing. Moreover, by \ref{hyp:ker4}--\ref{hyp:ker1} and compactness of $\mathbb{S}^{d-1}$,
there exists $C_0 \in \mathbb{R}$ such that $\mathcal{E}(\mu, t) \geq -C_0$ for every
$\mu \in \mathcal{P}_2(\mathbb{S}^{d-1})$ and $t \geq 0$. Combined with~\eqref{eq:integrable_time_derivative_energy}, this gives
\[
F(t)\geq -C_0-\sup_{t\geq 0}\int_0^t\partial_\tau\mathcal{E}(u_\tau,\tau)\,\dd\tau > -\infty,
\]
so $F$ is bounded from below.
Hence $F(t) \to \ell \in \mathbb{R}$, and
\eqref{eq:F_monotone_new} gives, for every fixed $s > 0$,
\begin{equation}\label{eq:decay_new}
\frac{1}{2}\int_t^{t+s} |\dot{u}|^2(\tau)\,\dd\tau
+ \frac{1}{2}\int_t^{t+s} \mathscr{G}_\tau^2(u_\tau)\,\dd\tau
= F(t) - F(t+s) \underset{t \to \infty}{\longrightarrow} 0.
\end{equation} 
By the explicit form of the velocity field in~\eqref{eq:cont_eq_with_mobility} and the bounds in~\ref{hyp:ker4} and~\ref{hyp:ker1}, the metric derivatives $|\dot u_t|$ and $|\dot\Theta_t|$ are uniformly bounded in $t$, and $t\mapsto\mathscr{G}_t^2(u_t)$ is Lipschitz on $[0,\infty)$, hence uniformly continuous. Together with~\eqref{eq:decay_new}, this yields
\begin{equation}\label{eq:G_to_zero}
\mathscr{G}_t(u_t)\underset{t\to\infty}{\longrightarrow}0.
\end{equation}
Now fix $\varphi\in\omega_u$ and let $t_n\nearrow\infty$ with $u_{t_n}\rightharpoonup\varphi$ narrowly. By~\eqref{eq:G_to_zero}, $\mathscr{G}_{t_n}(u_{t_n})\to 0$. Moreover, by~\ref{hyp:ker4}, $(\Theta_{t_n})$ is supported in a common compact set, so by Prokhorov's theorem we can extract a (non-relabeled) subsequence such that $\Theta_{t_n}\rightharpoonup\bar\Theta$ for some $\bar\Theta\in\Omega(\Theta)$. Together with Definition~\ref{def:stationary_point_fixed}, we conclude $\varphi\in\mathbb{E}_\mathscr{G}^{\bar\Theta}$.
\end{proof}

\begin{remark}\label{rem:role_of_theta}
We emphasize that the weight dynamics $\Theta_t$ enters Theorem~\ref{thm:stationary_points} through assumption~\eqref{eq:integrable_time_derivative_energy} and the set $\Omega(\Theta)$. By a direct computation from~\eqref{eq:full_transformer_ODE} and assumptions~\ref{hyp:ker4} and~\ref{hyp:ker1}, there exists a constant $C>0$ such that $|\partial_t\mathcal{E}(\mu,t)|\leq C\,|\dot\Theta_t|$ for every $\mu\in\mathcal{P}_2(\S^{d-1})$ and a.e.\ $t\geq 0$, where $|\dot\Theta_t|$ denotes the metric derivative of $\Theta_t$ in $\mathcal{P}_2(\R^\alpha)$. Hence, $\int_0^\infty|\dot\Theta_t|\,\dd t<\infty$ is a sufficient condition for~\eqref{eq:integrable_time_derivative_energy}. 
Moreover, we observe that Theorem~\ref{thm:stationary_points} associates to each $\omega$-limit point a specific $\bar\Theta \in \Omega(\Theta)$ via the subsequence along which $\Theta_{t_n}$ converges. Different $\omega$-limit points may, a priori, be stationary for different $\bar\Theta$. In particular, when $\Theta_t \rightharpoonup \bar\Theta$ as $t \to \infty$, the set $\Omega(\Theta)$ reduces to $\{\bar\Theta\}$ and the inclusion~\eqref{eq:inclusion} becomes $\omega_u \subseteq \mathbb{E}_\mathscr{G}^{\bar\Theta}$.
\end{remark}

\section{Robustness and stability}\label{sec:stability}
In this section we aim to investigate two important stability properties. First, we concentrate on proving the \emph{robustness} of the considered models with respect to input perturbations, from which we will deduce also their uniqueness. Informally, considering the common application of transformers to natural language models, this property corresponds to a good response of our architecture when fed with imperfect prompts (misspells, syntax or grammar errors). Second, we tackle the more subtle problem of
perturbing the weights initialization strategy. On the one hand, the continuous
dependence on the weights distribution provides interesting insights on the regularity
of the parameter landscape. On the other hand, recalling the mean-field nature of the
considered setup, this \emph{stability} result allows us to interpret the multi-headed
transformers used in practice as finite approximations of our continuous theory.

\subsection{Robustness under input perturbation}
We consider now a perturbation of the initial data $u_0^\eta\in B_\eta(u_0)$, $\eta\geq 0$, where $B_\eta(u_0)$ denotes the $\mathrm{W}_2$-ball of radius $\eta$ around $u_0$. The following holds.
\begin{thm}\label{thm:continuous_dependence_on_data}
Let $u_t^\eta$ and $u_t$ be gradient flows of $\mathcal{E}$ with initial data
$u_0^\eta$ and $u_0$, respectively. Then, there exist two positive constants $C_1$,
$C_2$ such that the following estimate holds:
\begin{equation}\label{eq:gronwall_estimate}
{\mathrm{W}}_{m,2}(u_t^\eta, u_t) \leq C_1\, {\mathrm{W}}_{m,2}(u_0^\eta, u_0)\, e^{C_2 t},
\qquad \forall\, t \in [0, T].
\end{equation}
\end{thm}
\begin{proof}
Since $u_t^\eta$ and $u_t$ are gradient flows of $\mathcal{E}$, by Theorem~\ref{thm:existence_of_sol_cont_eq_mobility} there exist Borel velocity fields $v_t^\eta,v_t\in\Gamma([0,T]\times\S^{d-1}\times\R^\alpha;\mathrm{T}(\S^{d-1}\times\R^\alpha))$ with $(u_t^\eta,v_t^\eta),(u_t,v_t)\in\CE_\theta(0,T)$ and
\[
\pi_1(v_t^\eta(x,\theta))=\frac{\mathcal{D}\mathscr{K}[u_t^\eta](x,\theta)}{m_{u_t^\eta}(x,\theta)},\qquad
\pi_1(v_t(x,\theta))=\frac{\mathcal{D}\mathscr{K}[u_t](x,\theta)}{m_{u_t}(x,\theta)}.
\]
Let $\Pi(a,b)$ denote the set of optimal transport plans between $a,b\in\mathcal{P}_2(\S^{d-1})$. For every $\gamma\in\Pi(u_t^\eta,u_t)$, adapting~\cite[Theorem~8.4.7]{ambrosio2008gradient} to $\S^{d-1}$ embedded in $\R^d$, we have
\begin{equation}\label{eq:derivative_wasserstein}
\frac{\dd}{\dd t}\mathrm{W}_2^2(u_t^\eta,u_t) = 2\iint_{\S^{d-1}\times\S^{d-1}}\biggl\langle x-y,\int_{\R^\alpha}\bigl(\pi_1(v_t^\eta(x,\theta))-\pi_1(v_t(y,\theta))\bigr)\,\dd\Theta_t(\theta)\biggr\rangle\dd\gamma(x,y),
\end{equation}
where $\langle\cdot,\cdot\rangle$ is the standard inner product in $\R^d$.

Now, we estimate $\|\pi_1(v_t^\eta(x,\theta))-\pi_1(v_t(y,\theta))\|$. By Theorem~\ref{thm:existence_of_sol_cont_eq_mobility} and~\eqref{eq:m}, we can write
\begin{equation}\label{eq:est_on_v_xy}
\begin{aligned}
&\|\pi_1(v_t^\eta(x,\theta))-\pi_1(v_t(y,\theta))\|\\
&\quad=\biggl\|\int_{\S^{d-1}}\frac{\nabla_{\S^{d-1}}\mathscr{K}(x,z,\theta)}{m_{u_t^\eta}(x,\theta)}\,\dd u_t^\eta(z)-\int_{\S^{d-1}}\frac{\nabla_{\S^{d-1}}\mathscr{K}(y,z,\theta)}{m_{u_t}(y,\theta)}\,\dd u_t(z)\biggr\|\\
&\quad\leq\underbrace{\frac{1}{m_{u_t^\eta}(x,\theta)}\int_{\S^{d-1}}\|\nabla_{\S^{d-1}}\mathscr{K}(x,z,\theta)-\nabla_{\S^{d-1}}\mathscr{K}(y,z,\theta)\|\,\dd u_t^\eta(z)}_{=:I_1}\\
&\qquad+\underbrace{\biggl|\frac{1}{m_{u_t^\eta}(x,\theta)}-\frac{1}{m_{u_t}(y,\theta)}\biggr|\int_{\S^{d-1}}\|\nabla_{\S^{d-1}}\mathscr{K}(y,z,\theta)\|\,\dd u_t(z)}_{=:I_2}.
\end{aligned}
\end{equation}
For $I_1$, by the $C^{1,1}$-regularity of $\mathscr{K}$ in \ref{hyp:ker1} and the
lower bound $m_\mu \geq c > 0$ from \ref{hyp:ker2}, we immediately get
\begin{equation}\label{eq:I1_bound}
I_1\leq C\|x-y\|.
\end{equation}
For $I_2$, by \ref{hyp:ker2}, we have
\begin{equation}\label{eq:inv_m_diff}
\biggl|\frac{1}{m_{u_t^\eta}(x,\theta)}-\frac{1}{m_{u_t}(y,\theta)}\biggr|\leq\frac{|m_{u_t^\eta}(x,\theta)-m_{u_t}(y,\theta)|}{c^2}.
\end{equation}
Recalling $m_\mu(x,\theta)=\int_{\S^{d-1}}M(x,z,\theta)\,\dd\mu(z)$ and adding and subtracting $\int_{\S^{d-1}}M(y,z,\theta)\,\dd u_t^\eta(z)$, we get
\begin{equation}\label{eq:m_diff_split}
|m_{u_t^\eta}(x,\theta)-m_{u_t}(y,\theta)|\leq\int_{\S^{d-1}}|M(x,z,\theta)-M(y,z,\theta)|\,\dd u_t^\eta(z)+\biggl|\int_{\S^{d-1}}M(y,z,\theta)\,\dd(u_t^\eta-u_t)(z)\biggr|.
\end{equation}
The first term is bounded by $C\|x-y\|$ by the Lipschitz regularity of $M(\cdot,z,\theta)$ inherited from~\ref{hyp:ker1}. For the second term, since $z\mapsto M(y,z,\theta)$ is Lipschitz with constant $C$ by~\ref{hyp:ker1}, the Kantorovich--Rubinstein duality gives
\[
\biggl|\int_{\S^{d-1}}M(y,z,\theta)\,\dd(u_t^\eta-u_t)(z)\biggr|\leq C\,\mathrm{W}_1(u_t^\eta,u_t)\leq C\,\mathrm{W}_2(u_t^\eta,u_t).
\]
Combining with~\eqref{eq:inv_m_diff} and using $\int\|\nabla_{\S^{d-1}}\mathscr{K}(y,z,\theta)\|\,\dd u_t(z)\leq C$ by~\ref{hyp:ker1}, we obtain
\begin{equation}\label{eq:I2_bound}
I_2\leq C\bigl(\|x-y\|+\mathrm{W}_2(u_t^\eta,u_t)\bigr).
\end{equation}
Putting~\eqref{eq:I1_bound} and~\eqref{eq:I2_bound} together, we conclude
\begin{equation}\label{eq:est_on_v_xy_final}
\|\pi_1(v_t^\eta(x,\theta))-\pi_1(v_t(y,\theta))\|\leq C\bigl(\|x-y\|+\mathrm{W}_2(u_t^\eta,u_t)\bigr).
\end{equation}
Substituting~\eqref{eq:est_on_v_xy_final} into~\eqref{eq:derivative_wasserstein} and applying Cauchy--Schwarz, we get
\[
\begin{aligned}
\frac{\dd}{\dd t}\mathrm{W}_2^2(u_t^\eta,u_t)
&\leq 2C\iint_{\S^{d-1}\times\S^{d-1}}\|x-y\|\bigl(\|x-y\|+\mathrm{W}_2(u_t^\eta,u_t)\bigr)\,\dd\gamma\\
&= 2C\,\mathrm{W}_2^2(u_t^\eta,u_t)+2C\,\mathrm{W}_2(u_t^\eta,u_t)\iint_{\S^{d-1}\times\S^{d-1}}\|x-y\|\,\dd\gamma\\
&\leq C\,\mathrm{W}_2^2(u_t^\eta,u_t),
\end{aligned}
\]
where we used $\iint\|x-y\|^2\,\dd\gamma=\mathrm W_2^2(u_t^\eta,u_t)$ since $\gamma$ is an optimal $\mathrm W_2$-plan.
Finally, by Gr\"onwall's lemma, we obtain
\[
{\mathrm{W}}_2(u_t^\eta, u_t) \leq {\mathrm{W}}_2(u_0^\eta, u_0)\,e^{Ct},
\]
and the thesis follows from the equivalence of ${\mathrm{W}}_{m,2}$ and ${\mathrm{W}}_2$
in Proposition~\ref{prop:equiv_of_norms}.
\end{proof}

\begin{remark}
Let us fix $q \in C^{2}(\mathbb{R})$ and consider $\mathscr{K}$ of the form
\begin{equation}\label{eq:inner_prod_kernel}
\mathscr{K}(x, y, \theta) = q(\langle x, D(\theta) y\rangle).
\end{equation}
Then, we get the refined estimate
\begin{equation}\label{eq:gronwall_estimate_dot_product}
{\mathrm{W}}_{m, 2}(u_t^\eta, u_t) \leq C_1 {\mathrm{W}}_{m, 2}(u_0^\eta, u_0) e^{C_2 M_\theta(t)}
\end{equation}
where
\begin{equation}\label{eq:norm_control}
M_\theta(t) = \int_{0}^{t} \int_{\R^{\alpha}} \|D(\theta)\|^2 \dd \Theta_\tau(\theta) \dd \tau.
\end{equation}

We point out that, considering an $H$-headed transformer architecture with $L$ layers and associated weights distribution $\Theta_t = \delta_{\tilde{\theta}_t}$, where
\[
\tilde\theta_t = \frac{1}{t\sqrt{H}}\left(\theta_t^{(1)}, ..., \theta_t^{(H)}\right),
\]
we get
\[
M_\theta(L) = \frac{1}{H} \sum_{l=1}^L \sum_{h = 1}^H \frac{\|\theta_l^{(h)}\|^2}{l^2} \leq C\max_{l, h} \|\theta_l^{(h)}\|^2,
\]
so that the expansion is independent of the network depth. The above normalization is standard in several numerical algorithms  aiming to stabilize the training of the networks of the models (see for instance \cite{pmlr-v130-hayou21a, hayou2021robust, zhang2022stabilize}). Thus, Theorem \ref{thm:continuous_dependence_on_data} theoretically supports common regularization techniques enhancing the robustness of the models such as weight decay and $L^2$-norm control.
\end{remark}

Finally, from the above results we obtain the following uniqueness corollary.
\begin{cor}\label{cor:uniqueness}
Let $u_0\in\mathcal{P}_2(\S^{d-1})$. Then, there exists exactly one $u\in AC^2([0,T];\mathcal{P}_2(\S^{d-1}))$ that is a gradient flow of $\mathcal{E}$ with initial datum $u_0$.
\end{cor}
\begin{proof}
The result immediately follows by letting $\eta\to 0$ in~\eqref{eq:gronwall_estimate}.
\end{proof}

\subsection{Stability under weight perturbation}
We consider now the stability of the gradient flow constructed in the previous section with respect to a perturbation of the weight parameters. For a more detailed analysis on the topic we refer to~\cite{sandier2004gamma,serfaty2011gamma,serfatyCoulombflows}.

Let us define the following perturbed version of~\eqref{eq:time_dep_interaction_energy}:
\begin{equation}\label{eq:time_dep_interaction_energy_perturbed}
\mathcal{E}^\varepsilon(\mu_t,t) = \frac{1}{2}\iiint_{\R^\alpha\times\S^{d-1}\times\S^{d-1}}\mathscr{K}(x,y,\theta)\,\dd\mu_t(y)\,\dd\mu_t(x)\,\dd\Theta_t^\varepsilon(\theta),
\end{equation}
where
\begin{enumerate}[label=(\textbf{S\arabic*})]
\item\label{hyp:stab_theta} $\Theta_t^\varepsilon$ satisfies~\eqref{eq:continuity_eq_full} for some $f^\varepsilon$, $g^\varepsilon$, and $\Theta_t^\varepsilon\rightharpoonup\Theta_t$ narrowly as $\varepsilon\to 0$ for a.e.\ $t\in[0,T]$; \item\label{hyp:stab_f} $f^\varepsilon\in L^1_{\mathrm{loc}}([0,\infty);L^\infty_{\mathrm{loc}}(\R^\alpha;\R^\alpha))$, $g^\varepsilon\in L^2_{\mathrm{loc}}([0,\infty);L^\infty_{\mathrm{loc}}(\R^\alpha;\R^{\alpha\times\alpha}))$, and $f^\varepsilon(\cdot,t)\to f(\cdot,t)$, $g^\varepsilon(\cdot,t)\to g(\cdot,t)$ weakly in $L^2_{\mathrm{loc}}$ as $\varepsilon\to 0$ for a.e.\ $t\in[0,T]$.
\end{enumerate}
We denote by $b_\mu^{\varepsilon}(x):=\bigl(\int 1/m_\mu(x,\theta)\,\dd\Theta_t^\varepsilon(\theta)\bigr)^{-1}$ the harmonic mean prefactor associated with the perturbed parameter measure.
The following holds.

\begin{lemma}\label{lem:gamma_conv_prop}
Suppose~\ref{hyp:stab_theta} and~\ref{hyp:stab_f} hold, and let $(\mu_t^\varepsilon)_{\varepsilon>0}\subset AC^2((0,T);\mathcal{P}_2(\S^{d-1}))$ be such that
\[
\mu_t^\varepsilon\rightharpoonup\mu_t\qquad\text{narrowly in }\mathcal{P}_2(\S^{d-1}),\quad\forall\,t\in[0,T],
\]
for some $\mu_t\in AC^2((0,T);\mathcal{P}_2(\S^{d-1}))$. Then, for every $t\in[0,T]$, it holds
\begin{enumerate}[label=\upshape(\roman*)]
\item\label{gamma_i} $\liminf_{\varepsilon\to 0}\mathcal{E}^\varepsilon(\mu_t^\varepsilon,t)\geq\mathcal{E}(\mu_t,t)$;
\item\label{gamma_ii} $\liminf_{\varepsilon\to 0}-\partial_t\mathcal{E}^\varepsilon(\mu_t^\varepsilon,t)\geq-\partial_t\mathcal{E}(\mu_t,t)$;
\item\label{gamma_iii} $\liminf_{\varepsilon\to 0}\mathscr{G}_t^\varepsilon(\mu_t^\varepsilon)\geq\mathscr{G}_t(\mu_t)$,
\end{enumerate}
where
\[
\mathscr{G}_t^\varepsilon(\mu) := \biggl(\int_{\S^{d-1}}b_\mu^{\varepsilon}(x)\biggl\|\int_{\R^\alpha}\frac{\mathcal{D}\mathscr{K}[\mu](x,\theta)}{m_\mu(x,\theta)}\,\dd\Theta_t^\varepsilon(\theta)\biggr\|^2\,\dd\mu(x)\biggr)^{1/2}
\]
is the perturbed strong upper gradient.
\end{lemma}

\begin{proof}
Since $\Theta_t^\varepsilon \to \Theta_t$ narrowly for a.e. $t\in [0,T]$ and by applying \cite[Theorem 2.8]{billingsley2013convergence} we have that
\begin{equation}\label{eq:product_conv}
\mu^\varepsilon_t \times \Theta_t^\varepsilon
\underset{\varepsilon \to 0}{\rightharpoonup}
\mu_t \times \Theta_t
\qquad \text{for a.e.\ } t \in [0, T].
\end{equation}
\ref{gamma_i} Recalling \eqref{eq:time_dep_interaction_energy_perturbed},
$\mathcal{E}^\varepsilon(\mu_t^\varepsilon, t)$ can be written as the integral of
$\mathscr{K}(x, y, \theta)$ against the product measure
$\mu_t^\varepsilon \otimes \mu_t^\varepsilon \otimes \Theta_t^\varepsilon$. The result follows from
\eqref{eq:product_conv} and the lower semicontinuity property for narrowly converging measures (see for instance \cite[Lemma 5.1.7]{ambrosio2008gradient}).

\ref{gamma_ii} By a direct computation from~\eqref{eq:full_transformer_ODE}, $\partial_t\mathcal{E}^\varepsilon(\mu_t^\varepsilon,t)$ depends on $f^\varepsilon$, $g^\varepsilon$, and $\Theta_t^\varepsilon$ through the terms $\int\langle\nabla_\theta\mathscr{K},f^\varepsilon\rangle\,\dd\Theta_t^\varepsilon$ and $\int\nabla_\theta^2\mathscr{K}:g^\varepsilon\,\dd\Theta_t^\varepsilon$. By~\ref{hyp:stab_theta} and~\ref{hyp:stab_f}, combined with the boundedness of $\nabla_\theta\mathscr{K}$ and $\nabla_\theta^2\mathscr{K}$ from~\ref{hyp:ker4} and~\ref{hyp:ker1}, $\partial_t\mathcal{E}^\varepsilon(\mu_t^\varepsilon,t)\to\partial_t\mathcal{E}(\mu_t,t)$ for a.e.\ $t$, which in particular implies the result.

\ref{gamma_iii} By~\ref{hyp:ker1} and~\ref{hyp:ker2}, the integrand in $\mathscr{G}_t^\varepsilon(\mu_t^\varepsilon)^2$ is lower semicontinuous in $\mu$ with respect to narrow convergence, and integration against $\Theta_t^\varepsilon$ converges narrowly by~\eqref{eq:product_conv}. The claim follows again from~\cite[Lemma~5.1.7]{ambrosio2008gradient}.
\end{proof}

We are then ready to state and prove the main result of this section.
\begin{thm}\label{thm:stability}
Let us consider the same setting as in Lemma~\ref{lem:gamma_conv_prop}. Let $u_t^\varepsilon$ be a gradient flow of $\mathcal{E}^\varepsilon$ with initial datum $u_0$. Then, there exists $u\in AC^2((0,T);\mathcal{P}_2(\S^{d-1}))$ such that
\begin{equation}\label{eq:conv_of_gf}
u_t^\varepsilon\rightharpoonup u_t\qquad\forall\,t\in[0,T],
\end{equation}
narrowly in $\mathcal{P}_2(\S^{d-1})$. Moreover, $u$ is a gradient flow of $\mathcal{E}$ with initial datum $u_0$.
\end{thm}

\begin{proof}
We divide the proof into two steps: in the first one, we show~\eqref{eq:conv_of_gf}; the second is devoted to proving that $u$ is a gradient flow of $\mathcal{E}$.

\textbf{Step 1.} Let us first observe that, since $u^\varepsilon \in AC^2((0, T); \mathcal{P}_2(\mathbb{S}^{d - 1}))$, we have 
\begin{equation}\label{eq:abs_cont_gd}
{\mathrm{W}}_{m, 2}(u_t^\varepsilon, u_s^\varepsilon) \leq \int_s^t |\dot{u}^\varepsilon|(r) \dd r.
\end{equation}
Moreover, recalling that $u^\varepsilon$ is a gradient flow, by Lemma \ref{lem:metric_der} and Theorem \ref{thm:existence_of_sol_cont_eq_mobility}, it holds
\[
|\dot u^\varepsilon|^2(r) = \int_{\S^{d-1}}b_{u_r^\varepsilon}^{\varepsilon}(x)\biggl\|\int_{\R^\alpha}\frac{\mathcal{D}\mathscr{K}[u_r^\varepsilon](x,\theta)}{m_{u_r^\varepsilon}(x,\theta)}\,\dd\Theta_r^\varepsilon(\theta)\biggr\|^2\,\dd u_r^\varepsilon(x)\leq C,
\]
for some constant $C>0$ independent of $\varepsilon$, thanks to~\ref{hyp:ker4},~\ref{hyp:ker1}, and~\ref{hyp:ker2}. Then, by Hölder inequality, we get
\[
{\mathrm{W}}_{m, 2}(u_t^\varepsilon, u_s^\varepsilon) \leq (t-s)^{1/2}
      \left(\int_s^t |\dot u^\varepsilon|^2(r)\,\mathrm{d}r\right)^{1/2} \leq C |t - s|^{1/2}.
\]
Finally, by compactness of $\S^{d-1}$ and Prokhorov's theorem, the family $(u_t^\varepsilon)_{\varepsilon>0,\,t\in[0,T]}$ is narrowly relatively compact in $\mathcal{P}(\S^{d-1})$; the Hölder estimate gives equicontinuity in $t$. Hence~\eqref{eq:conv_of_gf} follows by the refined Ascoli--Arzelà theorem~\cite[Prop.~3.3.1]{ambrosio2008gradient}.

\medskip
\noindent\textbf{Step 2.} 
Since $u^\varepsilon$ is a curve of maximal slope for
$\mathcal{E}^\varepsilon$, by Theorem~\ref{thm:GMMs_are_curve_of_max_slope} it
satisfies the energy-dissipation identity
\begin{equation}\label{eq:EDI_eps}
\begin{aligned}
\mathcal{E}^\varepsilon(u_0, 0) - \mathcal{E}^\varepsilon(u_T^\varepsilon, T)
= &\, -\int_0^T \partial_t\mathcal{E}^\varepsilon(u_t^\varepsilon, t)\,\dd t
  + \frac{1}{2}\int_0^T |\dot{u}^\varepsilon|^2(t)\,\dd t
  + \frac{1}{2}\int_0^T (\mathscr{G}_t^\varepsilon)^2(u_t^\varepsilon)\,\dd t.
\end{aligned}
\end{equation}
We take $\liminf_{\varepsilon \to 0}$ on both sides. For the left-hand side, since
$u_0^\varepsilon = u_0$ for every $\varepsilon > 0$ and
$\Theta_0^\varepsilon\rightharpoonup\Theta_0$ by~\ref{hyp:stab_theta}, we have
$\mathcal{E}^\varepsilon(u_0, 0) \to \mathcal{E}(u_0, 0)$. Moreover, by
\ref{gamma_i} in Lemma~\ref{lem:gamma_conv_prop}, we have
\[
\liminf_{\varepsilon \to 0} -\mathcal{E}^\varepsilon(u_T^\varepsilon, T)
\leq -\mathcal{E}(u_T, T),
\]
so that
\begin{equation}\label{eq:liminf_lhs}
\liminf_{\varepsilon \to 0}
\bigl[\mathcal{E}^\varepsilon(u_0, 0) - \mathcal{E}^\varepsilon(u_T^\varepsilon, T)\bigr]
\leq \mathcal{E}(u_0, 0) - \mathcal{E}(u_T, T).
\end{equation}
For the right-hand side, reasoning as in the proof of~\ref{gamma_iii} of Lemma~\ref{lem:gamma_conv_prop}, we get
\begin{equation}\label{eq:metric_der_conv}
\liminf_{\varepsilon\to 0}|\dot u^\varepsilon|^2(t)\geq\int_{\S^{d-1}}b_{u_t}(x)\biggl\|\int_{\R^\alpha}\pi_1(v_t(x,\theta))\,\dd\Theta_t(\theta)\biggr\|^2\,\dd u_t(x),
\end{equation}
where $\pi_1(v_t):=\mathcal{D}\mathscr{K}[u_t]/m_{u_t}$.

Since the left-hand side and right-hand side of \eqref{eq:EDI_eps} are equal for each
$\varepsilon$, combining \eqref{eq:liminf_lhs} with \eqref{eq:metric_der_conv},
\ref{gamma_ii}, \ref{gamma_iii} in Lemma~\ref{lem:gamma_conv_prop} and Fatou's lemma, we obtain 
\begin{equation}\label{eq:var_ineq_liminf}
\begin{aligned}
\mathcal{E}(u_0,0)-\mathcal{E}(u_T,T)&\geq -\int_0^T\partial_t\mathcal{E}(u_t,t)\,\dd t\\
&\quad+\frac{1}{2}\int_0^T\int_{\S^{d-1}}b_{u_t}(x)\biggl\|\int_{\R^\alpha}\pi_1(v_t(x,\theta))\,\dd\Theta_t(\theta)\biggr\|^2\,\dd u_t(x)\,\dd t\\
&\quad+\frac{1}{2}\int_0^T\int_{\S^{d-1}}b_{u_t}(x)\biggl\|\int_{\R^\alpha}\frac{\mathcal{D}\mathscr{K}[u_t](x,\theta)}{m_{u_t}(x,\theta)}\,\dd\Theta_t(\theta)\biggr\|^2\,\dd u_t(x)\,\dd t.
\end{aligned}
\end{equation}
Recalling now \eqref{eq:holder_young}
and \eqref{eq:total_diff}, we get that $\mathcal{E}(u_0, 0) - \mathcal{E}(u_T, T)$ is  also a lower bound for the right-hand side of the previous inequality, which implies that all the above inequalities are in fact equalities.
 In particular, equality in~\eqref{eq:var_ineq_liminf} gives
\[
\begin{aligned}
\mathcal{E}(u_T,T)-\mathcal{E}(u_0,0)&-\int_0^T\partial_t\mathcal{E}(u_t,t)\,\dd t\\
&+\frac{1}{2}\int_0^T\int_{\S^{d-1}}b_{u_t}(x)\biggl\|\int_{\R^\alpha}\pi_1(v_t(x,\theta))\,\dd\Theta_t(\theta)\biggr\|^2\,\dd u_t(x)\,\dd t\\
&+\frac{1}{2}\int_0^T\int_{\S^{d-1}}b_{u_t}(x)\biggl\|\int_{\R^\alpha}\frac{\mathcal{D}\mathscr{K}[u_t](x,\theta)}{m_{u_t}(x,\theta)}\,\dd\Theta_t(\theta)\biggr\|^2\,\dd u_t(x)\,\dd t = 0,
\end{aligned}
\]
and the result follows from the last claim of Theorem~\ref{thm:GMMs_are_curve_of_max_slope}.
\end{proof}

\begin{remark}
   Lemma~\ref{lem:gamma_conv_prop} and Theorem~\ref{thm:stability} are a reformulation
of \cite[Theorem~2]{serfaty2011gamma} in our setting. The key observation is that
Step~2 reverses the proof of Theorem~\ref{thm:GMMs_are_curve_of_max_slope}: the EDI
for $u^\varepsilon$ is combined with the $\Gamma$-liminf properties
\ref{gamma_i}--\ref{gamma_iii} to recover the EDI for $u$, from which the gradient
flow characterization follows. The $\Gamma$-limsup property (recovery sequence) is
used implicitly when passing from
$\mathcal{E}^\varepsilon(u_0, 0) \to \mathcal{E}(u_0, 0)$, which holds since
$u_0^\varepsilon = u_0$ is fixed and $\Theta_0^\varepsilon\rightharpoonup\Theta_0$ by~\ref{hyp:stab_theta}.

Moreover, since the measure valued formulation recovers the multi-head model when $\Theta_t$ is empirical,
it is natural to approximate a general parameter distribution $\Theta_t$ by empirical measures
\begin{equation}
\Theta_t^H:=\frac1H\sum_{h=1}^H \delta_{\theta_t^{(h,H)}}, \qquad H\in\mathbb N,
\end{equation}
chosen such that $\Theta_t^H \rightharpoonup \Theta_t$ as $H\to\infty$. If, in addition, these approximations satisfy the perturbative hypotheses \ref{hyp:stab_theta} and \ref{hyp:stab_f},
then Theorem~\ref{thm:stability} yields convergence of the corresponding gradient flows to the gradient flow
associated with $\Theta_t$.
\end{remark}

\section{Numerical Experiments}\label{sec:numerics_long_time}

We complement the theoretical results above with numerical examples demonstrating that the continuous framework employed can accurately describe the discrete model in~\eqref{eq:full_transformer_ODE}\footnote{The full code is available at \url{https://github.com/alexmassucco/Transformers_as_time_dependent_gradient_flows.git}}. Throughout this section, we write $D_t^{(h)} := D(\theta_t^{(h)})$ for the matrix-valued weights of head $h$, and set $\Theta_t = \frac{1}{H}\sum_{h=1}^H \delta_{D_t^{(h)}}$ for the corresponding empirical measure on $\mathrm{Sym}(d)$. We discretize~\eqref{eq:full_transformer_ODE} via a projected explicit Euler method, obtaining
\[
  x_i(t+\Delta t)
  = \Pi\!\left(
      x_i(t)
      + \frac{\Delta t}{H}
        \sum_{h=1}^{H}\sum_{j=1}^{n}
        \frac{e^{\langle x_i,\, D_t^{(h)} x_j\rangle}}
             {\sum_{p=1}^{n} e^{\langle x_i,\, D_t^{(h)} x_p\rangle}}
        D_t^{(h)}\,x_j
    \right),
\]
where $\Pi(x) = x/\|x\|$ denotes the radial projection onto the sphere, and the weight matrices evolve according to the Euler--Maruyama update
\[
  D_{t+\Delta t}^{(h)}
  = D_t^{(h)}
    + f\!\left(D_t^{(h)},t\right)\Delta t
    + g\!\left(D_t^{(h)},t\right) W_{\Delta t}^{(h)},
\]
with the symmetrized Gaussian increment
\[
  W_{\Delta t}^{(h)} = \tfrac{1}{2}(Z + Z^{\top}),
  \qquad
  Z \sim \mathcal{N}(0,\,\Delta t\,\mathrm{Id}),\qquad h = 1,\dots, H.
\]
This symmetry constraint ensures that $D_t^{(h)}$ remains in the space of symmetric matrices for all $t$, consistent with the continuous model.

\paragraph{Discrete energy balance.}
The stochastic formulation requires all quantities to be interpreted as random variables. Applying the discrete It\^{o} formula to the energy
\[
  E_t
  = \frac{1}{2Hn^2}
    \sum_{i,j,h} e^{\langle x_i,\, D_t^{(h)} x_j\rangle},
\]
the energy-balance identity~\eqref{eq:var_ineq} takes the discrete stochastic form
\begin{equation}\label{eq:discrete_energy_balance}
\begin{aligned}
E_T - E_0
  \;=\;&
  \sum_{k=0}^{\lfloor T/\Delta t\rfloor - 1}
  \frac{1}{2Hn^2}\sum_{i,j,h}
    e^{\langle x_i, D_{t_k}^{(h)} x_j\rangle}\,
    \bigl(x_i x_j^{\top} : f(D_{t_k}^{(h)},t_k)\bigr)
    \,\Delta t \\
  &+
  \sum_{k=0}^{\lfloor T/\Delta t\rfloor - 1}
  \frac{1}{4Hn^2}\sum_{i,j,h}
    e^{\langle x_i, D_{t_k}^{(h)} x_j\rangle}
    \bigl(x_i x_j^{\top} : g(D_{t_k}^{(h)},t_k)\bigr)^2\,\Delta t \\
  &-
  \sum_{k=0}^{\lfloor T/\Delta t\rfloor - 1}
    \mathscr{G}_{t_k}^2(u_{t_k})\,\Delta t
  \;+\; M_T,
\end{aligned}
\end{equation}
where $\mathscr{G}_{t}^2$ is the squared strong upper gradient defined in~\eqref{eq:upper_grad}, $A : B = \mathrm{tr}(A^\top B)$ denotes the Frobenius inner product, and
\[
  M_T
  = \frac{1}{2Hn^2}
    \sum_{k=0}^{\lfloor T/\Delta t\rfloor-1}
    \sum_{i,j,h}
    e^{\langle x_i, D_{t_k}^{(h)} x_j\rangle}
    \bigl(x_i x_j^{\top} : g(D_{t_k}^{(h)},t_k)\,W_{\Delta t}^{(h,k)}\bigr)
\]
is the martingale increment arising from the It\^{o} correction, with $W_{\Delta t}^{(h,k)}$ the independent Gaussian increment for head $h$ at step $k$. In the deterministic case $g\equiv 0$, identity~\eqref{eq:discrete_energy_balance} reduces to a discrete gradient-flow equality mirroring the continuous identity~\eqref{eq:var_ineq}. Moreover, since $M_T$ is a martingale with $\mathbb{E}[M_T]=0$, taking expectations in~\eqref{eq:discrete_energy_balance} eliminates the stochastic residual and recovers~\eqref{eq:var_ineq} in the continuous limit.

\begin{figure}[ht!]
  \centering
  \includegraphics[width=0.95\linewidth]{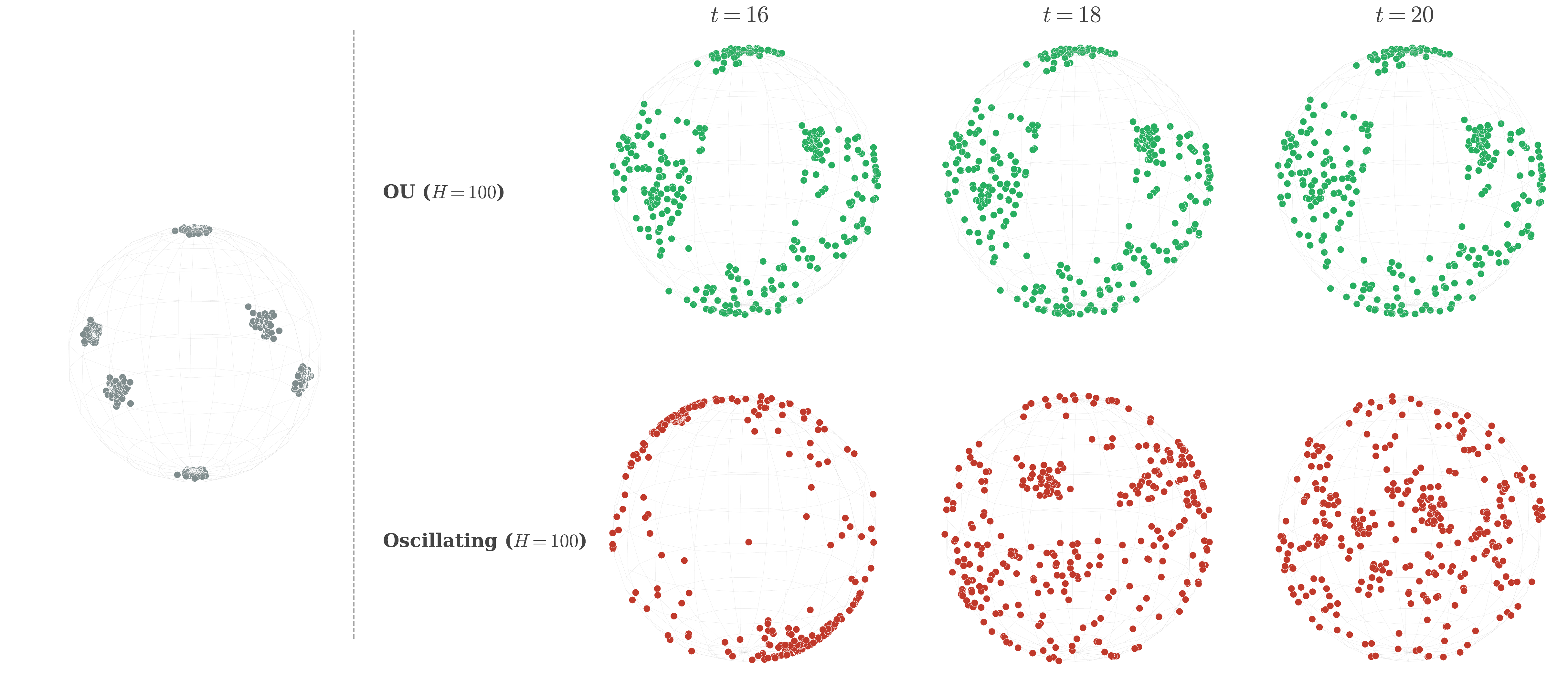}
  \caption{%
    Token cloud at $t\in\{0,16,18,20\}$ on $\mathbb{S}^2$ ($n=150$ tokens).
    Top row: OU weights (Monte Carlo mean over $N_{\mathrm{MC}} = 20$ trajectories), showing progressive token clustering consistent with convergence to a stationary point of $\mathcal{E}$ as predicted by Theorem~\ref{thm:stationary_points}.
    Bottom row: oscillating weights, showing no persistent clustering, consistent with the absence of a stationary limit when Assumption~\eqref{eq:integrable_time_derivative_energy} is violated.
  }
  \label{fig:combined_snapshots}
\end{figure}

\paragraph{Ornstein--Uhlenbeck weights.}
We consider the generalized Ornstein--Uhlenbeck (OU) drift and diffusion
coefficients
\[
  f(D,t) = F - D,
  \qquad
  g(D,t) = \sqrt{2\sigma^2},
\]
for a fixed symmetric matrix $F \in \mathrm{Sym}(d)$, so that the weight SDE becomes
\[
  \dd D_t^{(h)} = (F - D_t^{(h)})\,\dd t + \sqrt{2\sigma^2}\,\dd W_t^{(h)},\qquad h=1,\dots,H.
\]
This is a standard matrix-valued OU process whose law converges
exponentially fast to the invariant Gaussian distribution on $\mathrm{Sym}(d)$ with mean $F$ and entrywise variance $\sigma^2$, denoted $\overline\Theta = \mathcal{N}(F, \sigma^2\mathrm{Id})$, independently of the
initial distribution $\Theta_0$~\cite{Pavliotis2014}. In particular,
\begin{equation}\label{eq:exp_conv_mean}
  \mathbb{E}[D_t^{(h)}] = F + e^{-t}(\mathbb{E}[D_0^{(h)}] - F) \;\underset{t \to \infty}{\rightarrow}\; F,
  \qquad \mathrm{Cov}(D_t^{(h)}) = \sigma^2\,\mathrm{Id} + \mathcal{O}(e^{-2t}).
\end{equation}
Since $\mathcal{E}(u_t, t)$ depends on $t$ only through $\Theta_t$, and $\Theta_t$ reverts to its mean at an exponential rate, the partial derivative $\partial_t \mathcal{E}(u_t, t)$ relaxes exponentially fast in expectation. In particular,
\[
\mathbb{E}\bigl[\,\sup_{t \geqslant 0}
  \int_0^t \partial_\tau \mathcal{E}(u_\tau, \tau)\,\dd\tau\,\bigr]
  < +\infty,
\]
so that Assumption~\eqref{eq:integrable_time_derivative_energy} is satisfied in expectation along trajectories and Theorem~\ref{thm:stationary_points} applies. Every limit point of the flow therefore belongs to $\bigcup_{\bar\Theta \in \Omega(\Theta)} \mathbb{E}_{\mathscr{G}}^{\bar\Theta}$, confirming that the token dynamics relax to a stationary configuration of the energy at the limiting weight distribution $\overline\Theta = \mathcal{N}(F, \sigma^2\mathrm{Id})$. This setup provides a continuous-time interpretation of standard neural-network initialization schemes such as normal initialization ($\Theta_0 = \mathcal{N}(0, \sigma^2\mathrm{Id})$) and He initialization~\cite{He2015} ($\sigma^2 = 2/n_{\text{in}}$, with $n_{\text{in}}$ the fan-in), both of which are absorbed into the unique stationary distribution $\overline\Theta = \mathcal{N}(F, \sigma^2\mathrm{Id})$ of the OU flow.

We simulate $N_{\mathrm{MC}} = 20$ independent trajectories of the coupled system $(x_t, \Theta_t)$ with $F = \mathrm{Id}$ on $\mathbb{S}^2$ ($n=300$ tokens, $d=3$), using time-step $\Delta t = 1\times 10^{-2}$ up to $T=20$, and vary the number of attention heads $H \in \{1,10,100\}$ with fixed $\sigma^2 = 1$.

Figure~\ref{fig:combined_snapshots} (top row) displays snapshots of the mean token cloud (averaged over all $N_{\mathrm{MC}}$ trajectories) at times $t \in \{16,18,20\}$, alongside the initial configuration at $t=0$. For $H=100$, the mean token positions rapidly converge to a static distribution induced by the stationary distribution of the weights. Figure~\ref{fig:combined_energy_balance} (left) visualizes the discrete energy identity~\eqref{eq:discrete_energy_balance} for $H=100$, showing the Monte Carlo mean of each term---$\Delta\mathcal{E}$, $\int_0^t\mathscr{G}_s^2\,\dd s$, $\int_0^t\partial_s\mathcal{E}\,\dd s$, and the stochastic residual $M_t$---with $\pm 1\sigma$ shaded bands. The total balance remains zero throughout, providing a consistency check of the numerical scheme.

Recall that the velocity of the tokens is given by
\[
  v_i(t)
  \;=\;
  \frac{1}{H}\sum_{h=1}^{H}
  \underbrace{P_{x_i}^\perp\!\Bigl(\sum_{j=1}^{n} A_{ij}(D_t^{(h)})\,D_t^{(h)}\, x_j\Bigr)}_{=:\,V_i^{(h)}(t)\;\in\;T_{x_i}\mathbb{S}^{d-1}}, \qquad \forall i \in [1, n],
\]
where $A_{ij}(D)$ denotes the softmax attention weights from~\eqref{eq:self_att}.
Using i.i.d.\ sampling of the weight matrices, and conditioning on the current token configuration $\{x_j(t)\}$, the squared upper gradient decomposes as
\begin{equation}\label{eq:G2_decomp}
\mathbb{E}_{\Theta_t}\!\bigl[\mathscr{G}_t^2\bigr]
  \;=\;
  \frac{1}{n}\sum_{i=1}^{n}\mathbb{E}_{\Theta_t}\!\bigl[\bigl\|v_i(t)\bigr\|^2\bigr]
  \;=\; \frac{1}{n}\sum_{i=1}^{n}
    \Bigl(\frac{1}{H}\,\mathrm{Var}\!\left[V_i^{(1)}(t)\right]
    \;+\;
\left\|\mathbb{E}_{\Theta_t}\Bigl[V_i^{(1)}(t)\Bigr]\right\|^2 \Bigr),
\end{equation}
where $\mathrm{Var}[V_i^{(1)}] = \mathbb{E}_{\Theta_t}[\|V_i^{(1)}\|^2] -
\|\mathbb{E}_{\Theta_t}[V_i^{(1)}]\|^2$ is the per-token variance of a single head's tangential output.

The two terms in~\eqref{eq:G2_decomp} exhibit distinct decay mechanisms. On one hand, by~\eqref{eq:exp_conv_mean} and the boundedness of the attention weights, $\mathrm{Var}[V_i^{(1)}(t)]$ is uniformly bounded in $t$, so the variance term is $\mathcal{O}(1/H)$ uniformly in~$t$. On the other hand, the mean single-head velocity $\bar V_i(t):=\mathbb{E}_{\Theta_t}[V_i^{(1)}(t)]$ is analysed via the Taylor expansion of $D\mapsto V_i^{(1)}$ around the equilibrium $F$. Taking the expectation under the law of $D_t^{(1)}$ and under the invariant law $\overline\Theta$ and subtracting yields
\[
\bar V_i(t) = \bar V_i^\infty(t) + \mathcal{O}(e^{-t}),
\qquad
\bar V_i^\infty(t) := \mathbb{E}_{D\sim\overline\Theta}\!\left[V_i^{(1)}(D,\{x_j(t)\})\right].
\]
Indeed, by~\eqref{eq:exp_conv_mean} the law of $D_t^{(1)}$ and $\overline\Theta$ share the same mean $F$ and the same covariance $\sigma^2\,\mathrm{Id}$ up to $\mathcal{O}(e^{-2t})$, so the zeroth- and second-order terms in the expansion cancel, leaving only the first-moment difference $\mathbb{E}[D_t^{(1)} - F] = e^{-t}(\mathbb{E}[D_0] - F)$. The quantity $\bar V_i^\infty(t)$ is therefore the expected velocity at the current token configuration $\{x_j(t)\}$ when the parameter is sampled from the invariant law $\overline\Theta$ rather than from the actual law of $D_t^{(1)}$ at time~$t$. By Theorem~\ref{thm:stationary_points}, the tokens reach a critical configuration of $\mathcal{E}(\cdot,\infty)$, hence $\bar V_i^\infty(t) \to 0$. This is consistent with the classical theory of the OU process, for which we additionally know that $\bar V_i^\infty(t) = \mathcal{O}(e^{-t})$ due to the exponential mean-reversion. Combining the above, there exist constants $C_0, C_1 > 0$ such that the strong upper gradient satisfies the two-parameter bound
\begin{equation}\label{eq:G2_twoparam}
\mathbb{E}_{\Theta_t}\!\bigl[\mathscr{G}_t^2\bigr]
  \;\leq\;
  \frac{C_1}{H}
  \;+\;
  e^{-2 t}\,C_0.
\end{equation}
This theoretical prediction is confirmed numerically in Figure~\ref{fig:combined_gradient} (top row). The left column shows the Monte Carlo mean time series of $\mathscr{G}_t^2$ for $H \in \{1,10,100\}$, confirming the expected convergence to $0$ as $t\to\infty$ and $H\to\infty$. The right column plots the time-averaged mean $\overline{\mathscr{G}^2}$ against $H$. We fit the interpolating function $k(H;a,b) = a\,H^{b}$ to the data; and obtain the following $95\%$ uncertainty bands
\[
a \in [0.28, 0.53], \qquad b \in [-1.04, -0.83]
\]
with optimal parameters $a^* = 0.384$ and $b^* = -0.933$, yielding a linear fit in $\log$-$\log$ scale with Pearson coefficient $r_{OU} = 1.00$, strongly supporting the quantitative decay predicted by~\eqref{eq:G2_twoparam}.

\begin{figure}[ht!]
  \centering
  \includegraphics[width=0.95\linewidth]{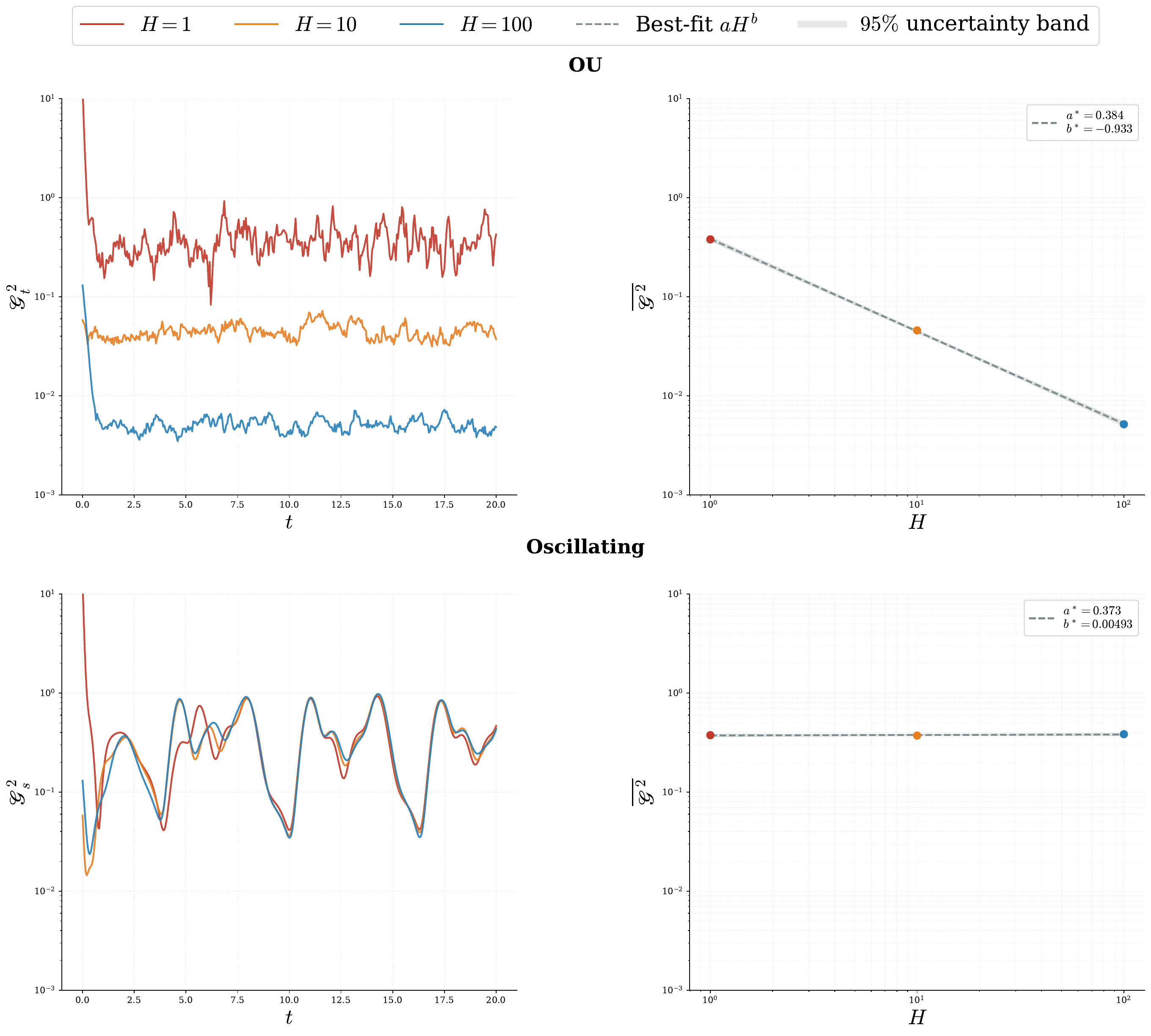}
  \caption{%
    Strong upper-gradient analysis on $\mathbb{S}^2$ for $H\in\{1,10,100\}$. Left column: $\mathscr{G}_t^2$ versus time. Right column: time-averaged mean $\overline{\mathscr{G}^2}$ versus $H$ with best-fit curve $aH^{b}$ (dashed) and $95\%$ uncertainty band. Top row: OU weights (Monte Carlo mean over $N_{\mathrm{MC}} = 20$ trajectories), confirming the gradient-variance decay of~\eqref{eq:G2_twoparam} with $\mathcal{O}(H^{-1})$ quantitative scaling. Bottom row: oscillating weights, exhibiting the non-decaying behavior induced by the non-stationary weights.
  }
  \label{fig:combined_gradient}
\end{figure}

\paragraph{Oscillating weights.}
We now consider the purely deterministic case
\[
  f(D, t) = F(t) - D,
  \qquad
  g(D, t) = 0,
\]
where $F(t)$ is a time-periodic symmetric matrix. Each weight matrix evolves according to $\dot D_t^{(h)} = F(t) - D_t^{(h)}$, with explicit solution
\[
  D_t^{(h)} = e^{-t} D_0^{(h)} + \int_0^t e^{-(t-s)} F(s)\,\dd s,\qquad h=1,\dots,H.
\]
Because $F$ is periodic and non-decaying, $D_t^{(h)}$ is attracted toward a time-periodic orbit; in particular, $\dot{D}_t^{(h)} = F(t) - D_t^{(h)}$ oscillates around a strictly positive value and does not decay to zero. Since $\mathcal{E}(u_t,t)$ inherits this periodicity through $D_t^{(h)}$, the integrand $\partial_t\mathcal{E}(u_t,t)$ also oscillates without a definite sign, and its running integral grows without bound:
\[
  \sup_{t \geqslant 0}
  \int_0^t \partial_\tau\mathcal{E}(u_\tau,\tau)\,\dd\tau = +\infty.
\]
Assumption~\eqref{eq:integrable_time_derivative_energy} is therefore violated, Theorem~\ref{thm:stationary_points} cannot be applied, and the conclusion $\omega_u \subseteq \bigcup_{\bar\Theta\in\Omega(\Theta)} \mathbb{E}_{\mathscr{G}}^{\bar\Theta}$ may fail. Correspondingly, the metric slope $\mathscr{G}_t^2$ receives a persistent non-zero contribution from the mean-velocity term: since $\dot{D}_t^{(h)}$ does not vanish, $\bar V_i(t)\not\to 0$ and bound~\eqref{eq:G2_twoparam} no longer holds.

Concretely, we simulate deterministic ($g = 0$) trajectories on $\mathbb{S}^2$ ($n=500$, $d=3$) with the oscillating symmetric targets
\begin{align*}
  F_{\mathbb{S}^2}^{(h)}(t)
  &=
  \begin{pmatrix}
    2 + 1.5\cos \left(t + 2\pi\frac{h - 1}{H}\right) & \sin \left(2t + 2\pi\frac{h - 1}{H}\right)       & \sin \left(t + 2\pi\frac{h - 1}{H}\right)   \\
    \sin \left(2t + 2\pi\frac{h - 1}{H}\right)       & 2+1.5\sin \left(t + 2\pi\frac{h - 1}{H}\right)   & \cos \left(2t + 2\pi\frac{h - 1}{H}\right)  \\
    \sin \left(t + 2\pi\frac{h - 1}{H}\right)        & \cos \left(2t + 2\pi\frac{h - 1}{H}\right)       & 2+1.5\cos\left(t + \frac{\pi}{4} + 2\pi\frac{h - 1}{H}\right)
  \end{pmatrix},
\end{align*}
and initial conditions $D_0^{(h)} \sim \mathcal{N}(0, \mathrm{Id})$ for $H \in \{1, 10, 100\}$. Each diagonal entry of $F(t)$ oscillates with non-decaying amplitude, ensuring that the integrability condition is persistently violated.

Figure~\ref{fig:combined_snapshots} (bottom row) displays snapshots of the token cloud at times $t\in\{16,18,20\}$ alongside the initial configuration at $t=0$. Unlike the OU case, the token cloud does not settle into a stable distribution but continues to evolve, reflecting the absence of a stationary limit.

Figure~\ref{fig:combined_energy_balance} (right) plots the identity in~\eqref{eq:discrete_energy_balance} along with its individual terms. As predicted by the theory, the energy balance is again preserved in time.

Finally, Figure~\ref{fig:combined_gradient} (bottom row) shows the time series of $\mathscr{G}_t^2$ (left column) and the time-averaged mean (right column). As in the OU case, we fit the interpolating function $k(H;a,b) = a\,H^{b}$ and obtain $95\%$ uncertainty bands $a \in [0.27, 0.54]$ and $b \in [-0.06, 0.07]$, with optimal parameters $a^* = 0.373$, $b^* = 0.005$ ($r_{\mathrm{osc}} = -0.33$ in $\log$-$\log$ scale), confirming the non-decaying behaviour of the flow expected from the violation of Assumption~\eqref{eq:integrable_time_derivative_energy}.

\begin{figure}[ht!]
  \centering
  \includegraphics[width=0.9\linewidth]{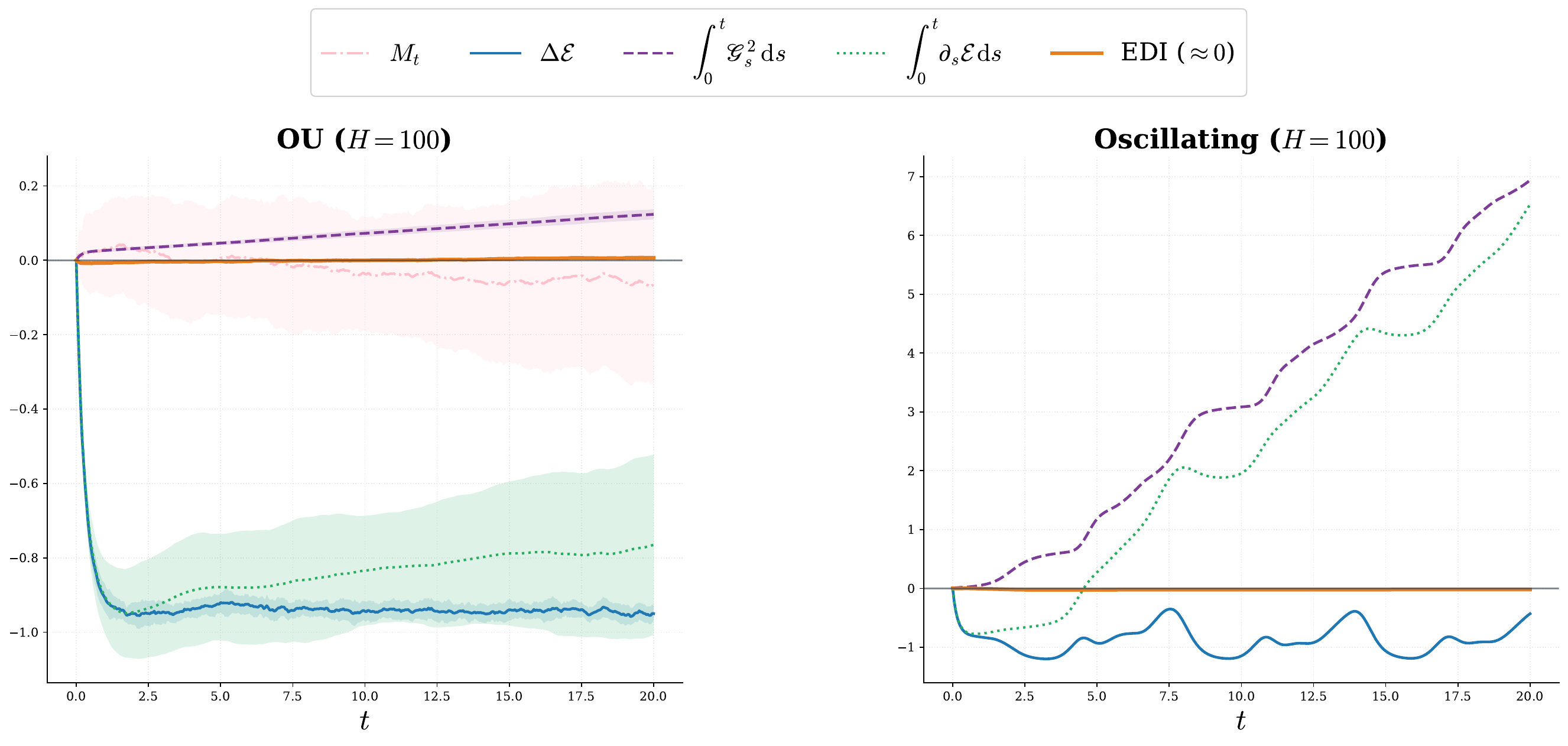}
  \caption{%
    Energy balance on $\mathbb{S}^2$ ($H=100$): individual terms of the energy
    identity~\eqref{eq:discrete_energy_balance}---$\Delta\mathcal{E}$
    (blue), $\int_0^t\mathscr{G}_s^2\,\dd s$ (purple),
    $\int_0^t\partial_s\mathcal{E}\,\dd s$ (green), and the stochastic
    residual $M_t$ (rose, OU only)---together with their total balance
    (EDI, orange). Left: OU weights (Monte Carlo mean over $N_{\mathrm{MC}} = 20$ trajectories, $\pm1\sigma$ bands shaded); the EDI lies within the error band $\mathcal{O}(H^{-1})$ as expected.
    Right: oscillating weights.
  }
  \label{fig:combined_energy_balance}
\end{figure}

\section{Conclusions and Future Directions}
In this work we generalized previous mean-field analyses of fixed-weight transformers to multi-headed, depth-dependent architectures, narrowing the gap with the architectures used in practice. Indeed, by modelling layer-wise parameter variation through a time-dependent probability measure $\Theta_t$, we formulated inference dynamics as a genuinely time-dependent Wasserstein gradient flow. A central theoretical contribution is the derivation of the effective scalar mobility $b^{\Theta_t}_\mu(x)$, the $\Theta_t$-weighted harmonic mean of the head-dependent mobilities, which compresses the multi-head transport geometry into a single scalar factor on the token space. This framework rigorously introduces the Transformer PDE and extends the energetic perspective of attention mechanisms to dynamic regimes in which the interaction energy $\mathcal{E}(\mu, t)$ carries an explicit time dependence.\\
We also established that the structural evolution of the weights intrinsically governs the macroscopic behaviour of the token distribution. When the parameter dynamics satisfy a suitable integrability condition, as shown by Theorem~\ref{thm:stationary_points} and illustrated by the Ornstein--Uhlenbeck simulations of Section~\ref{sec:numerics_long_time}, the token distribution relaxes to a stationary configuration and the strong upper gradient decays at rate $\mathcal{O}(H^{-1})$. Conversely, sustained oscillations in $\Theta_t$ obstruct this relaxation, providing numerical evidence that token clustering is not an inescapable artefact of deep attention but a phenomenon controlled by the parameter dynamics. Furthermore, we investigated the stability of our models with respect to perturbations of the initial data and of the weight distribution, obtaining a Grönwall-type estimate for the former and a $\Gamma$-convergence result for the latter; together these provide a mathematical justification for common architectural initialization techniques such as weight decay and $L^2$-norm control.

\subsection{Future Research Directions}
The rigorous connection between the Transformer PDE and non-autonomous gradient flows opens several paths for characterising the geometry of stationary points. A natural next step is a broader classification of the $\omega$-limit sets, and the investigation of how head multiplicity shapes the underlying energy landscape. In particular, one may ask whether specific spectral properties of the $\Theta_t$-averaged attention matrices dictate the number and spatial distribution of limiting clusters on the sphere $\mathbb{S}^{d-1}$. A range of quantitative results on clustering, synchronization, and metastability in the autonomous regime have recently been obtained (see for instance \cite{alcalde2026quantifying, bruno2025emergence, chen2025quantitative, criscitiello2024synchronization}); extending such estimates to the time-dependent, multi-head setting developed here is a natural next step, and would in turn shed light on the emergence of distinct token sub-populations and on sharp bounds on the expressivity of a given attention configuration.

Additionally, the time-dependent interaction energy $\mathcal{E}(\mu, t)$ introduced here provides a rigorous mathematical foundation for designing novel neural architectures, rather than merely analyzing existing ones. Within our framework both the parameter measure $\Theta_t$ and the energy $\mathcal{E}$ itself can be treated as design objects: one may formulate optimal control problems over $\Theta_t$, or modify $\mathcal{E}$ by adding regularizing terms. A natural example is the addition of an entropic penalty $\varepsilon\int\mu\log\mu$, whose Wasserstein gradient flow yields a McKean--Vlasov equation on the sphere, that is, an interaction dynamics with built-in diffusion of the tokens. The stationary states, bifurcation structure, and clustering dynamics of such noisy transformer models, and more generally of the weakly interacting diffusions underlying them, have recently attracted considerable attention (see e.g.\ \cite{   balasubramanian2025structure, fedorov2026clustering, gerber2025formation, shalova2026solutions} and references therein); the time-dependent formulation developed here provides a natural setting in which to design and analyze architectures of this kind that retain the dissipation and robustness properties established in this paper. The variational viewpoint further suggests building energy-stable transformer layers via structure-preserving discretizations of the resulting gradient flow, such as JKO-type minimizing-movement schemes that perform a time-discrete minimization of the chosen energy directly. Such physically motivated architectures would inherently resist token collapse, promote feature separation, and achieve robustness by design, offering a principled alternative to the regularization strategies currently prevalent in the literature.

\bibliographystyle{abbrv}
\bibliography{EntTr}

@inproceedings{
tong2025neural,
title={Neural {ODE} Transformers: Analyzing Internal Dynamics and Adaptive Fine-tuning},
author={Anh Tong and Thanh Nguyen-Tang and Dongeun Lee and Duc Nguyen and Toan Tran and David Leo Wright Hall and Cheongwoong Kang and Jaesik Choi},
booktitle={International Conference on Learning Representations},
year={2025},
url={https://openreview.net/forum?id=XnDyddPcBT}
}

@book{diestel1977vector,
  title={Vector Measures},
  author={Diestel, Joseph and Uhl, John Jerry},
  year={1977},
  publisher={American Mathematical Society}
}

@article{alcalde2026quantifying,
  title={Quantifying Concentration Phenomena of Mean-Field Transformers in the Low-Temperature Regime},
  author={Alcalde, Albert and Bungert, Leon and Riedl, Konstantin and Roith, Tim},
  journal={arXiv preprint arXiv:2605.10931},
  year={2026}
}

@article{vaswani2017attention,
  title={Attention is all you need},
  author={Vaswani, Ashish and Shazeer, Noam and Parmar, Niki and Uszkoreit, Jakob and Jones, Llion and Gomez, Aidan N and Kaiser, {L}ukasz and Polosukhin, Illia},
  journal={Advances in neural information processing systems},
  volume={30},
  year={2017}
}

@article{zhang2019root,
  title={Root mean square layer normalization},
  author={Zhang, Biao and Sennrich, Rico},
  journal={Advances in neural information processing systems},
  volume={32},
  year={2019}
}

@inproceedings{
hayou2021robust,
title={Robust Pruning at Initialization},
author={Soufiane Hayou and Jean-Francois Ton and Arnaud Doucet and Yee Whye Teh},
booktitle={International Conference on Learning Representations},
year={2021},
url={https://openreview.net/forum?id=vXj_ucZQ4hA}
}

@article{zhang2022stabilize,
  title={Stabilize deep {R}es{N}et with a sharp scaling factor $\tau$},
  author={Zhang, Huishuai and Yu, Da and Yi, Mingyang and Chen, Wei and Liu, Tie-Yan},
  journal={Machine Learning},
  volume={111},
  number={9},
  pages={3359--3392},
  year={2022},
  publisher={Springer}
}

@inproceedings{pmlr-v130-hayou21a,
  title = 	 {Stable {R}es{N}et},
  author =       {Hayou, Soufiane and Clerico, Eugenio and He, Bobby and Deligiannidis, George and Doucet, Arnaud and Rousseau, Judith},
  booktitle = 	 {Proceedings of The 24th International Conference on Artificial Intelligence and Statistics},
  pages = 	 {1324--1332},
  year = 	 {2021},
  editor = 	 {Banerjee, Arindam and Fukumizu, Kenji},
  volume = 	 {130},
  series = 	 {Proceedings of Machine Learning Research},
  month = 	 {13--15 Apr},
  publisher =    {PMLR},
  url = 	 {https://proceedings.mlr.press/v130/hayou21a.html}
}

@article{shen2023study,
  title={A study on {R}e{LU} and softmax in transformer},
  author={Shen, Kai and Guo, Junliang and Tan, Xu and Tang, Siliang and Wang, Rui and Bian, Jiang},
  journal={arXiv preprint arXiv:2302.06461},
  year={2023}
}

@inproceedings{hua2022transformer,
  title={Transformer quality in linear time},
  author={Hua, Weizhe and Dai, Zihang and Liu, Hanxiao and Le, Quoc},
  booktitle={International Conference on Machine Learning},
  pages={9099--9117},
  year={2022},
  organization={PMLR}
}

@inproceedings{
ramapuram2024theory,
title={Theory, Analysis, and Best Practices for Sigmoid Self-Attention},
author={Jason Ramapuram and Federico Danieli and Eeshan Gunesh Dhekane and Floris Weers and Dan Busbridge and Pierre Ablin and Tatiana Likhomanenko and Jagrit Digani and Zijin Gu and Amitis Shidani and Russell Webb},
booktitle={International Conference on Learning Representations},
year={2025},
url={https://openreview.net/forum?id=Zhdhg6n2OG}
}

@article{burger2025analysis,
  title={Analysis of mean-field models arising from self-attention dynamics in transformer architectures with layer normalization},
  author={Burger, Martin and Kabri, Samira and Korolev, Yury and Roith, Tim and Weigand, Lukas},
  journal={Philosophical Transactions A},
  volume={383},
  number={2298},
  pages={20240233},
  year={2025},
  publisher={The Royal Society}
}

@inproceedings{xiong2020layer,
  title={On layer normalization in the transformer architecture},
  author={Xiong, Ruibin and Yang, Yunchang and He, Di and Zheng, Kai and Zheng, Shuxin and Xing, Chen and Zhang, Huishuai and Lan, Yanyan and Wang, Liwei and Liu, Tieyan},
  booktitle={International Conference on Machine Learning},
  pages={10524--10533},
  year={2020},
  organization={PMLR}
}

@article{huo2025capturing,
  title={Capturing {AI}'s Attention: Physics of Repetition, Hallucination, Bias and Beyond},
  author={Huo, Frank Yingjie and Johnson, Neil F},
  journal={arXiv preprint arXiv:2504.04600},
  year={2025}
}

@inproceedings{
ramsauer2020hopfield,
title={Hopfield Networks is All You Need},
author={Hubert Ramsauer and Bernhard Sch{\"a}fl and Johannes Lehner and Philipp Seidl and Michael Widrich and Lukas Gruber and Markus Holzleitner and Thomas Adler and David Kreil and Michael K Kopp and G{\"u}nter Klambauer and Johannes Brandstetter and Sepp Hochreiter},
booktitle={International Conference on Learning Representations},
year={2021},
url={https://openreview.net/forum?id=tL89RnzIiCd}
}

@article{kan2025stability,
  title={Stability of transformers under layer normalization},
  author={Kan, Kelvin and Li, Xingjian and Zhang, Benjamin J and Sahai, Tuhin and Osher, Stanley and Kumar, Krishna and Katsoulakis, Markos A},
  journal={arXiv preprint arXiv:2510.09904},
  year={2025}
}

@article{geneva2022transformers,
  title={Transformers for modeling physical systems},
  author={Geneva, Nicholas and Zabaras, Nicholas},
  journal={Neural Networks},
  volume={146},
  pages={272--289},
  year={2022},
  publisher={Elsevier}
}

@article{zhang2024transformers,
  title={Why transformers need {A}dam: A {H}essian perspective},
  author={Zhang, Yushun and Chen, Congliang and Ding, Tian and Li, Ziniu and Sun, Ruoyu and Luo, Zhi-Quan},
  journal={Advances in neural information processing systems},
  volume={37},
  pages={131786--131823},
  year={2024}
}

@inproceedings{zhang2019fixup,
  title={Fixup initialization: Residual learning without normalization},
  author={Zhang, Hongyi and Dauphin, Yann N and Ma, Tengyu},
  booktitle={International Conference on Learning Representations},
  year={2019}
}

@article{radford2019language,
  title={Language models are unsupervised multitask learners},
  author={Radford, Alec and Wu, Jeffrey and Child, Rewon and Luan, David and Amodei, Dario and Sutskever, Ilya and others},
  journal={OpenAI blog},
  volume={1},
  number={8},
  pages={9},
  year={2019}
}

@article{rossi2008metric,
  title={A metric approach to a class of doubly nonlinear evolution equations and applications},
  author={Rossi, Riccarda and Mielke, Alexander and Savar{\'e}, Giuseppe},
  journal={Annali della Scuola Normale Superiore di Pisa, Classe di Scienze},
  pages={97--169},
  year={2008}
}

@article{team2023gemini,
  title={Gemini: a family of highly capable multimodal models},
  author={{Gemini Team} and Anil, Rohan and Borgeaud, Sebastian and Alayrac, Jean-Baptiste and Yu, Jiahui and Soricut, Radu and Schalkwyk, Johan and Dai, Andrew M and Hauth, Anja and Millican, Katie and others},
  journal={arXiv preprint arXiv:2312.11805},
  year={2023}
}

@book{billingsley2013convergence,
  title={Convergence of probability measures},
  author={Billingsley, Patrick},
  year={2013},
  publisher={John Wiley \& Sons}
}

@article{liu2024deepseek,
  title={{D}eep{S}eek-{V}3 technical report},
  author={Liu, Aixin and Feng, Bei and Xue, Bing and Wang, Bingxuan and Wu, Bochao and Lu, Chengda and Zhao, Chenggang and Deng, Chengqi and Zhang, Chenyu and Ruan, Chong and others},
  journal={arXiv preprint arXiv:2412.19437},
  year={2024}
}

@article{wu2016google,
  title={Google's neural machine translation system: Bridging the gap between human and machine translation},
  author={Wu, Yonghui and Schuster, Mike and Chen, Zhifeng and Le, Quoc V and Norouzi, Mohammad and Macherey, Wolfgang and Krikun, Maxim and Cao, Yuan and Gao, Qin and Macherey, Klaus and others},
  journal={arXiv preprint arXiv:1609.08144},
  year={2016}
}

@inproceedings{he2016deep,
  title={Deep residual learning for image recognition},
  author={He, Kaiming and Zhang, Xiangyu and Ren, Shaoqing and Sun, Jian},
  booktitle={Proceedings of the IEEE Conference on Computer Vision and Pattern Recognition},
  pages={770--778},
  year={2016}
}

@inproceedings{
dosovitskiy2020image,
title={An Image is Worth 16x16 Words: Transformers for Image Recognition at Scale},
author={Alexey Dosovitskiy and Lucas Beyer and Alexander Kolesnikov and Dirk Weissenborn and Xiaohua Zhai and Thomas Unterthiner and Mostafa Dehghani and Matthias Minderer and Georg Heigold and Sylvain Gelly and Jakob Uszkoreit and Neil Houlsby},
booktitle={International Conference on Learning Representations},
year={2021},
url={https://openreview.net/forum?id=YicbFdNTTy}
}

@inproceedings{radford2021learning,
  title={Learning transferable visual models from natural language supervision},
  author={Radford, Alec and Kim, Jong Wook and Hallacy, Chris and Ramesh, Aditya and Goh, Gabriel and Agarwal, Sandhini and Sastry, Girish and Askell, Amanda and Mishkin, Pamela and Clark, Jack and others},
  booktitle={International Conference on Machine Learning},
  pages={8748--8763},
  year={2021},
  organization={PMLR}
}

@article{baevski2020wav2vec,
  title={wav2vec 2.0: A framework for self-supervised learning of speech representations},
  author={Baevski, Alexei and Zhou, Yuhao and Mohamed, Abdelrahman and Auli, Michael},
  journal={Advances in neural information processing systems},
  volume={33},
  pages={12449--12460},
  year={2020}
}

@inproceedings{radford2023robust,
  title={Robust speech recognition via large-scale weak supervision},
  author={Radford, Alec and Kim, Jong Wook and Xu, Tao and Brockman, Greg and McLeavey, Christine and Sutskever, Ilya},
  booktitle={International Conference on Machine Learning},
  pages={28492--28518},
  year={2023},
  organization={PMLR}
}

@article{jumper2021highly,
  title={Highly accurate protein structure prediction with {A}lpha{F}old},
  author={Jumper, John and Evans, Richard and Pritzel, Alexander and Green, Tim and Figurnov, Michael and Ronneberger, Olaf and Tunyasuvunakool, Kathryn and Bates, Russ and {\v{Z}}{\'\i}dek, Augustin and Potapenko, Anna and others},
  journal={Nature},
  volume={596},
  number={7873},
  pages={583--589},
  year={2021},
  publisher={Nature Publishing Group UK London}
}

@article{stokes2020deep,
  title={A deep learning approach to antibiotic discovery},
  author={Stokes, Jonathan M and Yang, Kevin and Swanson, Kyle and Jin, Wengong and Cubillos-Ruiz, Andres and Donghia, Nina M and MacNair, Craig R and French, Shawn and Carfrae, Lindsey A and Bloom-Ackermann, Zohar and others},
  journal={Cell},
  volume={180},
  number={4},
  pages={688--702},
  year={2020},
  publisher={Elsevier}
}

@article{lam2023learning,
  title={Learning skillful medium-range global weather forecasting},
  author={Lam, Remi and Sanchez-Gonzalez, Alvaro and Willson, Matthew and Wirnsberger, Peter and Fortunato, Meire and Alet, Ferran and Ravuri, Suman and Ewalds, Timo and Eaton-Rosen, Zach and Hu, Weihua and others},
  journal={Science},
  volume={382},
  number={6677},
  pages={1416--1421},
  year={2023},
  publisher={American Association for the Advancement of Science}
}

@article{gu2020empirical,
  title={Empirical asset pricing via machine learning},
  author={Gu, Shihao and Kelly, Bryan and Xiu, Dacheng},
  journal={The Review of Financial Studies},
  volume={33},
  number={5},
  pages={2223--2273},
  year={2020},
  publisher={Oxford University Press}
}

@article{topol2019high,
  title={High-performance medicine: the convergence of human and artificial intelligence},
  author={Topol, Eric J},
  journal={Nature Medicine},
  volume={25},
  number={1},
  pages={44--56},
  year={2019},
  publisher={Nature Publishing Group US New York}
}

@article{geshkovski2023mathematical,
  title={A mathematical perspective on transformers},
  author={Geshkovski, Borjan and Letrouit, Cyril and Polyanskiy, Yury and Rigollet, Philippe},
  journal={Bulletin of the American Mathematical Society},
  volume={62},
  number={3},
  pages={427--479},
  year={2025}
}

@article{sandier2004gamma,
  title={Gamma-convergence of gradient flows with applications to {G}inzburg--{L}andau},
  author={Sandier, Etienne and Serfaty, Sylvia},
  journal={Communications on Pure and Applied Mathematics: A Journal Issued by the Courant Institute of Mathematical Sciences},
  volume={57},
  number={12},
  pages={1627--1672},
  year={2004},
  publisher={Wiley Online Library}
}

@book{villani2008optimal,
  title={Optimal transport: old and new},
  author={Villani, C{\'e}dric},
  volume={338},
  year={2008},
  publisher={Springer}
}

@book {Santambrogio,
    AUTHOR = {Santambrogio, F.},
     TITLE = {Optimal transport for applied mathematicians},
    SERIES = {Progress in Nonlinear Differential Equations and their
              Applications},
    VOLUME = {87},
 PUBLISHER = {Birkh\"{a}user/Springer, Cham},
      YEAR = {2015},
     PAGES = {xxvii+353},
      ISBN = {978-3-319-20827-5; 978-3-319-20828-2},
   MRCLASS = {49-02 (35J96 49J45 49M29 58E50 90C05 90C48 91B02)},
  MRNUMBER = {3409718},
MRREVIEWER = {Luigi De Pascale},
       DOI = {10.1007/978-3-319-20828-2},
       URL = {https://doi.org/10.1007/978-3-319-20828-2},
}

@article{serfaty2011gamma,
  title={Gamma-convergence of gradient flows on {H}ilbert and metric spaces and applications},
  author={Serfaty, Sylvia},
  journal={Discrete and Continuous Dynamical Systems},
  volume={31},
  number={4},
  pages={1427--1451},
  year={2011},
}

@inproceedings{kim2021lipschitz,
  title={The {L}ipschitz constant of self-attention},
  author={Kim, Hyunjik and Papamakarios, George and Mnih, Andriy},
  booktitle={International Conference on Machine Learning},
  pages={5562--5571},
  year={2021},
  organization={PMLR}
}

@inproceedings{sander2022sinkformers,
  title={Sinkformers: Transformers with doubly stochastic attention},
  author={Sander, Michael E and Ablin, Pierre and Blondel, Mathieu and Peyr{\'e}, Gabriel},
  booktitle={International Conference on Artificial Intelligence and Statistics},
  pages={3515--3530},
  year={2022},
  organization={PMLR}
}

@article{karagodin2024clustering,
  title={Clustering in causal attention masking},
  author={Karagodin, Nikita and Polyanskiy, Yury and Rigollet, Philippe},
  journal={Advances in Neural Information Processing Systems},
  volume={37},
  pages={115652--115681},
  year={2024}
}

@article{serfatyCoulombflows,
  title={Mean field limit for {C}oulomb-type flows},
  author={Serfaty, Sylvia},
  journal={Duke Mathematical Journal},
  volume={169},
  number={15},
  pages={2887--2935},
  year={2020},
}

@article{geshkovski2024emergence,
  title={The emergence of clusters in self-attention dynamics},
  author={Geshkovski, Borjan and Letrouit, Cyril and Polyanskiy, Yury and Rigollet, Philippe},
  journal={Advances in Neural Information Processing Systems},
  volume={36},
  year={2024}
}

@article{geshkovski2024measure,
  title={Measure-to-measure interpolation using Transformers},
  author={Geshkovski, Borjan and Rigollet, Philippe and Ruiz-Balet, Dom{\`e}nec},
  journal={arXiv preprint arXiv:2411.04551},
  year={2024}
}

@article{castin2025unified,
  title={A Unified Perspective on the Dynamics of Deep Transformers},
  author={Castin, Val{\'e}rie and Ablin, Pierre and Carrillo, Jos{\'e} Antonio and Peyr{\'e}, Gabriel},
  journal={arXiv preprint arXiv:2501.18322},
  year={2025}
}

@book{ambrosio2008gradient,
  title={Gradient flows: in metric spaces and in the space of probability measures},
  author={Ambrosio, Luigi and Gigli, Nicola and Savar{\'e}, Giuseppe},
  year={2008},
  publisher={Springer Science \& Business Media}
}

@article{ferreira2018gradient,
  title={Gradient flows of time-dependent functionals in metric spaces and applications to {PDE}s},
  author={Ferreira, Lucas CF and Valencia-Guevara, Julio C},
  journal={Monatshefte f{\"u}r Mathematik},
  volume={185},
  pages={231--268},
  year={2018},
  publisher={Springer}
}

@inproceedings{bruno2025emergence,
  title={Emergence of meta-stable clustering in mean-field transformer models},
  author={Bruno, Giuseppe and Pasqualotto, Federico and Agazzi, Andrea},
  booktitle={International Conference on Learning Representations},
  pages={7496--7526},
  year={2025}
}

@article{dolbeault2009new,
  title={A new class of transport distances between measures},
  author={Dolbeault, Jean and Nazaret, Bruno and Savar{\'e}, Giuseppe},
  journal={Calculus of Variations and Partial Differential Equations},
  volume={34},
  number={2},
  pages={193--231},
  year={2009},
  publisher={Springer}
}

@article{MMCS,
author = {Matthes, D. and McCann, R. J. and Savar\'{e}, G.},
title = {A Family of Nonlinear Fourth Order Equations of Gradient Flow Type},
journal = {Communications in Partial Differential Equations},
volume = {34},
number = {11},
pages = {1352--1397},
year = {2009},
publisher = {Taylor \& Francis},
doi = {10.1080/03605300903296256},
URL = {https://doi.org/10.1080/03605300903296256},
eprint = {https://doi.org/10.1080/03605300903296256}
}

@article{benamou2000computational,
  title={A computational fluid mechanics solution to the {M}onge--{K}antorovich mass transfer problem},
  author={Benamou, Jean-David and Brenier, Yann},
  journal={Numerische Mathematik},
  volume={84},
  number={3},
  pages={375--393},
  year={2000},
  publisher={Springer-Verlag Berlin/Heidelberg}
}

@article{criscitiello2024synchronization,
  title={Synchronization on circles and spheres with nonlinear interactions},
  author={Criscitiello, Christopher and Rebjock, Quentin and McRae, Andrew D and Boumal, Nicolas},
  journal={arXiv preprint arXiv:2405.18273},
  year={2024}
}

@article{touvron2023llama,
  title={Llama: Open and efficient foundation language models},
  author={Touvron, Hugo and Lavril, Thibaut and Izacard, Gautier and Martinet, Xavier and Lachaux, Marie-Anne and Lacroix, Timoth{\'e}e and Rozi{\`e}re, Baptiste and Goyal, Naman and Hambro, Eric and Azhar, Faisal and others},
  journal={arXiv preprint arXiv:2302.13971},
  year={2023}
}

@book{Pavliotis2014,
author = {Pavliotis, Grigorios A.},
address = {New York},
booktitle = {Stochastic processes and applications : diffusion processes, the Fokker-Planck and Langevin equations},
isbn = {9781493913237},
language = {eng},
publisher = {Springer},
series = {Texts in applied mathematics ; Volume 60},
title = {Stochastic processes and applications : diffusion processes, the Fokker-Planck and Langevin equations },
url = {https://ebookcentral.proquest.com/lib/surrey/detail.action?docID=6314316},
year = {2014},
}

@article{shalova2026solutions,
  title={Solutions of stationary {M}c{K}ean--{V}lasov equation on a high-dimensional sphere and other {R}iemannian manifolds},
  author={Shalova, Anna and Schlichting, Andr{\'e}},
  journal={Advances in Nonlinear Analysis},
  volume={15},
  number={1},
  pages={20250141},
  year={2026},
  publisher={De Gruyter}
}

@article{fedorov2026clustering,
  title={Clustering in Deep Stochastic Transformers},
  author={Fedorov, Lev and Sander, Micha{\"e}l E and Elie, Romuald and Marion, Pierre and Lauri{\`e}re, Mathieu},
  journal={arXiv preprint arXiv:2601.21942},
  year={2026}
}

@article{gerber2025formation,
  title={Formation of clusters and coarsening in weakly interacting diffusions},
  author={Gerber, Nicolai and Gvalani, Rishabh S and Hairer, Martin and Pavliotis, Grigorios A and Schlichting, Andr{\'e}},
  journal={arXiv preprint arXiv:2510.17629},
  year={2025}
}

@article{balasubramanian2025structure,
  title={On the structure of stationary solutions to {M}c{K}ean--{V}lasov equations with applications to noisy transformers},
  author={Balasubramanian, Krishnakumar and Banerjee, Sayan and Rigollet, Philippe},
  journal={arXiv preprint arXiv:2510.20094},
  year={2025}
}

@inproceedings{he2015,
  title={Delving deep into rectifiers: Surpassing human-level performance on {I}mage{N}et classification},
  author={He, Kaiming and Zhang, Xiangyu and Ren, Shaoqing and Sun, Jian},
  booktitle={Proceedings of the IEEE International Conference on Computer Vision},
  pages={1026--1034},
  year={2015}
}

@article{chen2025quantitative,
  title={Quantitative Clustering in Mean-Field Transformer Models},
  author={Chen, Shi and Lin, Zhengjiang and Polyanskiy, Yury and Rigollet, Philippe},
  journal={arXiv preprint arXiv:2504.14697},
  year={2025}
}

@article {KLS,
    AUTHOR = {Hauer, Daniel and Maz\'on, Jos\'e{} M.},
     TITLE = {Kurdyka--{L}ojasiewicz--{S}imon inequality for gradient flows in
              metric spaces},
   JOURNAL = {Trans. Amer. Math. Soc.},
  FJOURNAL = {Transactions of the American Mathematical Society},
    VOLUME = {372},
      YEAR = {2019},
    NUMBER = {7},
     PAGES = {4917--4976},
      ISSN = {0002-9947,1088-6850},
   MRCLASS = {49Q20 (35K90 39B62 49J52 58J35)},
  MRNUMBER = {4009443},
MRREVIEWER = {Robin\ Neumayer},
       DOI = {10.1090/tran/7801},
       URL = {https://doi.org/10.1090/tran/7801},
}

\end{document}